\documentclass{article} 
\usepackage{iclr2025_conference,times}

\usepackage{amsmath,amsfonts,bm}

\def\eqref#1{equation~\ref{#1}}

\def\1{\bm{1}}

\DeclareMathAlphabet{\mathsfit}{\encodingdefault}{\sfdefault}{m}{sl}
\SetMathAlphabet{\mathsfit}{bold}{\encodingdefault}{\sfdefault}{bx}{n}

\usepackage{hyperref}
\usepackage{url}
\usepackage{amsmath}
\usepackage{amsthm}
\usepackage{fourier}
\usepackage{listings}
\usepackage{makecell}
\usepackage{amsmath}
\usepackage{bm}
\usepackage{amsmath}
\usepackage{tabu}

\usepackage{hyperref}
\usepackage{url}
\usepackage{enumitem}
\usepackage{amsmath}
\usepackage[bb=dsserif]{mathalpha}
\usepackage{bm}

\usepackage{tikz}
\usepackage{graphics}
\usepackage{color}
\usepackage{float}
\usepackage{caption}
\usepackage{subcaption}
\usepackage{mathtools, cases}
\usepackage{annotate-equations}
\usepackage{enumitem,kantlipsum}
\usepackage{lipsum}
\usepackage{amsmath}
\usepackage{array, caption, tabularx, makecell, booktabs}%
\usepackage[framemethod=TikZ]{mdframed}
\usepackage{amsthm}
\usepackage{multirow}
\usepackage{wrapfig}
\usepackage{lipsum}  

\usepackage{algorithm}
\usepackage{algpseudocode}

\usepackage{listings}
\usepackage{tikz}
\usetikzlibrary{positioning, arrows.meta}
\usepackage{ifthen}
\usepackage{xspace}
\usepackage{xcolor}

\graphicspath{ {./figures/} }

\newlength{\maxwidth}
\newcommand{\algalign}[2]%
{\makebox[\maxwidth][r]{$#1{}$}${}#2$}

\DeclareDocumentCommand{\jingAlgo}{ m O{=\ } }{%
	{\rlap{$#1$} \hphantom{text}$#2$}%
}

\usepackage{tcolorbox}
\usepackage{colortbl}
\usepackage{xcolor}
\usepackage{thmtools}
\usepackage{thm-restate}
\definecolor{blue}{rgb}{0, 0, .7}
\definecolor{darkred}{rgb}{.5, 0, 0}
\definecolor{darkgreen}{rgb}{0, .5, 0}
\usepackage{multirow}
\usepackage{animate}

\usepackage{xcolor}
\usepackage{thmtools}
\usepackage{thm-restate}
\definecolor{blue}{rgb}{0, 0, .7}
\definecolor{darkred}{rgb}{.5, 0, 0}
\definecolor{darkgreen}{rgb}{0, .5, 0}
\usepackage{multirow}

\DeclareMathAlphabet{\mathcal}{OMS}{cmsy}{m}{n}
\SetMathAlphabet{\mathcal}{bold}{OMS}{cmsy}{b}{n}

\usepackage[T1]{fontenc}
\usepackage[utf8]{inputenc}
\usepackage{mathtools} 

\usepackage{epigraph}

\newcommand{\bW}{\boldsymbol{W}}
\newcommand{\bw}{\boldsymbol{w}}
\newcommand{\bX}{\boldsymbol{X}}
\newcommand{\bY}{\boldsymbol{Y}}

\newcommand{\bU}{\boldsymbol{U}}

\newcommand{\bI}{\boldsymbol{I}}
\newcommand{\bx}{\boldsymbol{x}}
\newcommand{\by}{\boldsymbol{y}}
\newcommand{\bz}{\boldsymbol{z}}

\newcommand{\bu}{\boldsymbol{u}}
\newcommand{\bv}{\boldsymbol{v}}
\newcommand{\bq}{\boldsymbol{q}}

\newcommand{\be}{\boldsymbol{e}}

\newcommand{\ba}{\boldsymbol{a}}

\newcommand{\bb}{\boldsymbol{b}}
\newcommand{\bp}{\boldsymbol{p}}

\newcommand{\bK}{\boldsymbol{K}}
\newcommand{\bV}{\boldsymbol{V}}
\newcommand{\bP}{\boldsymbol{P}}

\newcommand{\bA}{\boldsymbol{A}}
\newcommand{\bB}{\boldsymbol{B}}

\newcommand{\bC}{\boldsymbol{C}}

\newcommand{\bJ}{\boldsymbol{J}}

\newcommand{\bM}{\boldsymbol{M}}
\newcommand{\bG}{\boldsymbol{G}}

\newcommand{\cU}{\boldsymbol{\mathcal{U}}}
\newcommand{\cV}{\boldsymbol{\mathcal{V}}}

\newcommand{\fu}{\boldsymbol{\mathfrak{u}}}
\newcommand{\fv}{\boldsymbol{\mathfrak{v}}}

\newcommand{\widesim}[2][1.5]{
  \mathrel{\overset{#2}{\scalebox{#1}[1]{$\sim$}}}
}

\definecolor{commentcolor}{RGB}{110,154,155}   

\newcommand{\ie}{\textit{i}.\textit{e}.}
\newcommand{\eg}{\textit{e}.\textit{g}.}

\definecolor{green2}{rgb}{0.21,0.74,0.49}
\definecolor{rebuttal}{RGB}{134, 1, 175}

\theoremstyle{break}
\newtheorem{theorem}{Theorem}
\newtheorem{principle}{Principle}
\newtheorem{mrule}{Rule}
\newtheorem{definition}{Definition}

\newtheorem{proposition}{Proposition}

\newcommand{\myparagraph}[1]{\noindent\textbf{#1.}}

\title{\centering Taming Transformer \\ without using learning rate warmup}

\author{
Xianbiao Qi$^{1}$, Yelin He$^{1}$, Jiaquan Ye$^{1}$, Chun-Guang Li$^{2\dagger}$, Bojia Zi$^{3}$, Xili Dai$^{4}$, Qin Zou$^{5}$,
\textbf{Rong Xiao$^{1\dagger}$} \\
\textsuperscript{1}Intellifusion Inc. \quad 
\textsuperscript{2}BUPT \quad 
\textsuperscript{3}CUHK \quad 
\textsuperscript{4}HKUST (GZ) \quad 
\textsuperscript{5}WHU
}

\iclrfinalcopy 
\begin{document}


\maketitle
\renewcommand{\thefootnote}{} 
\footnotetext{$\dagger$ Corresponding authors}

\begin{abstract}
Scaling Transformer to a large scale without using some technical tricks such as learning rate warump and using an obviously lower learning rate is an extremely challenging task, and is increasingly gaining more attention. In this paper, we provide a theoretical analysis for the process of training Transformer and reveal the rationale behind the model crash phenomenon in the training process, termed \textit{spectral energy concentration} of ${\bW_q}^{\top} \bW_k$, which is the reason for a malignant entropy collapse, where ${\bW_q}$ and $\bW_k$ are the projection matrices for the query and the key in Transformer, respectively. 
To remedy this problem, motivated by \textit{Weyl's Inequality}, we present a novel optimization strategy, \ie, making the weight updating in successive steps smooth---if the ratio $\frac{\sigma_{1}(\nabla \bW_t)}{\sigma_{1}(\bW_{t-1})}$ is larger than a threshold, we will automatically bound the learning rate to a weighted multiple of $\frac{\sigma_{1}(\bW_{t-1})}{\sigma_{1}(\nabla \bW_t)}$, where $\nabla \bW_t$ is the updating quantity in step $t$. Such an optimization strategy can prevent spectral energy concentration to only a few directions, and thus can avoid malignant entropy collapse which will trigger the model crash. We conduct extensive experiments using ViT, Swin-Transformer and GPT, showing that our optimization strategy can effectively and stably train these Transformers without using learning rate warmup.
\end{abstract}

\epigraph{\textit{``{\large N}othing in life is to be feared. It is only to be understood.''}}{--- \textit{Marie Curie}}

\section{Introduction}
\label{sec:introduction}

Transformer~\citep{transformer_vaswani2017attention} has revolutionized various domains of artificial intelligence, including natural language processing~\citep{gpt1_radford2018improving, gpt2_radford2019language, gpt3_brown2020language, palm_chowdhery2023palm, llama2_touvron2023llama, llama3_dubey2024llama} and computer vision~\citep{vit_dosovitskiy2020image, liu2021swin} and many more applications~\citep{clip_radford2021learning, dalle1_ramesh2021zero, dit_peebles2023scalable}, owning 
to their ability to capture long-range dependencies through self-attention mechanisms. 
However, despite their widespread application and empirical success, training deep Transformer models remains quite challenging. Practitioners 
frequently encounter variant issues, such as gradient explosion~\citep{qi2023understanding}, rank collapse~\citep{rank_collapse_dong2021attention}, entropy collapse~\citep{stabilizing_transformer_zhai2023stabilizing} and general training instability~\citep{lipschitz_constant_kim2021lipschitz, qi2023understanding}, 
especially during the initial training stage. 

To address these challenges, researchers have proposed various modifications to the original Transformer architecture, including altering the placement of Layer Normalization~\citep{prenorm_wang2019learning, on_layer_norm_prenorm_xiong2020layer} (\eg, pre-LN vs. post-LN schemes), carefully conditioning the residual connections~\citep{rezero_bachlechner2021rezero}, and QKNorm~\citep{qk_norm_henry2020query, scaling22b_dehghani2023scaling} for the self-attention module. 
Similarly, DeepNet~\citep{deepnorm_wang2022deepnet} introduces a new normalization function to modify the residual connection in Transformer. ReZero~\citep{rezero_bachlechner2021rezero} 
introduces a learnable residual scalar parameter for the residual shortcut, and requires initiating it to 0 at the start stage of training.
More recent approaches~\citep{lipschitz_constant_kim2021lipschitz, lipsformer_qilipsformer} have focused on examining and enforcing Lipschitz continuity properties of Transformer components, which can provide insights into the network behavior and the training stability. 
\emph{Although there are a few 
works~\citep{rezero_bachlechner2021rezero, lipsformer_qilipsformer} that can avoid using learning rate warmup to train Transformer successfully, all of them 
require significant modifications of the network structure. 
}

\textit{Learning rate warmup}~\citep{warmup_loshchilov2016sgdr} seems to be a must-have technology for standard optimizers~\citep{sgd_robbins1951stochastic, adagrad_duchi2011adaptive, adam_kingma2014adam, adamw_IlyaLoshchilov2018FixingWD} in some popular large Transformer models~\citep{gpt1_radford2018improving, gpt2_radford2019language, gpt3_brown2020language, palm_chowdhery2023palm, llama2_touvron2023llama}. Without the learning rate warmup stage, the Transformer training will   
prone to diverge.

Although it is usual to train a Transformer by 
modifying the network structure as mentioned above or using the learning rate warmup, 
two natural and interesting questions remain: 
\begin{enumerate}[leftmargin=*]
\item \textit{What are 
the training dynamics of a Transformer model when its training fails or succeeds? }    
\item \textit{Can we successfully tame  
a Transformer without changing the network structure or without using learning rate warmup?}
\end{enumerate}

This paper aims to answer these questions. 
To answer the first question, we examine  
the training processes of three types of Transformers, 
by visualizing the changing trajectories along with the training process of 15 (or 13) quantities about the parameters, activations, and attention maps. 
By doing so, we observe that the model crash is accompanied by a weird phenomenon that the entropy of the attention map is almost 0 and the spectral norm of ${\bW_q}^{\top} \bW_k$ increases to a very large value. 
By conducting mathematical analysis for the Transformer training, we identify that the 
Spectral Energy Concentration (SEC) of ${\bW_q}^{\top} \bW_k$ is the key problem leading to the model crash. 
To answer the second question, 
motivated by Weyl' Inequality, we present a novel optimization strategy, \ie, making weight updating smooth, and verify empirically that our optimization strategy can prevent 
spectral energy concentration and thus achieving a stable convergence in training. 

\myparagraph{Paper Contributions} The contributions of the paper are highlighted as follows. 
\begin{itemize}[leftmargin=*]
\item We visualize the training dynamics of Transformers that train successfully or unsuccessfully and summarize two important observations from unsuccessful training that: 
a) the rank of the attention map matrix tends to very low and the entropy of attention probability matrix tends to 0; and 
b) $\sigma_1({\bW_q}^{\top} \bW_k)$ increases rapidly to a very large value. 

\item We present theoretical analysis for the Transformer training, 
finding that the Jacobian matrix $\frac{\partial \operatorname{vec}(\bP)}{\partial \operatorname{vec}({\bW_q}^{\top}{\bW_k} )} =  \bX^{\top} \otimes \bX^{\top}$, where $\bP = \bX^{\top} {\bW_q}^{\top} {\bW_k} {\bX}$. It implies that the gradient of ${\bW_q}^{\top}{\bW_k}$ is largely 
dominated by the rank of $\bX^{\top} \otimes \bX^{\top}$.

\item We reveal that 
SEC of $\boldsymbol{W}_q^{\top} \boldsymbol{W}_k$ makes the attention map matrix to be sparse yet low-rank and it is the 
reason leading to model crash.

\item Motivated by the Weyl's inequality, we introduce a novel strategy to address the problem of  
SEC of ${\bW_q}^{\top} \bW_k$ by controlling the rapid growth of singular values, and verify that our strategy leads to a stable training process. 

\end{itemize}

\section{Preliminaries}
\label{sec:preliminary}

\textbf{Matrix norm.} Given a matrix $\bW$, its 
$\ell_p$-norm 
is defined as: $\|\bW\|_p = \sup_{\bx \neq 0} \frac{\|\bW\bx\|_p}{\|\bx\|_p}$. 
When \( p = 2 \), the induced matrix norm is the \textit{spectral norm}, 
which is defined as the largest singular value of \( \bW \) and 
is also expressed as the square root of the largest eigenvalue of the Gram matrix  \( \bW^{\top}\bW \). 
The spectral norm of a matrix \( \bW \) can be calculated as: $\|\bW\|_2 = \max _{\bx \in S^{n-1}}\|\bA \bx\|_{2} = \sqrt{\lambda_{\max} (\bW^{\top}\bW)} = \sigma_{1} (\bW)$,
where \( \sigma_{1} (\bW) \) denotes the largest singular value of matrix \( \bW \) and $\lambda_{\max} (\bW^{\top}\bW)$ denotes the largest eigenvalue of $\bW^{\top}\bW$, 
and $S^{n-1}$ denotes a unit sphere in $\mathcal{R}^n$. 

\textbf{Power iteration to compute matrix spectral norm.} The power iteration algorithm starts with a vector \( \bx_0 \) of unit $\ell_2$-norm.  
The entire iteration process is as follows: 
$\bx_{k+1} = \frac{\bW\bx_k}{\|\bW\bx_k\|_2}$ for $k=0,\cdots,K-1$.
At every iteration, \( \bx_k \) is multiplied by matrix \( \bW \) and normalized. 
After $K$ iterations, 
$\|\bx_K\|_2$ is
used as the 
estimated spectral norm. 
Usually, it takes 3 to 5 iterations to converge, and thus the computation cost is cheap.

\textbf{Adam Optimizer.} Adam
~\citep{adam_kingma2014adam} is currently the most widely used optimizer for training neural networks, owing to its efficiency and effectiveness. Adam can be simply defined as: 
${\bM}_t = \beta_1  {\bM}_{t-1} +  (1 - \beta_1) {\bG}_t, \quad 
{\bV}_t = \beta_2 {\bV}_{t-1} +    (1 - \beta_2) {\bG}_t^2, \quad 
{\bW}_{t} = {\bW}_{t-1} - \alpha_t  {\bM}_t \oslash \sqrt{{{\bV}_{t}} + \epsilon }, $
where ${\bG}_t$ is the gradient at 
step $t$, ${\bG}_t^2$ is the element-wise square of ${\bG}_t$, and $\oslash$ denotes element-wise division, $\alpha_t$ is the learning rate at 
step $t$, and $\beta_1, \beta_2$ are the first-order and the second-order momentum factors, respectively.

\begin{figure}[htbp]
	\centering
	\begin{subfigure}{0.33\linewidth}
	\centering
	\includegraphics[width=4.5cm,height=3.8cm]{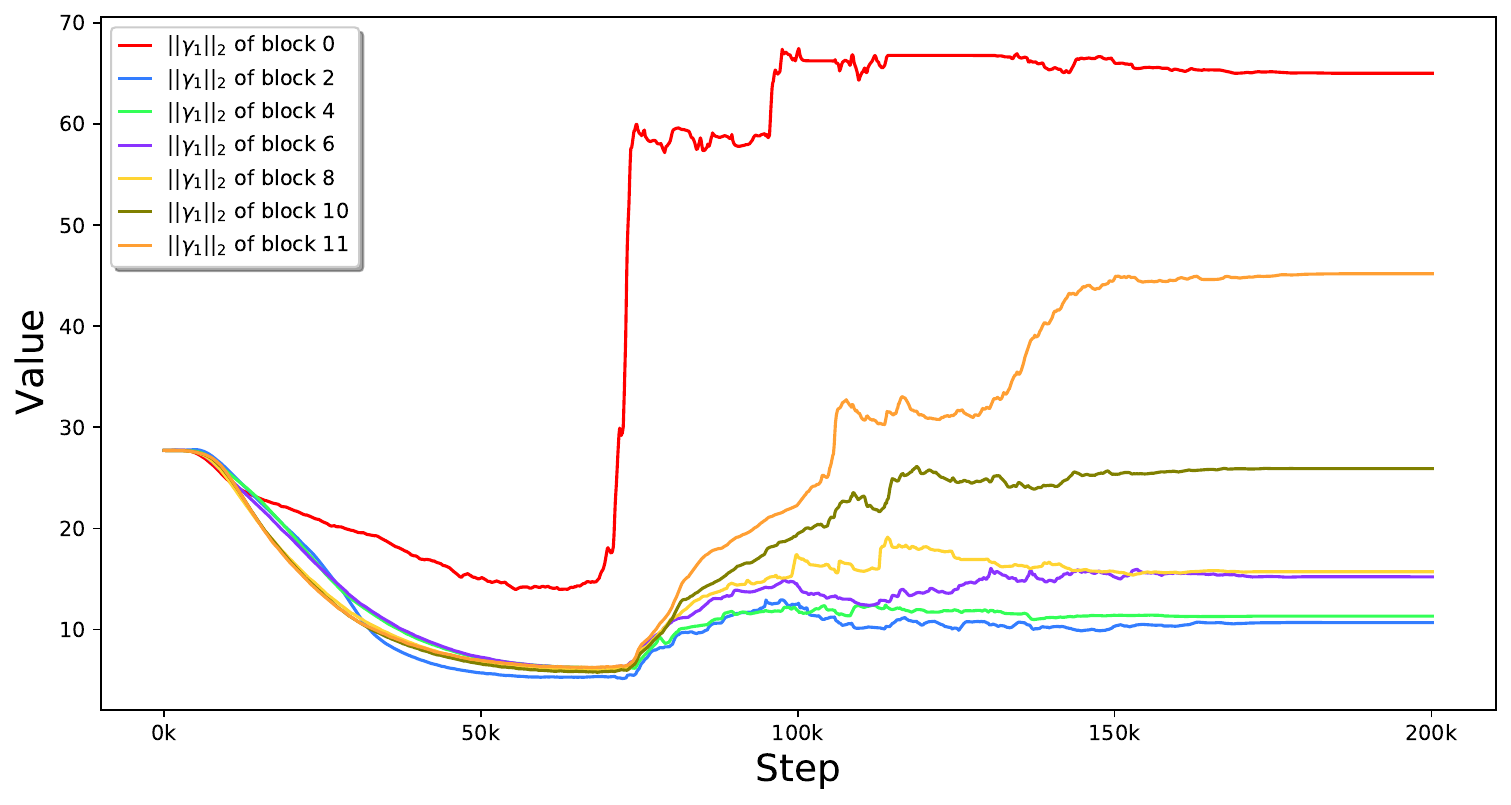}
	\caption*{(a) $\|\boldsymbol{\gamma_1}\|_2$}
	\end{subfigure} 
	\begin{subfigure}{0.33\linewidth}
		\centering
		\includegraphics[width=4.5cm,height=3.8cm]{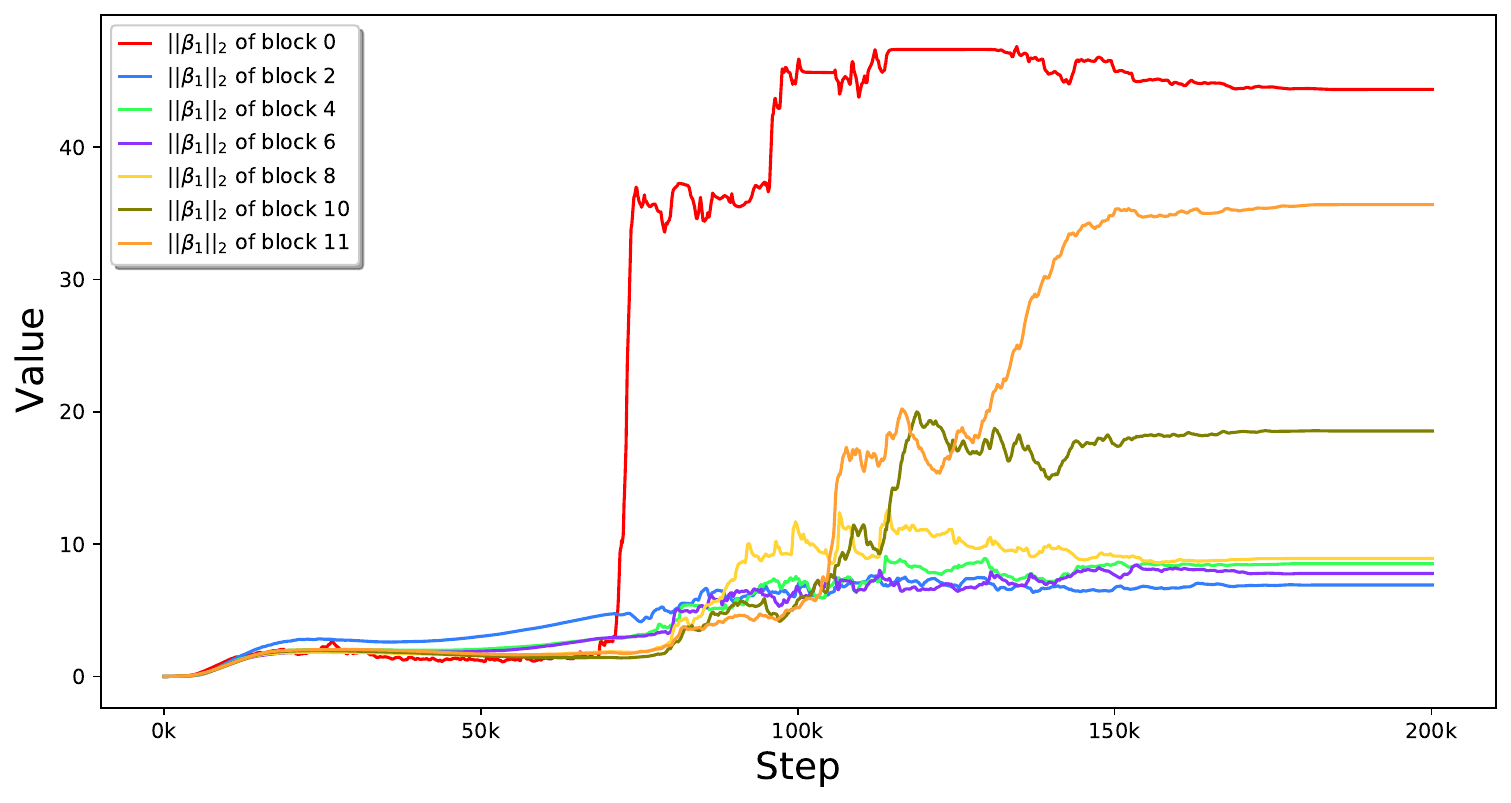}
		\caption*{(b) $\|\boldsymbol{\beta_1}\|_2$}		
	\end{subfigure}
	\begin{subfigure}{0.33\linewidth}
		\centering
		\includegraphics[width=4.5cm,height=3.8cm]{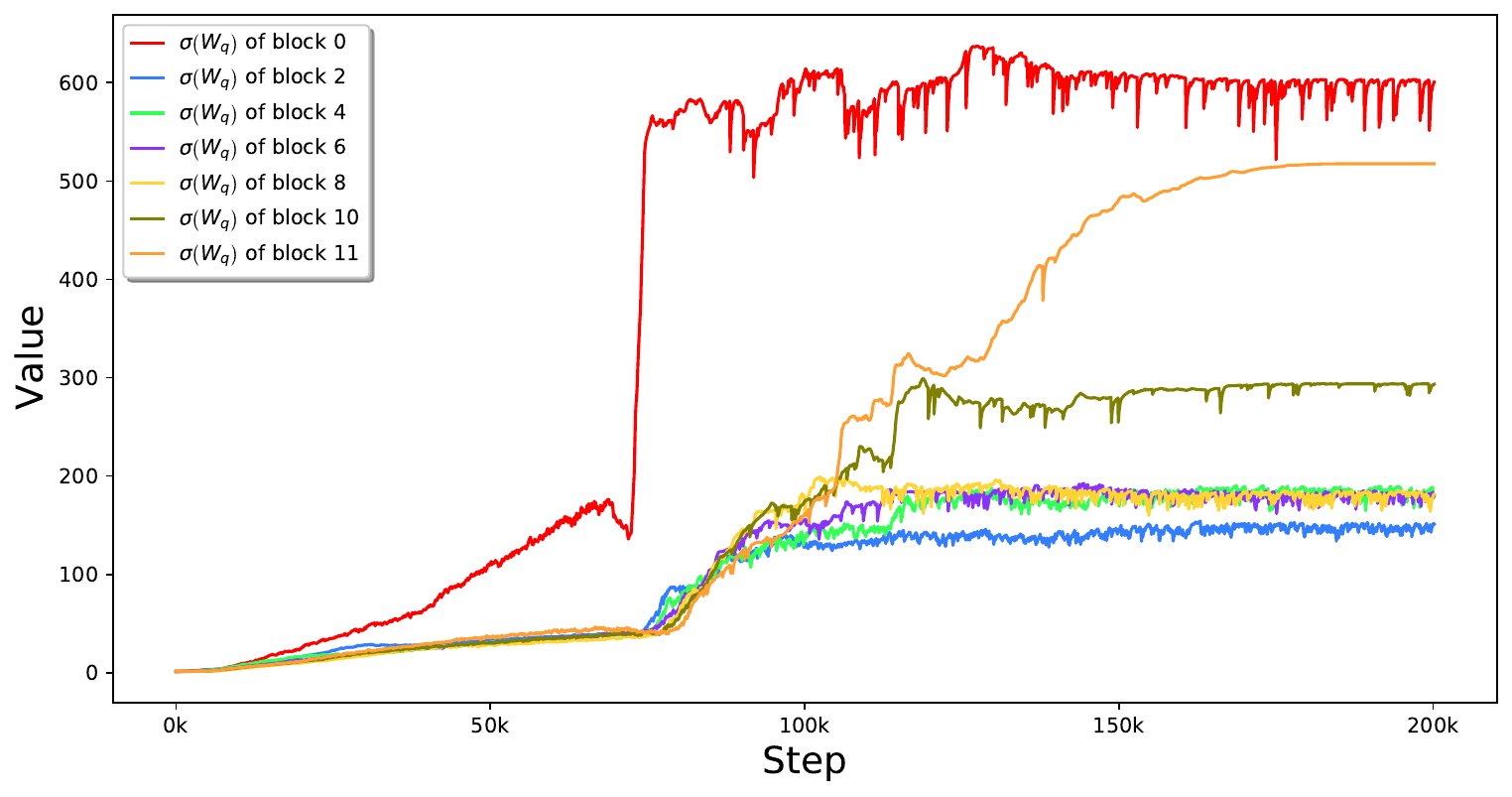}
		\caption*{ (c) $\sigma_1 \left({\bW_q}\right)$}		
	\end{subfigure}
        \begin{subfigure}{0.33\linewidth}
		\centering
		\includegraphics[width=4.5cm,height=3.8cm]{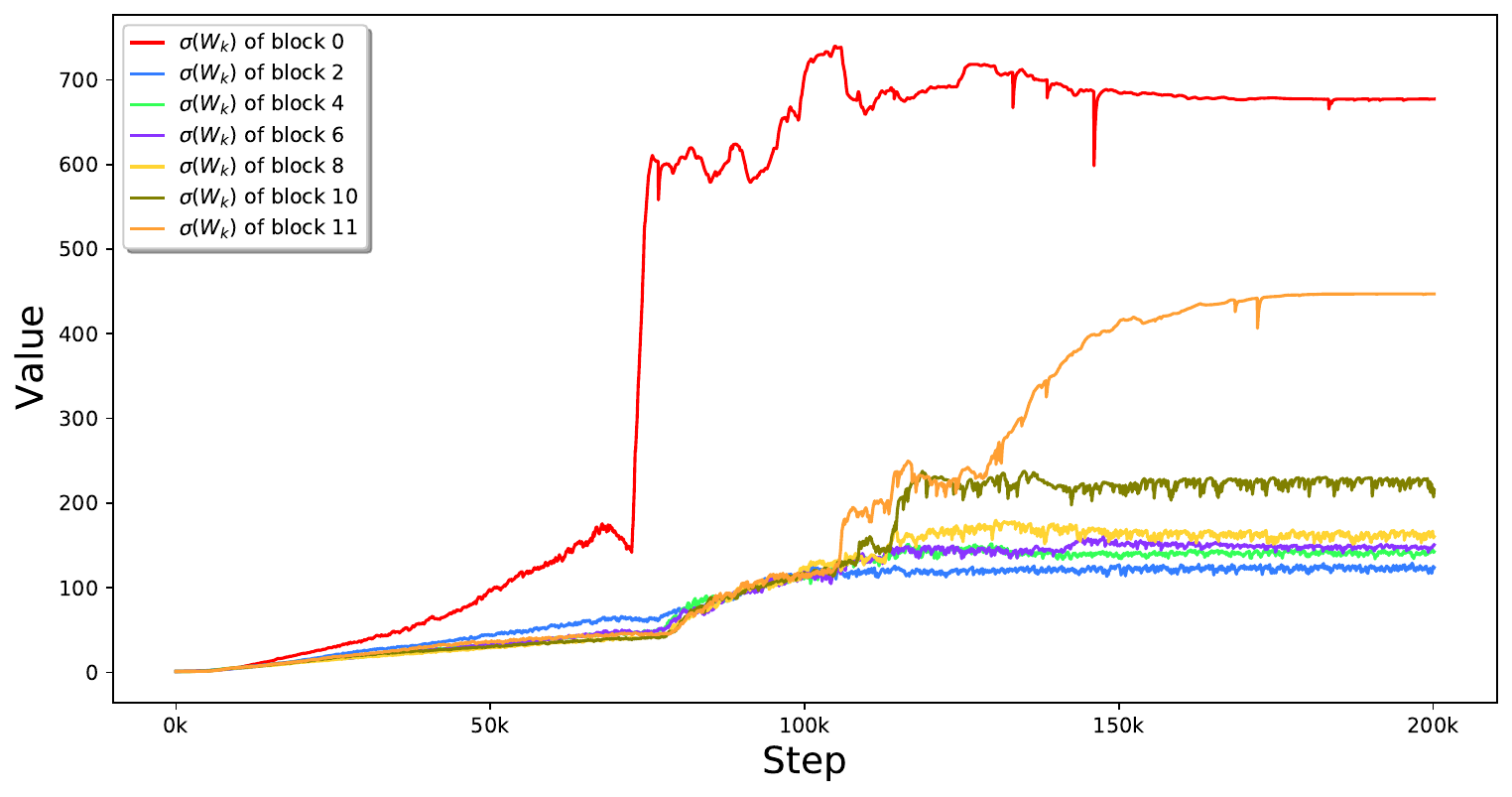}
		\caption*{(d) $\sigma_1\left({\bW_k}\right)$}		
	\end{subfigure}
	\begin{subfigure}{0.33\linewidth}
		\centering
		\includegraphics[width=4.5cm,height=3.8cm]{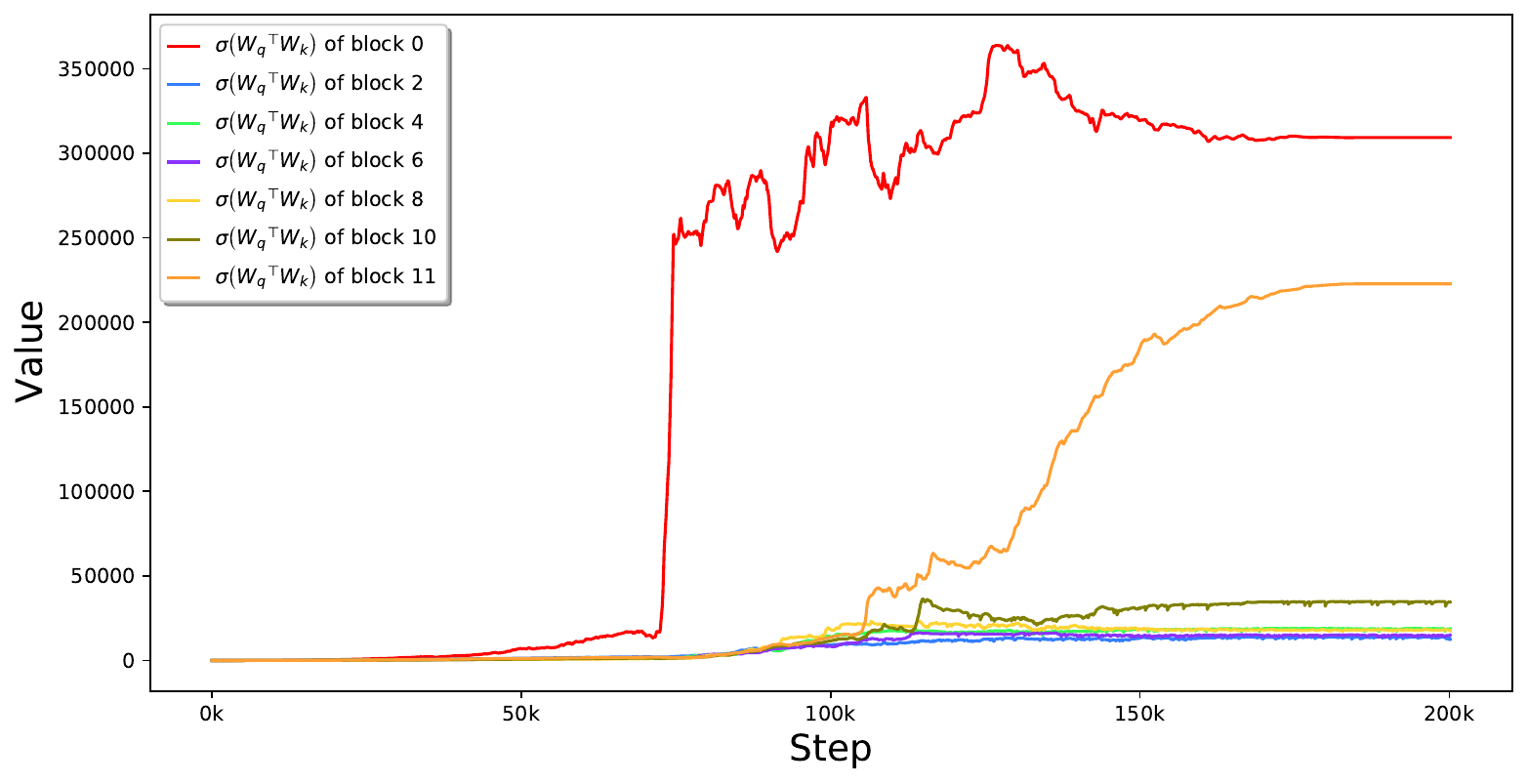}
		\caption*{(e) $\sigma_1\left({{\bW_q}^{\top}{\bW_k}}\right)$}
		\label{fig:vit_sub_success_wqwk}
	\end{subfigure}
	\begin{subfigure}{0.33\linewidth}
		\centering
		\includegraphics[width=4.5cm,height=3.8cm]{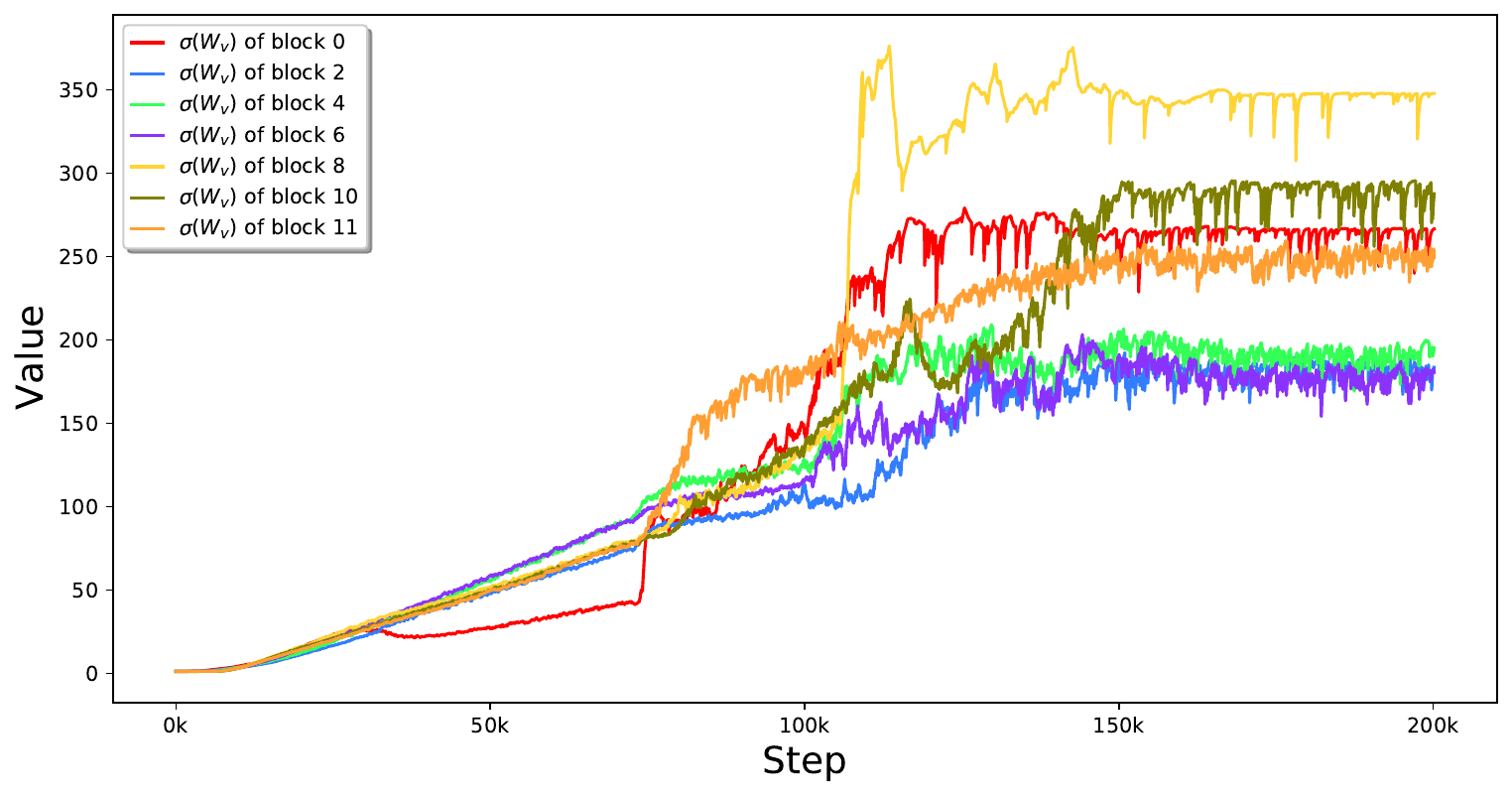}
		\caption*{(f) $\sigma_1\left({\bW_v}\right)$}		
	\end{subfigure}

        \begin{subfigure}{0.33\linewidth}
		\centering
		\includegraphics[width=4.5cm,height=3.8cm]{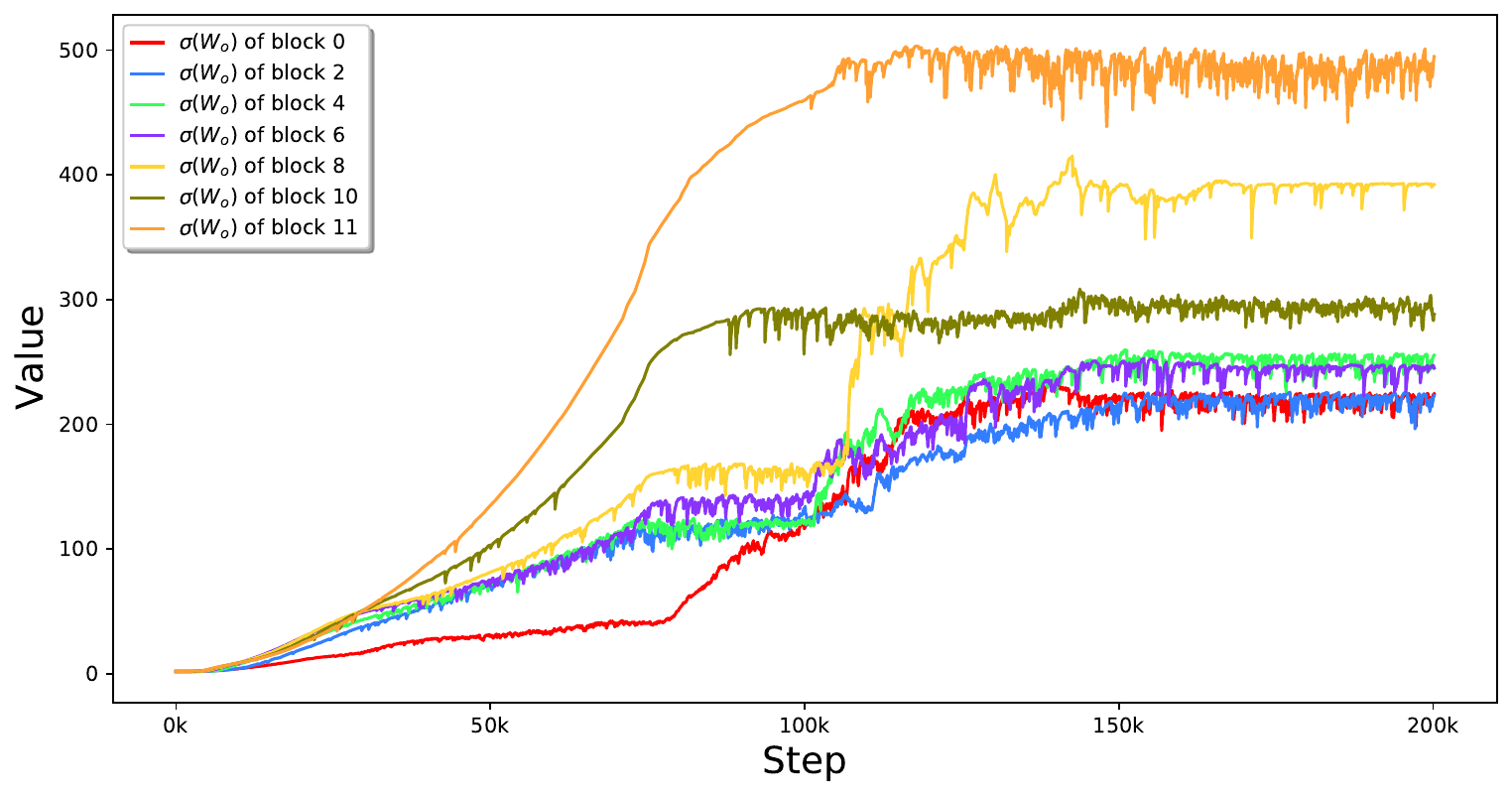}
		\caption*{(g) $\sigma_1\left({\bW_o}\right)$}
		
	\end{subfigure}
	\begin{subfigure}{0.33\linewidth}
		\centering
		\includegraphics[width=4.5cm,height=3.8cm]{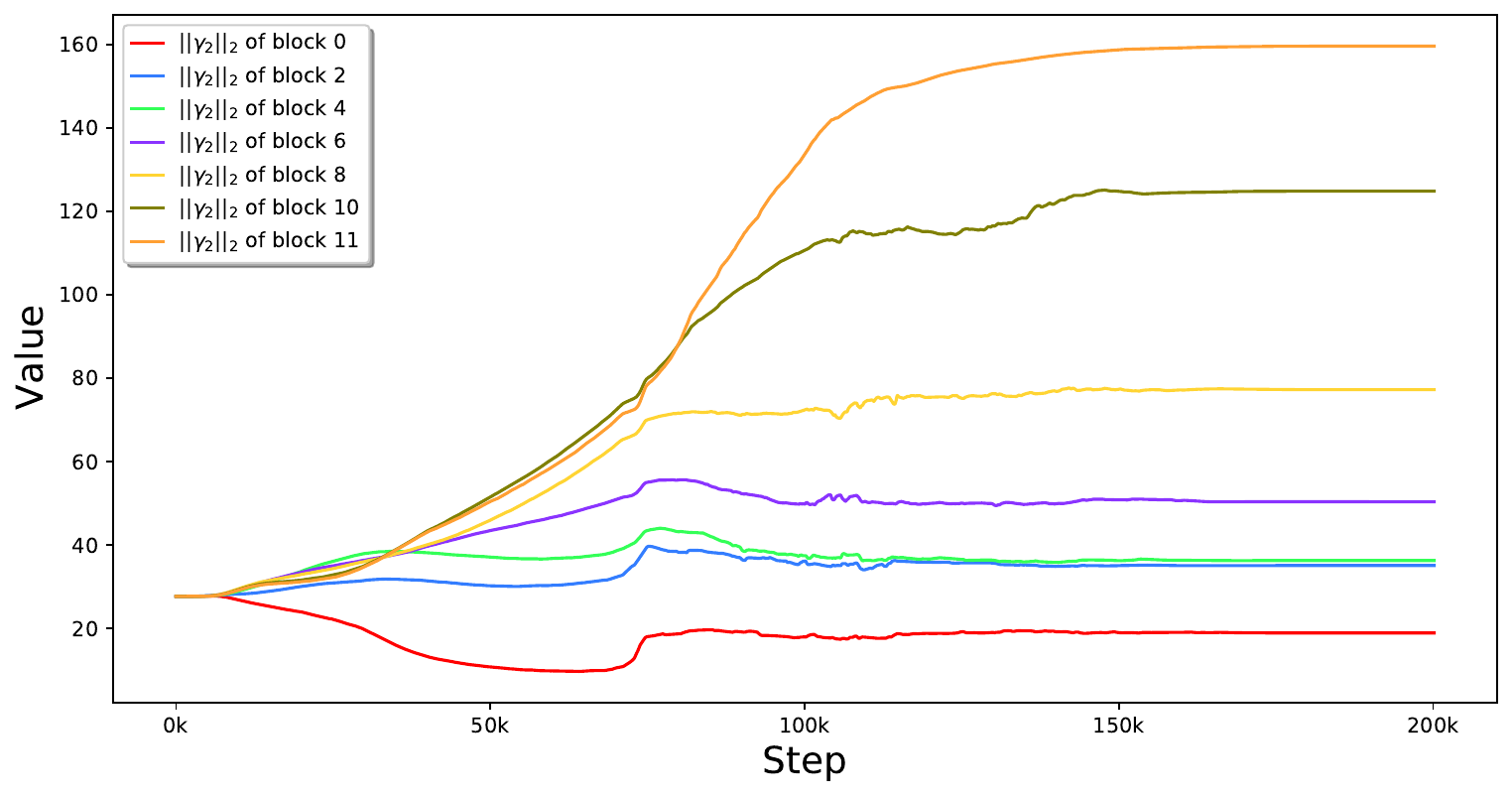}
		\caption*{(h) $\|\boldsymbol{\gamma_2}\|_2$}
		
	\end{subfigure}
	\begin{subfigure}{0.33\linewidth}
		\centering
		\includegraphics[width=4.5cm,height=3.8cm]{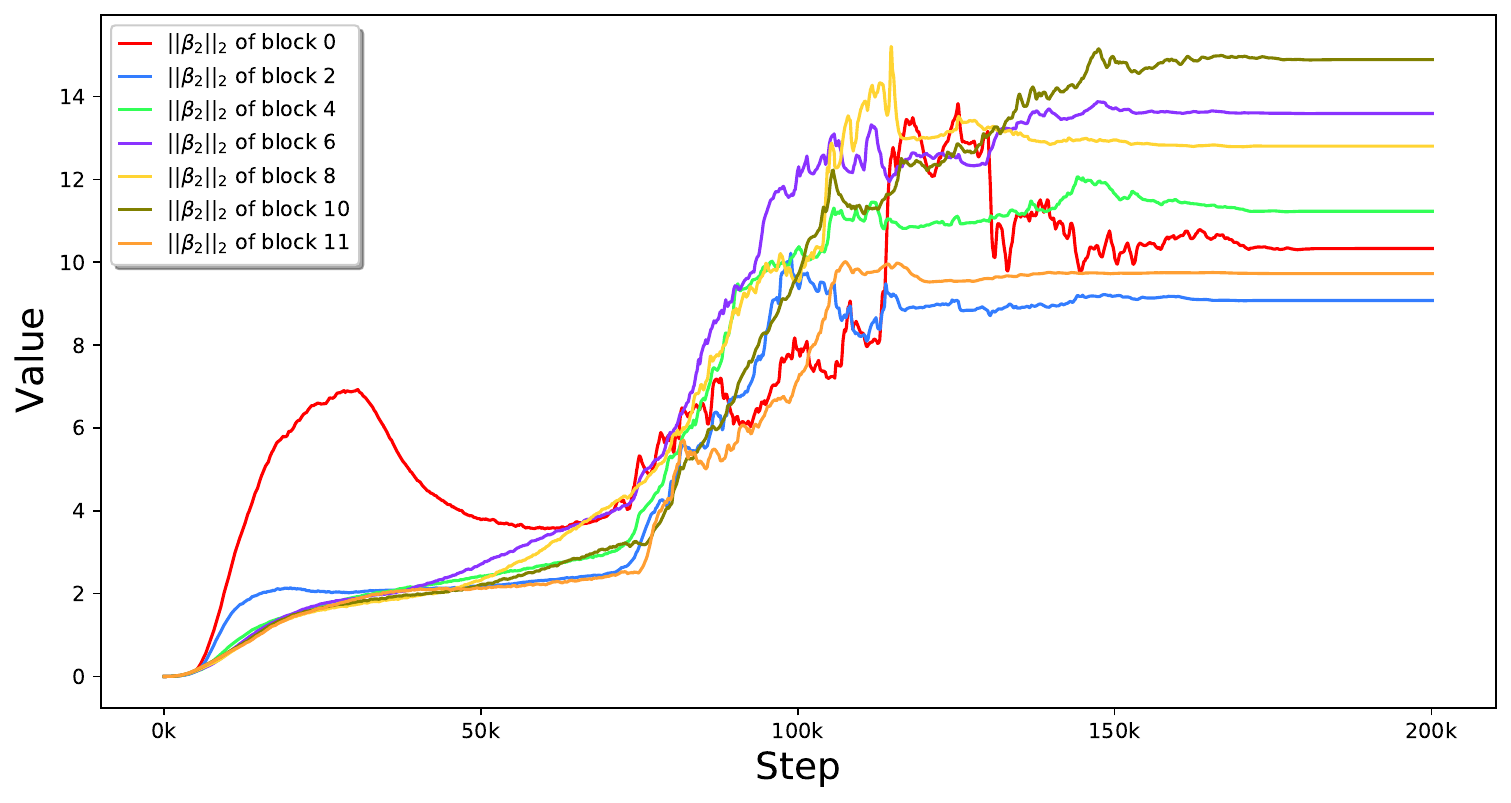}
		\caption*{(i) $\|\boldsymbol{\beta_2}\|_2$}
		
	\end{subfigure}
	
	\begin{subfigure}{0.33\linewidth}
		\centering
		\includegraphics[width=4.5cm,height=3.8cm]{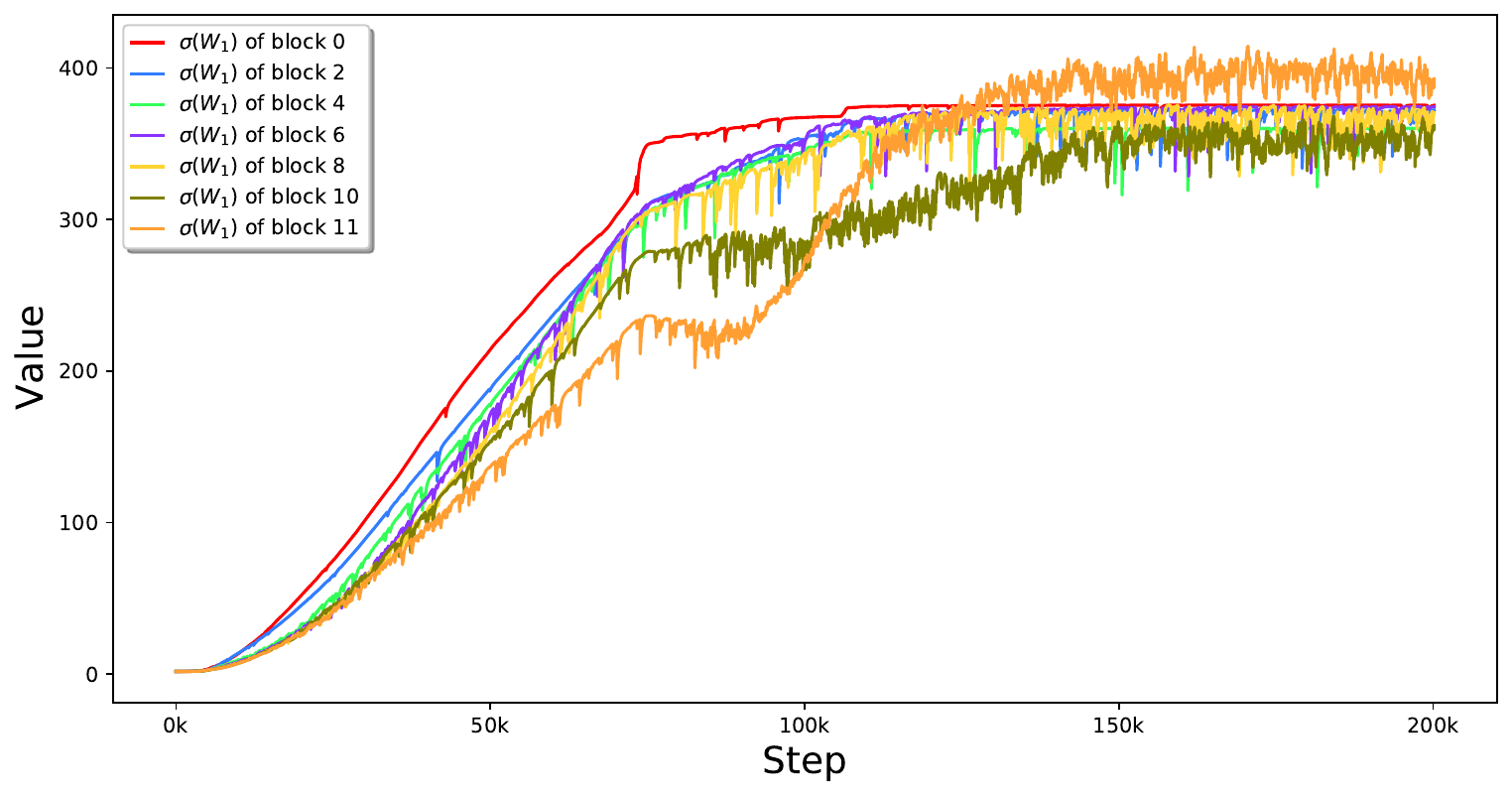}
		\caption*{(j) $\sigma_1\left({\bW_1}\right)$}
		
	\end{subfigure}
	\begin{subfigure}{0.33\linewidth}
		\centering
            \includegraphics[width=4.5cm,height=3.8cm]{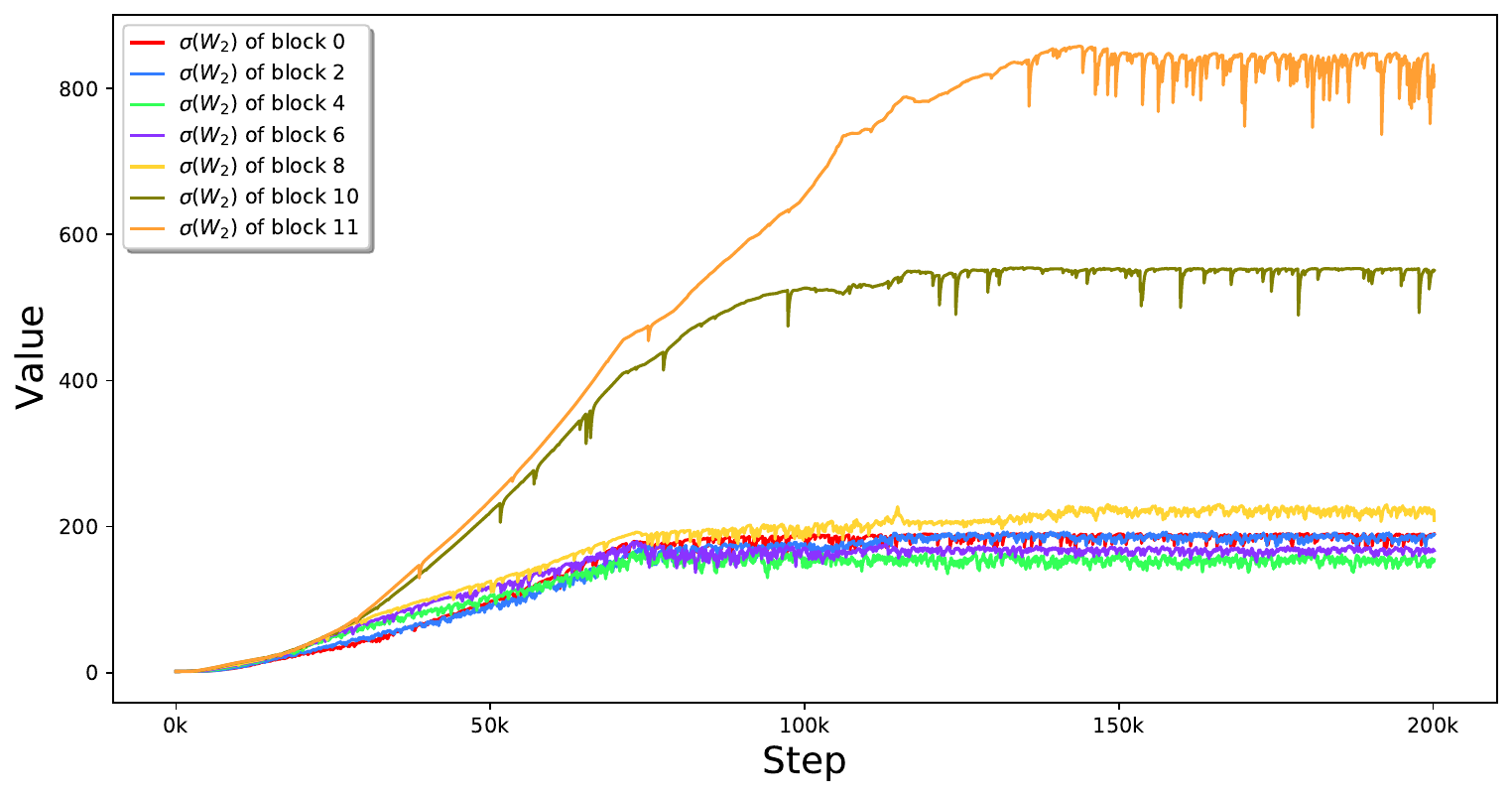}
		\caption*{(k) $\sigma_1\left({\bW_2}\right)$}
		
	\end{subfigure}
        \begin{subfigure}{0.33\linewidth}
		\centering
            \includegraphics[width=4.5cm,height=3.8cm]{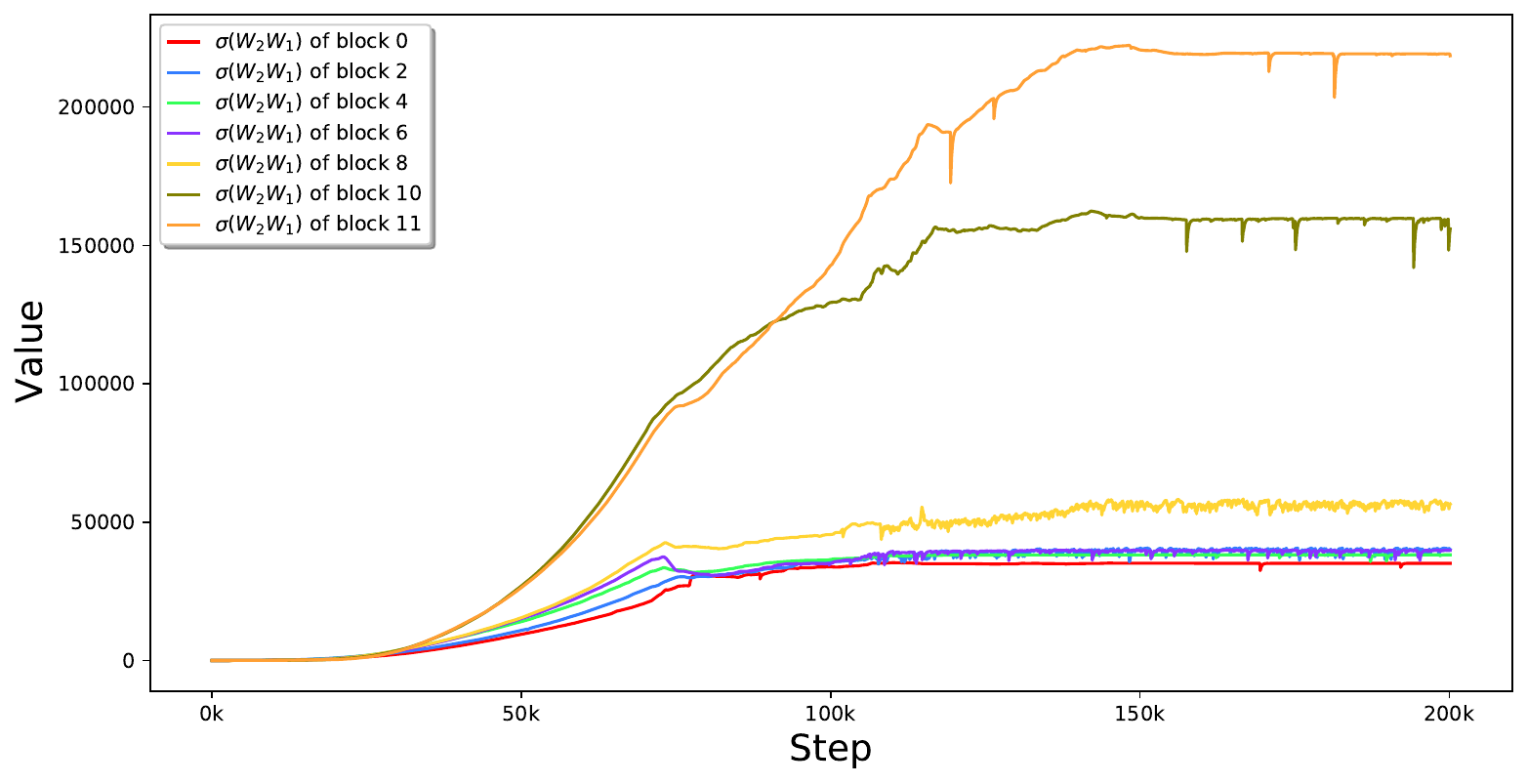}
  
		\caption*{(l) $\sigma_1\left({\bW_2\bW_1}\right)$}
		
	\end{subfigure}
        \begin{subfigure}{0.33\linewidth}
		\centering
            \includegraphics[width=4.5cm,height=3.8cm]{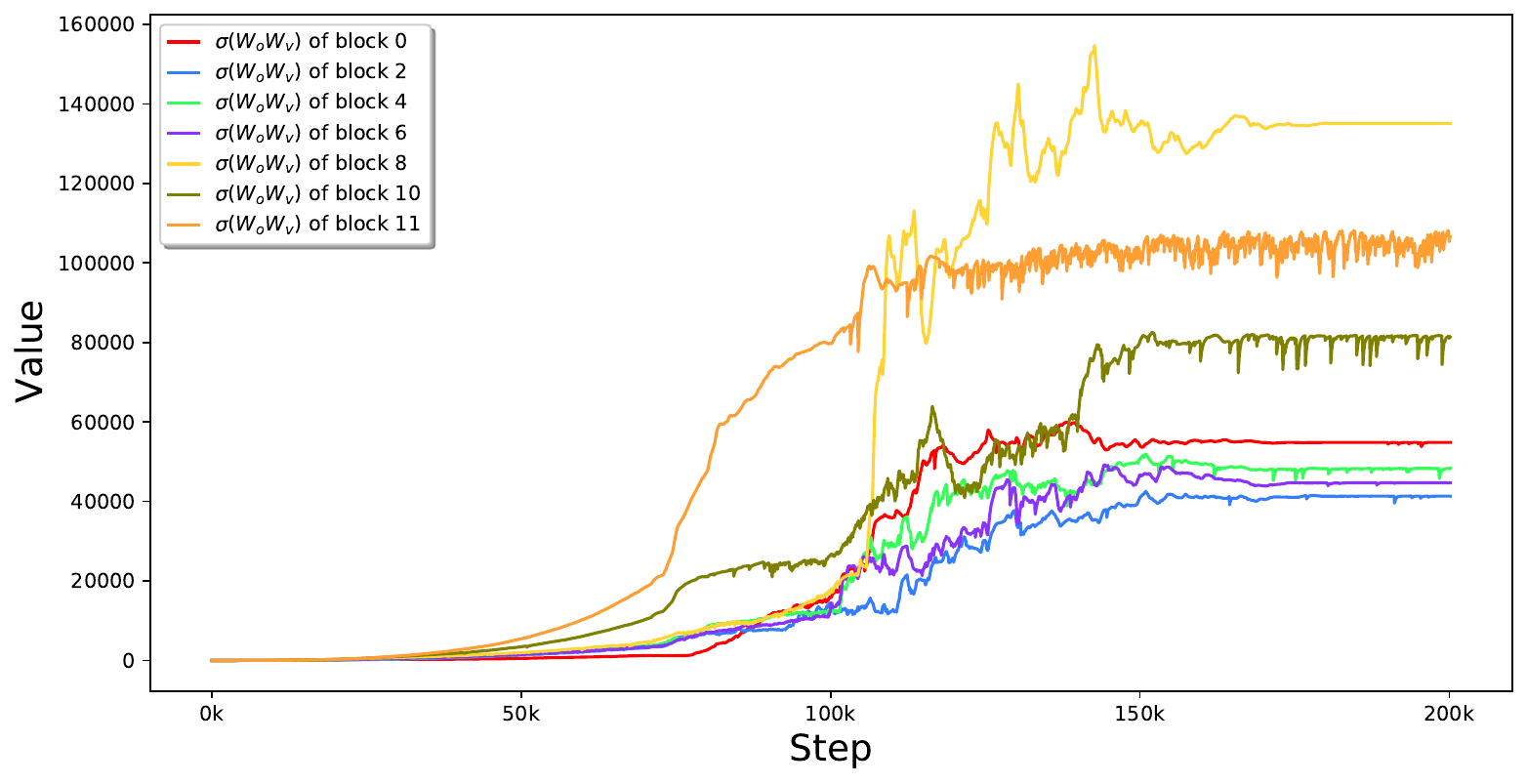}
		\caption*{(m) $\sigma_1\left({\bW_o\bW_v}\right)$}
		
	\end{subfigure}
        \begin{subfigure}{0.33\linewidth}
		\centering
            \includegraphics[width=4.5cm,height=3.8cm]{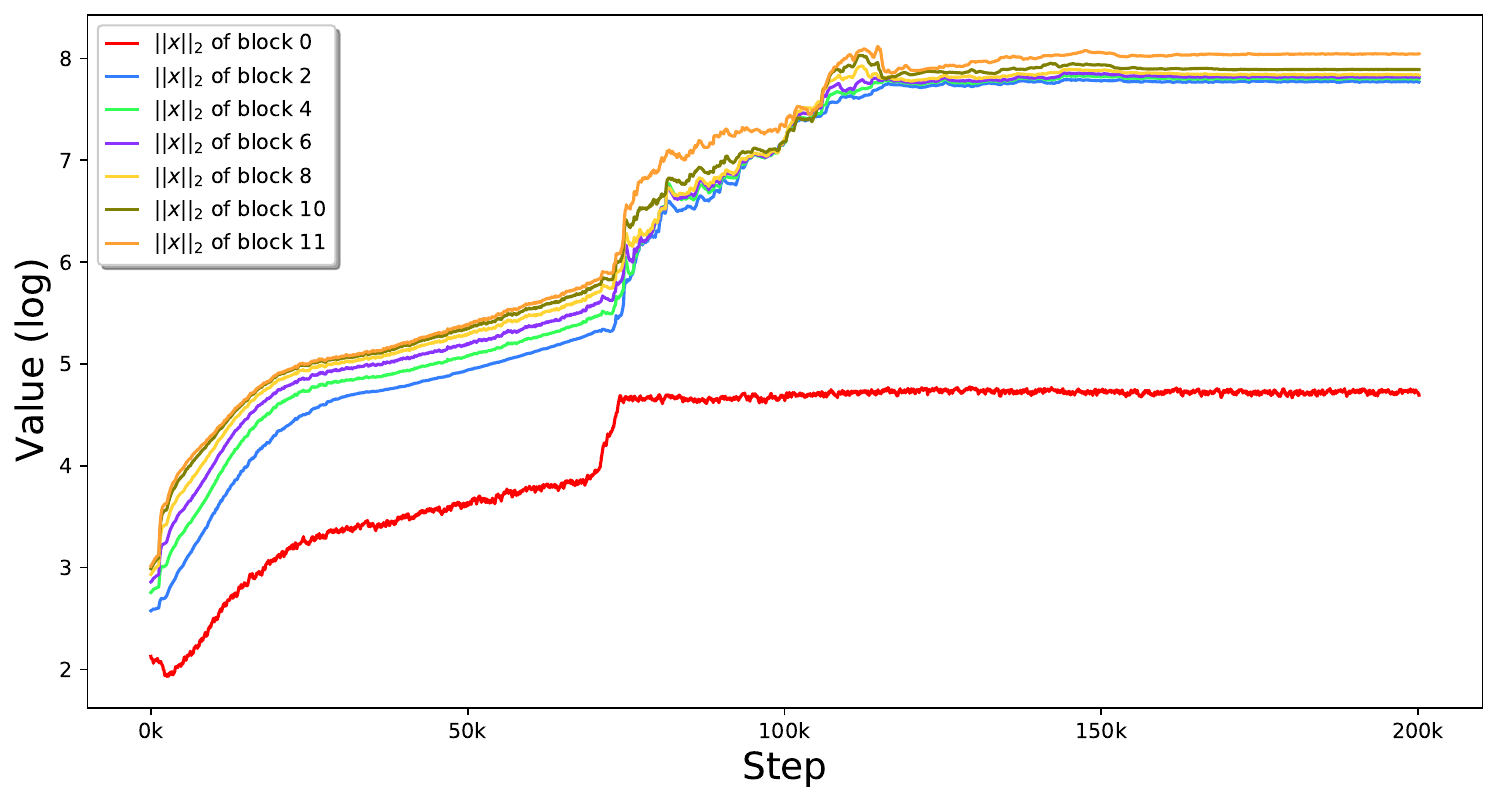}
		\caption*{(n) $\left\|{\bx}\right\|_2$}
		
	\end{subfigure}
        \begin{subfigure}{0.33\linewidth}
		\centering
            \includegraphics[width=4.5cm,height=3.8cm]{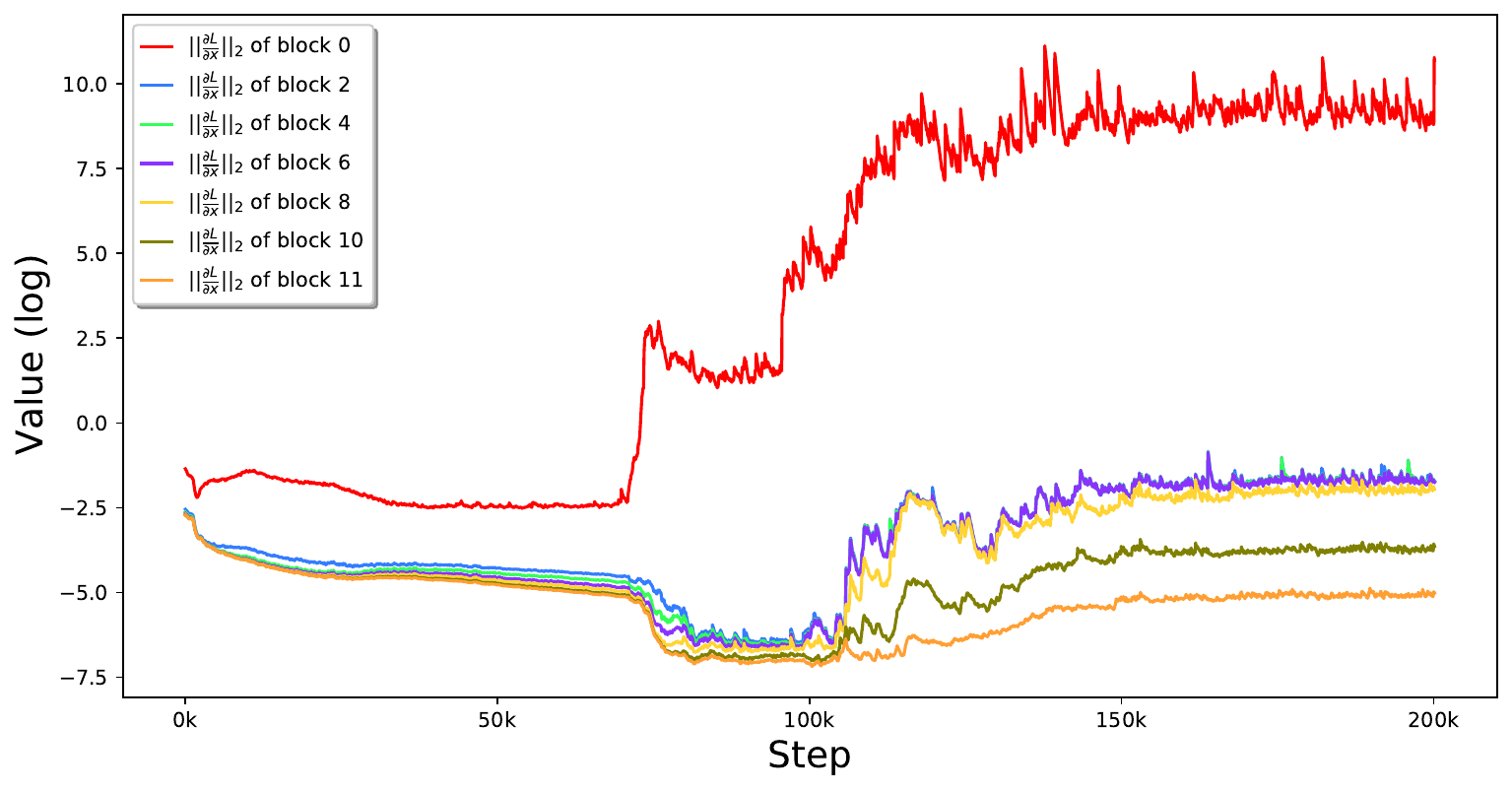}
            \caption*{(o) $\|{\frac{\partial L}{\partial \bx}}\|_2$}
			
	\end{subfigure}
 \caption{Training dynamics of a failure ViT. This figure shows how 
 15 items of quantities as defined 
 in Equation~(\ref{eq:15_terms}) change during the training period. 
 Please pay more attention to subfigures (a)-(e).} 
 \label{fig:why_vit_fail}
\end{figure}

\section{Taming Transformer Requires to Revisit the Training Dynamics}
\label{seq:taming_transformer_requires_revisiting}

To begin with, we first give some basic notions in a Transformer, which includes an attention module, an FFN module and two normalization modules that 
are used before the attention module and the FFN module. 
For the attention module, we usually use a multi-head attention which 
allows the model to jointly attend to information from different representations 
with different heads. Here, for the convenience, 
without losing generality, we only use a single-head attention.
To be precise, we define each of them as follows: 
\begin{align*}
    \operatorname{Attn}(\bX;\bW_q, \bW_k, \bW_v, \bW_o) \quad  &= \quad  \bW_o {\bW_v} \bX \operatorname{softmax}\left( \frac{\bX^{\top} {\bW_q}^{\top} {\bW_k} {\bX}  }{\sqrt{d_q}} \right), \\
    \operatorname{FFN}(\bx;\bW_{1}, \bW_{2} ) \quad  &= \quad \bW_{2} \operatorname{ReLU}(\bW_{1} \bx ), \\
    \operatorname{LN}(\bx)\quad  &= \quad \boldsymbol{\gamma} \odot \bz+\boldsymbol{\beta}, \quad \text{where } \boldsymbol{z} = \frac{\by}{\operatorname{std}(\by)} \text{ and} \quad \by=\left(\bI-\frac{1}{D} \mathbf{1 1}^{\top}\right) \bx,
\end{align*}
where $\bX \in \mathcal{R}^{d \times n}$, $\bW_q \in \mathcal{R}^{d_q \times d}$, $\bW_k \in \mathcal{R}^{d_q \times d}$, $\bW_v \in \mathcal{R}^{d_v \times d}$, $\bW_o \in \mathcal{R}^{d \times d}$, $\bW_1 \in \mathcal{R}^{4d \times d}$, $\bW_2 \in \mathcal{R}^{d \times 4d}$, $\boldsymbol{\gamma} \in \mathcal{R}^{d}$ and $ \boldsymbol{\beta} \in \mathcal{R}^{d}$.
Note that only in a single-head definition, $\bW_o$ can be put before $\bW_v$; otherwise, it should be after a concatenation operator. For the 
convenience of  
discussion, we define:  
\begin{equation*}
    \bA = \operatorname{softmax}(\frac{\bP}{\sqrt{d_q}}) \quad \text{and } \quad  \bP = \bX^{\top} {\bW_q}^{\top} {\bW_k} {\bX},  
\end{equation*}
where $\bA$ is called the attention map and $\frac{\bP}{\sqrt{d_q}}$ is usually called the logit.

\subsection{Visualization: What happens when a Transformer training fails or succeeds}
\label{sec:what_happends_when}
Visualizations are commonly used as an effective means 
to help us examine 
whether the neural network's training succeeds or fails. 
In particular, what happens when the training of a Transformer tends to crash? 
And what happens when the training succeeds?

One of the most important aspects of understanding the training of neural networks is the observation of changes in parameters and activations. Since the parameters or activations and their gradients are 
matrices or vectors, 
the norm is the best way to observe 
these quantities. In this paper, for a Transformer training, we summarize the following {\bf{15 terms}} to watch:
\begin{equation}
\begin{aligned} 
& \sigma_{1}(\bW_q), \quad \sigma_{1}(\bW_k), \quad \sigma_{1}(\bW_v), \quad \sigma_{1}(\bW_o), \quad \sigma_{1}(\bW_{1}), \quad \sigma_{1}(\bW_{2}), \quad
\sigma_{1}({\bW_q}^{\top}{\bW_k}),\\   
& \sigma_{1}({\bW_o}{\bW_v}), \quad \|\boldsymbol{\gamma}_1\|_2, \quad  \|\boldsymbol{\beta}_1\|_2, \quad  \|\boldsymbol{\gamma}_2\|_2, \quad  \|\boldsymbol{\beta}_2\|_2,\quad  
 \sigma_{1}({\bW_2}{\bW_1}), \quad \|\bx\|_2, \quad \left\|\frac{\partial L}{\partial \bx}\right\|_2,
 \end{aligned}
 \label{eq:15_terms}
 \end{equation}
 where 
 $\boldsymbol{\beta_1}$ and $\boldsymbol{\gamma_1}$ denote the parameters of the first LayerNorm, and $\boldsymbol{\beta_2}$ and $\boldsymbol{\gamma_2}$ denote the parameters of the second LayerNorm. 
 When $\operatorname{RMSNorm}$~\citep{rmsnorm_zhang2019root} is used, there are only 13 terms that will be analyzed since it does not have $\boldsymbol{\beta}$. For the weight matrix, we use the spectral norm. 
 For a vector, we use its $\ell_2$ norm.

To ensure 
that the phenomena we observed can generalize well, we visualized them on both ViT~\citep{vit_dosovitskiy2020image} and GPT~\citep{gpt1_radford2018improving}. 
ViT is a pure encoder architectures whereas GPT is a pure decoder architecture.
 
\begin{wrapfigure}{l}{0.70\textwidth}
  \vspace{0pt}
  \centering
\begin{subfigure}{0.23\textwidth}
      \animategraphics[poster=last, width=1\textwidth]{3}{final_figures/baseline_short/block0_head6/image_baseline_}{1}{15}
      \subcaption*{(a) \emph{\footnotesize Block 0 (successful).}}     
\end{subfigure}
\begin{subfigure}{0.23\textwidth}
      \animategraphics[poster=last, width=1\textwidth]{3}{final_figures/baseline_short/block6_head6/image_baseline_}{1}{15}
      \subcaption*{(b) \emph{\footnotesize Block 6 (successful).}}
\end{subfigure}
\begin{subfigure}{0.23\textwidth}
      \animategraphics[poster=last, width=1\textwidth]{3}{final_figures/baseline_short/block11_head6/image_baseline_}{1}{15}
      \subcaption*{(c) \emph{\footnotesize Block 11 (successful).}}
\end{subfigure}
\begin{subfigure}{0.23\textwidth}
      \animategraphics[poster=last, width=1\textwidth]{3}{final_figures/collapse_short/block0_head6/image_collapse_}{1}{15}
      \subcaption*{(d) \emph{\footnotesize Block 0 (crashed).}}     
\end{subfigure}
\begin{subfigure}{0.23\textwidth}
      \animategraphics[poster=last, width=1\textwidth]{3}{final_figures/collapse_short/block6_head6/image_collapse_}{1}{15}
      \subcaption*{(e) \emph{\footnotesize Block 6 in (crashed).}}
\end{subfigure}
\begin{subfigure}{0.23\textwidth}
      \animategraphics[poster=last, width=1\textwidth]{3}{final_figures/collapse_short/block11_head6/image_collapse_}{1}{15}
      \subcaption*{(f) \emph{\footnotesize Block 11 (crashed).}}
\end{subfigure}
\caption{Visualization of the dynamics process of attention map in different training steps for a successful and a crashed ViT-base model, respectively. \emph{{Please click the images to play the flash.} Best viewed with Acrobat Reader.}}
\label{fig:attention_map_gif}
\end{wrapfigure}

In this section, due to the space limitation, 
we will only visualize ViT-base, and put more visualization results into the Appendix~\ref{appendix:visualization_gpt}. 
For the ViT implementation, we use Timm~\cite{timm_rw2019timm}, in which ``timm'' library provides rich model architectures of many pre-trained image models in PyTorch. 
For the GPT implementation, we use nanoGPT, which uses LayerNorm without a bias term, and 
thus only watch 
13 terms, rather than 
15 terms in ViT.

To achieve a successful ViT training, we use a long learning rate warmup. For instance, we use 60 epochs of warmup and the whole training process takes 150 epochs. To obtain the dynamics of a failure training of ViT, we do not use warmup.

Figure~\ref{fig:why_vit_fail} visualizes a failure training process of a ViT-base model.
The model includes 12 blocks with index from 0 to 11. In Figure~\ref{fig:why_vit_fail}, we visualize the weight matrices in blocks $\{0, 2, 4, 6, 8, 10, 11\}$, where 
the index of the last block is 11. For the features $\bx$ and the gradients $\frac{\partial L}{\partial \bx}$, we hook the input features that enter into the corresponding blocks. 
Figure~\ref{fig:why_vit_succusss} in the Appendix~\ref{sec:appendix_vit_gpt} visualizes a successful training process of a ViT-L model. Meanwhile, in Figure~\ref{fig:attention_map_gif}, we visualize the dynamic processes of attention maps during the training period about a successful case and a failure case of the ViT-base model, respectively. 
We visualize the attention map of the same image at different steps.

We observe the following phenomena from Figures~\ref{fig:why_vit_fail} and \ref{fig:attention_map_gif}. 
\begin{itemize}[leftmargin=*]
    \item As shown in Figure~\ref{fig:why_vit_fail}, at the beginning, the maximum singular value 
    $\sigma_1({\bW_q}^{\top}{\bW_k})$ gradually increases, and at a certain point, the maximum singular value suddenly and rapidly increases to a very large value (\eg, around 200,000), at where the loss divergence emerges. 
    However, for a successful training process, $\sigma_1({\bW_q}^{\top}{\bW_k})$ gradually increases to a medium value as shown in Figure~\ref{fig:why_vit_succusss} and then vibrates around that value.
    
    \item As shown in Figure~\ref{fig:attention_map_gif}, in a failure training process of Transformer, the attention maps gradually become sparse and low-rank. Finally, 
    the entropy of the attention map is 0; whereas 
    in a successful training process, the attention map is not too sparse and of a medium rank.
    
    \item The normalization layers exhibit a huge difference: $\boldsymbol{\gamma_1}$ and $\boldsymbol{\beta_1}$ in a successful ViT training process are very smooth, but $\|\boldsymbol{\gamma_1}\|_2$ and $\|\boldsymbol{\beta_1}\|_2$ 
    suddenly increase dramatically in an unsuccessful case.
    
    \item The ranges of the activation and the gradients are very large in a crashed model, and the gradients in different blocks vary much larger than that in a successful model. 
\end{itemize}

\myparagraph{Remark} We summarize that the successful training and the unsuccessful training of a Transformer exhibit significant differences among their $\sigma_1({\bW_q}^{\top}{\bW_k})$, their normalization parameters $\boldsymbol{\gamma}$ and $\boldsymbol{\beta}$, and their activations $\bx$ and gradients $\frac{\partial L}{ \partial \bx}$.

\subsection{Theoretical Analysis: Matrix Calculus of Transformer}
\label{sec:matrix_calculus_transformer}
To understand the training dynamics of Transformer, we should investigate the process of back-propagation~\citep{bp_rumelhart1986learning, efficient_bp_lecun2002efficient, lecun1989backpropagation, lecun1998gradient}. 
In the attention mechanism, however, the input and the output are both 
matrices, we cannot directly use vector calculus. Instead, we need to use \textit{Vectorization}~\footnote{\url{https://en.wikipedia.org/wiki/Vectorization_(mathematics)}}~\citep{kronecker_product_and_matrix_calculus_graham2018kronecker, matrix_cookbook_petersen2008matrix} and \textit{Kronecker Product}~\footnote{\url{https://en.wikipedia.org/wiki/Kronecker_product}}~\citep{kronecker_product_and_matrix_calculus_graham2018kronecker, matrix_cookbook_petersen2008matrix}. 
In matrix calculus, the vectorization of a matrix is a linear transformation 
that converts a matrix into a vector. 
Specifically, the vectorization of a matrix $\bM \in \mathcal{R}^{m\times n}$, denoted as $\text{vec}(\bM)$, is a column vector, obtained by an ordered stacking of the columns of the matrix $\bM$, \ie, $\text{vec}(\bM) \in \mathcal{R}^{mn}$. 
For example, for a $2\times3$ matrix $\bM = \begin{bmatrix} a & b & c\\ d & e & f \end{bmatrix}$, the vectorization of $\bM$ is $\text{vec}(\bM) = [a\ d\ b\ e\ c\ f]^{\top}$.
For the attention module of Transformer, we have the following proposition about the Jacobian matrix of the output $\bP$ with respect to the input $\bX$ and the parameters.
\begin{proposition}[Matrix Calculus for Self-Attention]
\label{proposition_selfattention}
Let $\bP = {\bX^{\top} {\bW_q}^{\top} {\bW_k} {\bX}  }$, where $\bX \in \mathcal{R}^{d\times n}, \bW_q \in \mathcal{R}^{d_q\times d}, \bW_k \in \mathcal{R}^{d_q\times d}$, according to vectorization and matrix calculus, we have the following derivations:
\begin{align}
 \frac{\partial \operatorname{vec}(\bP)}{\partial \operatorname{vec}({\bW_q}^{\top}{\bW_k} )}   &=  \bX^{\top} \otimes \bX^{\top}, \quad
 \frac{\partial \operatorname{vec}(\bP)}{\partial \operatorname{vec}(\bX)}   =  (\bX^{\top}{\bW_k}^{\top}{\bW_q} \otimes \bI_n)\bK + (\bI_n \otimes \bX^{\top}{\bW_q}^{\top}{\bW_k}), \label{eq:dp_dwqk_dx} \\
\frac{\partial \operatorname{vec}(\bP)}{\partial \operatorname{vec}(\bW_q^{\top})}    &=   (\bW_k\bX)^{\top} \otimes \bX^{\top}, \quad \frac{\partial \operatorname{vec}(\bP)}{\partial \operatorname{vec}(\bW_k)}    =  \bX^{\top} \otimes (\bW_q \bX)^{\top}, \label{eq:dp_dwk}
\end{align}
\noindent where $\otimes$ denotes the Kronecker product, $\bI_n \in \mathcal{R}^{n\times n}$ denotes an 
identity matrix with shape $n\times n$, $\bK$ is the commutation matrix, which depends on the dimensions of $\bX$. Since $\bX \in \mathcal{R}^{d\times n}$, then we know $\bK \in \mathcal{R}^{nd\times nd}$. The commutation matrix $\bK$ has the property that $\operatorname{vec}(\bX^\top) = \bK \operatorname{vec}(\bX)$ for any matrix $\bX$.
\end{proposition}

In Appendices~\ref{appendix:kronecker_product} and \ref{appendix:single_head_attention}, we supply some elementary 
information for the vectorization and the Kronecker product and the derivations of the Jacobian matrix for a single-head attention.

We have the following observations from Proposition~\ref{proposition_selfattention}. 
\begin{itemize}[leftmargin=*]

\item  We have $\frac{\partial \operatorname{vec}(\bP)}{\partial \operatorname{vec}({\bW_q}^{\top}{\bW_k} )} = \bX^{\top} \otimes \bX^{\top}$ in Equation~\ref{eq:dp_dwqk_dx}, and we know 
about the Kronecker product that $\operatorname{rank}(\bX^{\top} \otimes \bX^{\top}) = {\operatorname{rank}(\bX^{\top})}^2$, which 
implies that if the rank of $\bX$ 
is very low, then the rank of $ \frac{\partial \operatorname{vec}(\bP)}{\partial \operatorname{vec}({\bW_q}^{\top}{\bW_k} )}$ will also be very low. 
Note that $\bX$ being low rank means that the features across different timestep are highly correlated or coherent. 
If all $\bx_i$ in $\bX$ collapses to a single 
point, then ${\bX^{\top} \otimes \bX^{\top}}$ will only have a large singular value, and the rest are 0.

\item The Jacobian matrix $\frac{\partial \operatorname{vec}(\bP)}{\partial \operatorname{vec}(\bX)}$ in Equation~\ref{eq:dp_dwqk_dx}, is in direct proportion to $\bX$ and ${\bW_q}^{\top} \bW_k$. If the spectral norm $\sigma_1({\bW_q}^{\top} \bW_k)$ is very large, it 
implies that the gradient $\frac{\partial L}{\partial \bX}$ will more likely 
to be magnified a lot. 
     
\item Equation~\ref{eq:dp_dwk} suggests that changes in the query weights $\bW_q$ are related to both the input $\bX$ and the key representation $\bW_k \bX$. Equation~\ref{eq:dp_dwk} suggests that changes in the key weights $\bW_k$ are related to both the input $\bX$ and the query representation $\bW_q \bX$.
     
\item All these relationships are interconnected, with changes in one variable potentially affecting the others. 
For instance, if $\bW_k$ increases fast, then according to Equation~\ref{eq:dp_dwk}, $\frac{\partial \operatorname{vec}(\bP)}{\partial \operatorname{vec}(\bW_q^{\top})}$ will 
more likely to be very large. In this way, $\bW_q$ will likely 
increase very fast.
\end{itemize}

\subsection{Rationale in Model Crash: Spectral Energy Concentration}
\label{subsection:energy_concerntration}
Before we reveal the rationale 
in model crash, let us first discuss a bit on the entropy collapse. 
In~\citep{stabilizing_transformer_zhai2023stabilizing}, the entropy of an attention map $\bA$ is defined as $\mathcal{E}(\bA) =- \frac{1}{n}\sum_{j=1}^{n} \sum_{i=1}^{n} A_{i, j} \log \left(A_{i, j}\right)$.
The so-called {\emph{entropy collapse} refers to the fact that the entropy $\mathcal{E}(\bA)$ of the attention probability matrix is almost 0.}

\textbf{Two Entropy Collapse Modes.} {In experiments, we observe two types of attention entropy collapse modes.  Note that when attention collapse happens, the attention map tends to a sparse matrix (\ie, there are a few dominate nonzero coefficients in attention map), and thus the entropy of the attention map is vanishing. To be more specific, when the attention map is sparse but not low-rank, we call it a \emph{benign collapse}; whereas if the attention map is sparse and simultaneously low-rank,\footnote{For a sparse and simultaneously low-rank matrix, we refer the reader to a recent textbook~\citep{wright2022high}.} we call it a \emph{malignant collapse}.
When benign collapse occurs, the attention map is shown in the right panel of Figure~\ref{fig:three_attention_maps} that, there 
is almost an identity matrix. In this way, the diagonal elements are almost 1, and the non-diagonal elements have values around 0.
Unfortunately, when malignant collapse happens, the attention 
map 
becomes a sparse and simultaneously low-rank matrix as shown in the middle panel of 
Figure~\ref{fig:three_attention_maps}. 
Furthermore, we observe that the distribution of the spectral energy of ${\bW_q}^{\top} \bW_k$ for the benign collapse is relatively uniform; whereas the spectral energy of the attention matrix for the malignant collapse tends to concentrate on a few dominate singular values.}
\begin{wrapfigure}{l}{0.68\textwidth}
  \vspace{0pt}
  \centering
	\begin{subfigure}{0.32\linewidth}
		\includegraphics[width=1\textwidth]{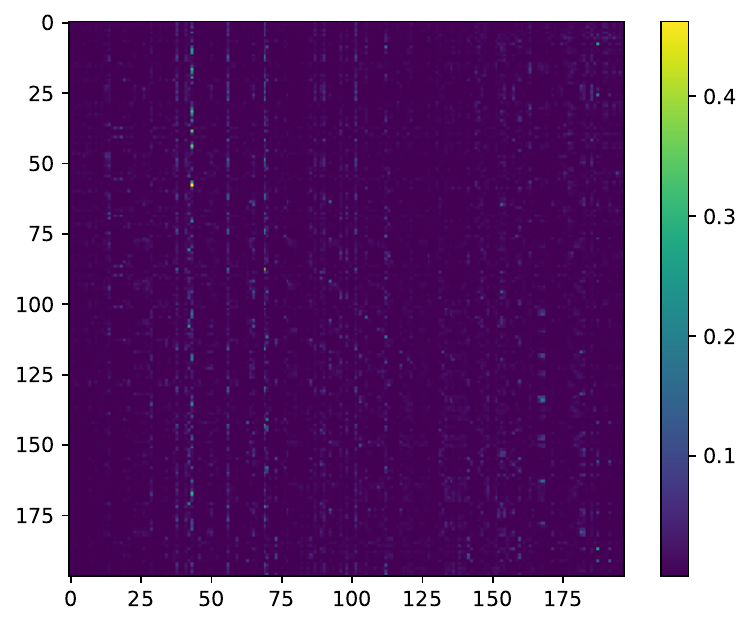}
            \subcaption*{\emph{\scriptsize (a) Normal attention}}     
		\label{chutian1}
	\end{subfigure}
	\begin{subfigure}{0.32\linewidth}
		\includegraphics[width=1\textwidth]{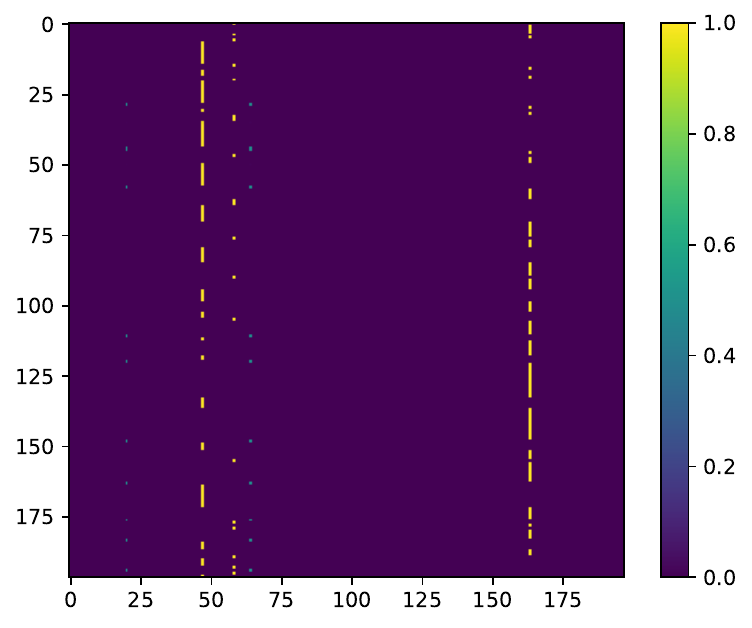}
            \subcaption*{\emph{\scriptsize (b) Malignant collapse}}
		
	\end{subfigure}
	\begin{subfigure}{0.32\linewidth}
		\includegraphics[width=1\textwidth]{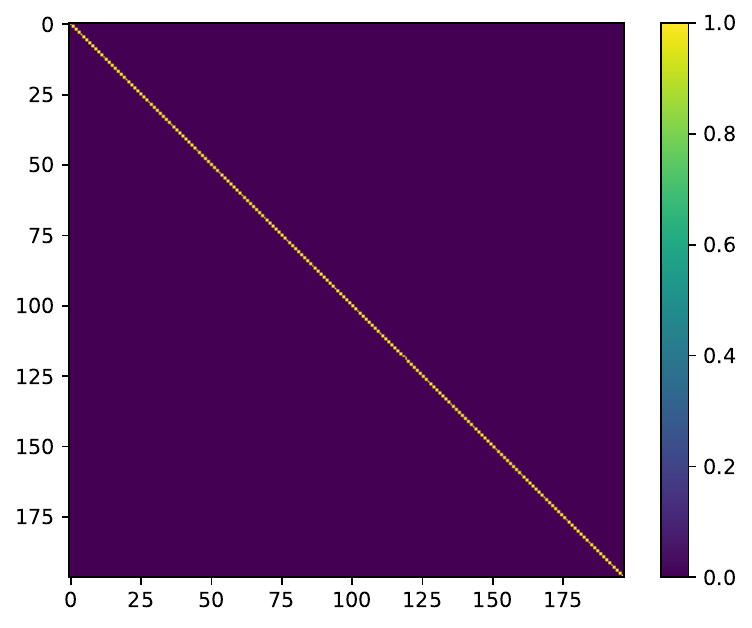}
            \subcaption*{\emph{\scriptsize (c) Benign collapse}}
		
	\end{subfigure}
 \caption{Three modes of attention maps. The left panel shows a normal attention map. The middle panel shows a classical attention map when model crash occurs, for which the entropy is almost 0. The right panel shows an attention map from a normal model training but its entropy is almost 0.} 
  \vspace{-20pt}
   \label{fig:three_attention_maps}
\end{wrapfigure}

By analyzing the matrix ${\bW_q}^{\top} \bW_k$ in the benign collapse when it happens in the experiments, we find that it has the following property: ${\bW_q}^{\top} \bW_k$ is usually a non-symmetric positive quasi-definite square matrix (see Appendix~\ref{appendix:non_symmetric} for details), 
and the self-attention layer degenerates into a linear projection layer because $ \bY = \bW_v \bX \bA \approx \bW_v \bX \bI = \bW_v \bX$. We give an intuitive analysis in Appendix~\ref{appendix:benign_entropy_collapse}.

When the malignant collapse happens, the model will usually crash. We identify that the 
rationale behind the model crash is a phenomenon called \textit{spectral energy concentration} (\textit{SEC}). 
Before we present our theorem about SEC, let us first introduce an index to 
quantify it. 
Recall that $\bW_q \in \mathcal{R}^{d_q \times d}$ and $\bW_k \in \mathcal{R}^{d_q \times d}$, where $d_q < d$. We have that ${\bW_q}^{\top} \bW_k \in \mathcal{R}^{d \times d}$, but its rank is less than or equal to $d_q$.  
Precisely, we define a \textit{SEC index} as follows:
\begin{equation}
\operatorname{SEC}(d_q, s) = \frac{\sum_{i=1}^s { \sigma_i^2({{\bW_q}^{\top} \bW_k})}}{\sum_{i=1}^{d_q} {\sigma_i^2({{\bW_q}^{\top} \bW_k})}},
\end{equation}
where $d_q$ is the head dimension and $s \le d_q$. 
For instance, if we have $d_q = 64$ and $s=4$, and if at this time, $\operatorname{SEC}(64, 4) > 99\%$, we could say the spectral energy of ${\bW_q}^{\top} \bW_k$ highly concentrates on only four dominant singular values.

{
To be precise, we have the following theorem for the reason to cause a malignant collapse. 
\begin{theorem}[Malignant Entropy Collapse]
\label{theorem:sec_leads_to_crash}
Let
$\bP = \bX^T \bW \bX \quad \text{and} \quad \bA = \operatorname{softmax}\left(\frac{\bP}{\sqrt{d_q}}\right), \quad \bW = \bW_q^T \bW_k \in \mathbb{R}^{d \times d}.$
Suppose that the following two conditions are simultaneously satisfied:
\begin{itemize}[leftmargin=*]
    \item  \( \bX \) is a low-rank matrix;
    \item  \( \bW \) is a low-rank matrix with only a few dominant singular values (e.g., the singular values are greater than \( C_0 \cdot \sqrt{d_q} \) where \( C_0 \gg 1 \) is a constant).
\end{itemize}
Then, the attention map \( \bA \) will be sparse 
and simultaneously low-rank in high probability. 
\end{theorem}

When the malignant entropy collapse happens, the training process will crash. We provide a proof for Theorem~\ref{theorem:sec_leads_to_crash} in Appendix~\ref{appendix:malignant_entropy_collapse}.

According to Theorem~\ref{theorem:sec_leads_to_crash}, we know that the model crash is caused by the spectral energy concentration. 
In Figure~\ref{fig:spectral_energy_concentration}, we evaluate and compare the curves of $\operatorname{SEC}(d_q, s)$ of three different blocks under a successfully trained model and a crashed model. In a successfully trained model, the spectral energy distributes in all directions.
However, the spectral energy of a crashed model only concentrates on a few directions. 
As shown in Figure~\ref{fig:spectral_energy_concentration}, in a crashed model, the SEC collapses into less than 10 directions.
\begin{figure}[htbp]
	\centering
	\begin{subfigure}{0.33\linewidth}
		\centering
		\includegraphics[width=1.0\linewidth]{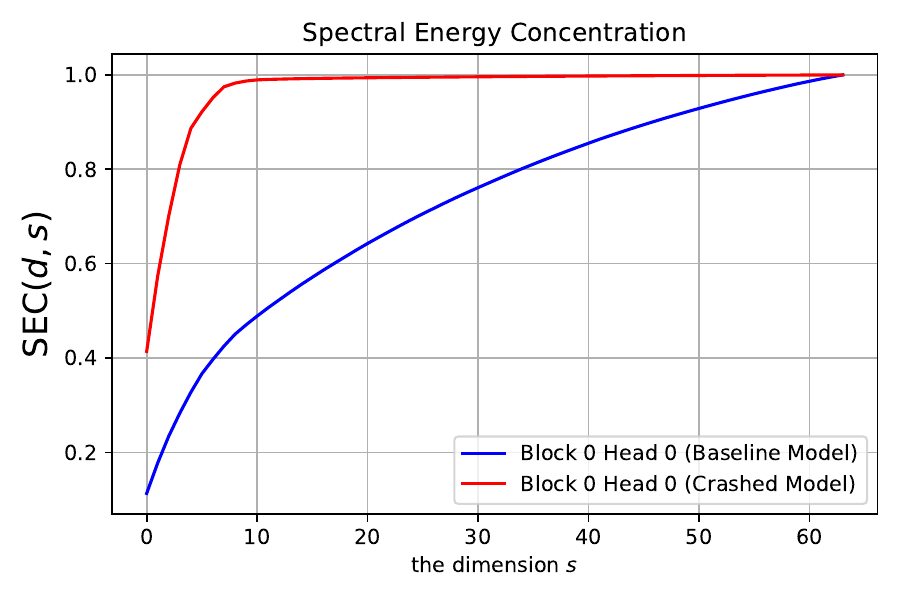}
            \subcaption*{(a) \emph{Block 0}}     
		\label{chutian1}
	\end{subfigure}
	\begin{subfigure}{0.33\linewidth}
		\centering
		\includegraphics[width=1.0\linewidth]{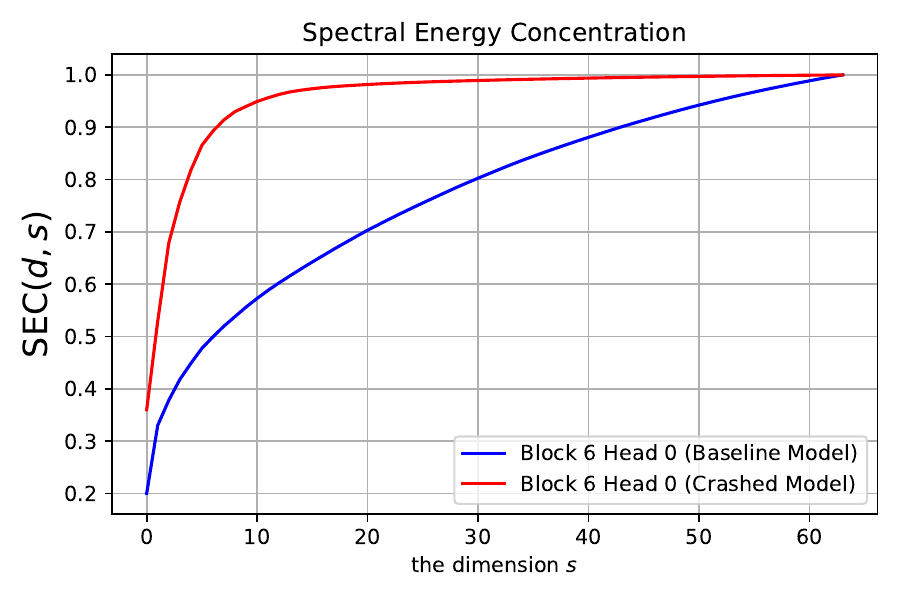}
            \subcaption*{(b) \emph{Block 6}}
		
	\end{subfigure}
	\begin{subfigure}{0.33\linewidth}
		\centering
		\includegraphics[width=1.0\linewidth]{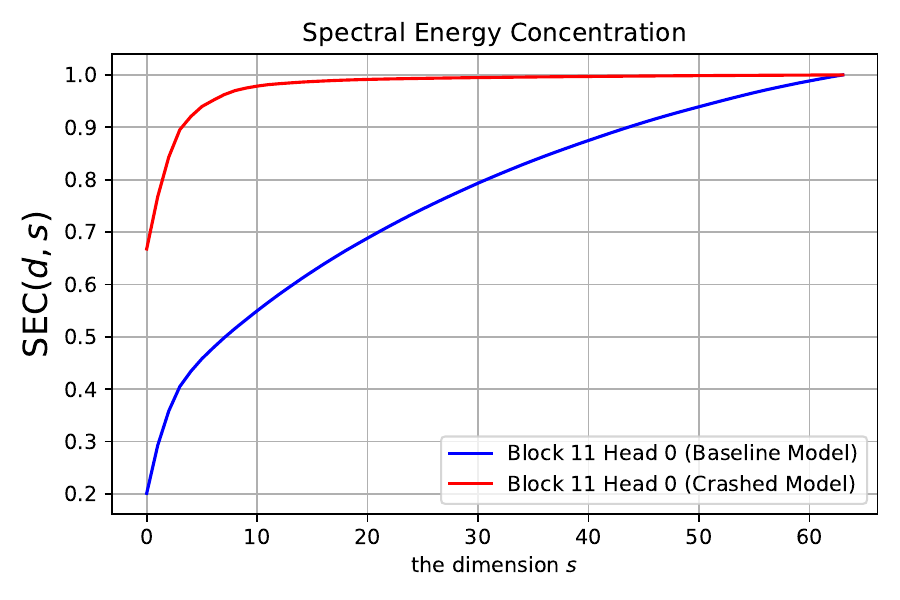}
            \subcaption*{(c) \emph{Block 11}}
		
	\end{subfigure}
 \caption{Comparison of spectral energy concentration index SEC between a successfully trained model and a crashed model. 
 The attention maps 
 of three different blocks are shown. The spectral energy distributes in all directions in a successful training case; whereas the spectral energy only concentrates on a few directions in a crashed model.
 } 
\label{fig:spectral_energy_concentration}
\end{figure}
\begin{figure}[h]
\scriptsize
\centering
\vspace{-8pt}
\begin{tikzpicture}[
node distance = 0.40cm and 0.65cm,
    arrow/.style = {-Stealth, thick},
    every node/.style = {font=\normalsize}
]
\node (X) {${\bX^l}_{\big\downarrow}$};
\node[right=1cm of X] (XX) {${ \bX^{\top} \otimes \bX^{\top} }_{\big\downarrow}$};
\node[right=1cm of XX] (delta_WqWk) {${\frac{\partial \operatorname{vec}(\bP)}{\partial \operatorname{vec}({\bW_q}^{\top}{\bW_k} )}}_{\big\downarrow}$};
\node[right=1cm of delta_WqWk] (WqWk) {${ {\bW_q}^{\top}{\bW_k} }_{\big\downarrow}$};
\node[right=1cm of WqWk] (AA) {$\bA_{\big\Downarrow}$};
\node[right=1cm of AA] (X2) {${\bX^{l+1}}_{\big\downarrow}$};
\draw[arrow] (X) -- (XX);
\draw[arrow] (XX) -- (delta_WqWk);
\draw[arrow] (delta_WqWk) -- (WqWk);
\draw[arrow] (WqWk) -- (AA);
\draw[arrow] (AA) -- (X2);
\end{tikzpicture}
\caption{Attribution flow chart of attention collapse.}
\label{fig:attribution_flow_chart}
\end{figure}
{
Figure~\ref{fig:attribution_flow_chart} reveals how the attention collapse propagates through each term, 
illustrating the entire attribution chain 
from the input $\bX^l$ to the output $\bX^{l+1}$ in the attention module. Note that $\big\downarrow$ indicates being low-rank and $\Downarrow$ means being sparse and simultaneously low-rank. If $\boldsymbol{X}$ is low-rank, then $\boldsymbol{X}^{\top} \otimes \boldsymbol{X}^{\top}$ is also low-rank because 
$
\operatorname{rank}(\boldsymbol{X} \otimes \boldsymbol{X}) = \operatorname{rank}(\boldsymbol{X}) \cdot \operatorname{rank}(\boldsymbol{X}).
$
According to the gradient computation, we have
$
\frac{\partial \operatorname{vec}(\boldsymbol{P})}{\partial \operatorname{vec}(\boldsymbol{W}_q^{\top} \boldsymbol{W}_k)} = \boldsymbol{X}^{\top} \otimes \boldsymbol{X}^{\top},
$
thus we know $\frac{\partial \operatorname{vec}(\boldsymbol{P})}{\partial \operatorname{vec}(\boldsymbol{W}_q^{\top} \boldsymbol{W}_k)}$ is also low-rank. 
Meanwhile, it should be noted that the spectral energy of $\frac{\partial \operatorname{vec}(\boldsymbol{P})}{\partial \operatorname{vec}(\boldsymbol{W}_q^{\top} \boldsymbol{W}_k)}$ is over-concentrated, which means that the gradient update will dramatically 
change $\boldsymbol{W}_q^{\top} \boldsymbol{W}_k$, thus $\boldsymbol{W}_q^{\top} \boldsymbol{W}_k$ will have large probability to be low-rank. In the paper, we have proved that being low-rank and  the leading singular values of $\boldsymbol{W}_q^{\top} \boldsymbol{W}_k$ are very large will lead to the attention map $\boldsymbol{A}$ over-concentrated (see Appendix~\ref{appendix:benign_entropy_collapse} for the proof), and become to a sparse and simultaneously low-rank matrix. 
Finally, an over-concentrated $\boldsymbol{A}$ will lead to $\boldsymbol{X}^{l+1}$ to be low-rank. We provide a more detailed analysis of Figure~\ref{fig:attribution_flow_chart} in Appendix~\ref{sec:illustration_of_figure5}.
}

\subsection{Our Solution: Taming Transformer via Weyl's Inequality}
\label{sec:without_using_lr_warmup}
The analysis above reveals that spectral energy concentration is the key reason 
leading to unstable training. 
One manifestation of spectral energy concentration is the rapid growth of the singular values of of the weight matrices. 
Thus, our 
basic idea to prevent malignant collapse is to suppress  
the fast growth of the singular values. 
Fortunately, Weyl's inequality provides us a simple but effective tool.

\begin{theorem}[Weyl's Inequality on Singular Values]
\label{theorem:weyl_inequality_on_singular_value}
Let \( \bW_1, \bW_2 \in \mathcal{R}^{m\times n} \) where $m \geq n$, $\sigma_1 (\bW_1) \geq \sigma_2 (\bW_1) \geq ... \geq \sigma_n (\bW_1)$ be the ordered singular values of $\bW_1$, and $\sigma_1 (\bW_2 \geq \sigma_2 (\bW_2) \geq ... \geq \sigma_n (\bW_2)$ be the ordered singular values of $\bW_2$. 
Then we have:
\begin{align*}
\sigma_{i+j-1}(\bW_1 + \bW_2) \leq \sigma_{i}(\bW_1) + \sigma_{j}(\bW_2).  
\end{align*}
\end{theorem}
The proof for Theorem~\ref{theorem:weyl_inequality_on_singular_value} is provided in Appendix~\ref{appendix_weyl_inequality_proof}. 
From Theorem~\ref{theorem:weyl_inequality_on_singular_value}, it is easy to see that $\sigma_1(\bW_1 + \bW_2) \leq \sigma_{1}(\bW_1) + \sigma_{1}(\bW_2)$. 
Let $\bW_t$ be 
the weight matrix at time step $t$, $\nabla \bW_t$ 
be the quantity computed from the gradients and their 
derivations (where $\nabla \bW_t$ can be obtained by SGD~\citep{sgd_robbins1951stochastic}, Adagrad~\citep{adagrad_duchi2011adaptive}, or Adam~\citep{adam_kingma2014adam}), 
$\alpha_t$ is the learning rate at time step $t$. Usually, our update equation is $\bW_{t} = \bW_{t-1} - \alpha_t \nabla \bW_t$.
According to Weyl's Inequality, we have, 
\begin{equation}
\label{eq:weight_update}
     \sigma_{1}(\bW_t) = \sigma_{1}(\bW_{t-1} - \alpha_t \nabla \bW_t) \leq \sigma_{1}(\bW_{t-1})+ \alpha_t \sigma_{1}(\nabla \bW_t).
\end{equation}
An important observation from Equation~\ref{eq:weight_update} is that if $\sigma_1(\nabla \bW)$ is very large, then $\bW_t$ will be 
significantly different from $\bW_{t-1}$. It means that the successive updates of $\bW_t$ at time step $t$ from time step $t-1$ would ``{\emph{jump}'' too much. A ``{\emph{smoother}'' updating should satisfy the following rule.

\begin{mrule}[Steady Weights Updating Rule] 
\label{principle:smoothing}
Given weights matrix $\bW_{t-1}$ and the updating quantity $\nabla \bW_t$ at step $t$, with the learning rate $\alpha_t$, a 
steady weights should satisfy the following inequality:
$   \| \bW_{t-1} - \alpha_t \nabla \bW_t \|_2  \leq (1+\tau) \| \bW_{t-1}  \|_2, $  where $\tau > 0$ is a small factor.
\end{mrule}

{To meet Rule \ref{principle:smoothing}, what we need is that $\sigma_{1}(\bW_{t-1})+ \alpha_t \sigma_{1}(\nabla \bW_t) \leq  (1+\tau) \| \bW_{t-1}  \|_2 = (1+\tau)  \sigma_{1}(\bW_{t-1})$. It is easy to see that: }
\begin{equation}
\label{eq:learning_rate_bound_by_tau}
    \alpha_t \leq \tau \frac{\sigma_{1}(\bW_{t-1})}{\sigma_{1}(\nabla \bW_t)}.
\end{equation}
This inequality 
tell us that the learning rate $\alpha_t$ should be bounded by a ratio of singular values 
${\sigma_{1}(\bW_{t-1})}$ and ${\sigma_{1}(\nabla \bW_t)}$. 
Generally, $\tau$ is a small value, \eg, 0.004 or 0.005.
The intuition behind 
is that if the spectral norm of $\nabla \bW_t$ is 
significantly larger than that of $\bW_t$, then the model is potentially undergoing rapid changes. In such cases, a large learning rate could lead to training instability. Therefore, our 
strategy is that if 
$\alpha_t > \tau \frac{\sigma_{1}(\bW_{t-1})}{\sigma_{1}(\nabla \bW_t)} $
then we truncate $\alpha_t$ 
to $\tau \frac{\sigma_{1}(\bW_{t-1})}{\sigma_{1}(\nabla \bW_t)}$. Since our 
base optimizer is AdamW~\citep{adamw_IlyaLoshchilov2018FixingWD} and our algorithm is motivated by Weyl's Inequality, 
we term our algorithm as AdamW$^2$.

\begin{algorithm}[t]
\small
\caption{$\text{AdamW}^2$: Taming Transformer via Weyl' Inequality without learning rate warmup.} 
    \hspace*{\algorithmicindent} {Input: learning rate scheduler $\alpha_t$, weight decay $\lambda$, and first-order and second-order momentum $\beta_1$, $\beta_2$} \\
    \hspace*{\algorithmicindent} Output: updated weight $\bw_T$
	\begin{algorithmic}[1]
		\For {$t=1,2,\ldots, T$}
		    \State $\boldsymbol{G}_{t}=\nabla_{\boldsymbol{W_{t-1}}} \mathcal{L}$

		            \State ${{{{{\boldsymbol{M}}_{t}}}} = \beta_{1} \boldsymbol{M}_{t-1}+\left(1-\beta_{1}\right) \boldsymbol{G}_{t}}, \quad  {{{{\boldsymbol{V}}_{t}}}} = \beta_{2} \boldsymbol{V}_{t-1}+\left(1-\beta_{2}\right) \boldsymbol{G}_{t}^{2}$

		            \State ${{{{\hat{\boldsymbol{M}}_{t}}}} = \boldsymbol{M}_{t} /\left(1-\beta_{1}^{t}\right)}, \quad {{{{\hat{\boldsymbol{V}}_{t}}}} = \boldsymbol{V}_{t} /\left(1-\beta_{2}^t\right)}$

		            \State ${{{{\nabla \boldsymbol{W}_{t}}}}} = {\hat{\boldsymbol{M}}_{t}} \oslash \sqrt{{\hat{\boldsymbol{V}}_{t}} + \epsilon }$ 
                                \Comment{ ${{\nabla \boldsymbol{W}_{t}}}$ is the final update quantity.}

                        \State {\textcolor{blue}{$\hat \delta_1 = \operatorname{PowerIter}(\nabla \bW_t),\quad \hat \sigma_1 = \operatorname{PowerIter}(\bW_{t-1})$}}  
                        \Comment{Power iteration to compute spectral norm.}

                    \If {\textcolor{blue}{$\frac{\alpha_t  \hat \delta_1}{\hat \sigma_1}>\tau$}} \Comment{If Rule~\ref{principle:smoothing} does not meet, 
                    adjust the learning rate $\alpha_t$.}
			        \State {\textcolor{blue}{$\alpha_{t} =   \frac{\tau \hat \sigma_1}{\hat \delta_1} $}}
                    \EndIf

			    \If {Weight Decay is Yes} 
			        \State {\jingAlgo{\boldsymbol{W}_{t}} $\boldsymbol{W}_{t-1} - \alpha_{t}{{\nabla \boldsymbol{W}_{t}}} - \alpha_{t} \lambda_{t} \boldsymbol{W}_{t-1}$}
                    \Else 
                         \State {\jingAlgo{\boldsymbol{W}_{t}} $\boldsymbol{W}_{t-1} - \alpha_{t} {{\nabla \boldsymbol{W}_{t}}}$}
			    \EndIf
		            
		\EndFor
	\end{algorithmic} 
\label{algo:adamww}
\end{algorithm}

For clarity, we summarize our AdamW$^2$ in Algorithm~\ref{algo:adamww}, 
where lines 6-9 highlight 
our improvements over 
the base optimizer, the other codes are same as AdamW.
According to line 6 in Algorithm~\ref{algo:adamww}, 
$\sigma_{1}(\nabla \bW_t)$ and $\sigma_{1}(\bW_{t-1})$ are 
computed via a fast power iteration method. In practice, we set the maximum iterations in power iteration to 3. 
Actually, we find that two iterations are enough to estimate the spectral norm of matrices. According to Equation~\ref{eq:learning_rate_bound_by_tau}, if $\alpha_t > \tau \frac{\sigma_{1}(\bW_{t-1})}{\sigma_{1}(\nabla \bW_t)}$, then the learning rate $\alpha_t$ will be truncated to $\tau \frac{\sigma_{1}(\bW_{t-1})}{\sigma_{1}(\nabla \bW_t)}$, or the algorithm will adjust $\alpha_t$ and use the default learning rate set by the learning rate schedule.

\begin{table}[t]
\centering
\caption{Quantitative comparison of AdamW  and $\text{AdamW}^2$ with and without learning rate warmup. $\text{AdamW}^2$ demonstrates a very competitive performance compared to AdamW.}
\begin{tabular}{cccccc}
\toprule
Method & \multicolumn{2}{c}{\textbf{ViT (Acc. $\uparrow$)}} & \multicolumn{1}{c}{\textbf{GPT (Loss $\downarrow$)}} & \multicolumn{2}{c}{\textbf{Swin-Transformer (Acc. $\uparrow$)}} \\
Configurations &  ViT-B & ViT-L & GPT-S & Swin-S & Swin-B  \\
Parameters     &  86M & 307M   & 125M & 50M & 88M \\
\midrule
AdamW (with warmup)    &  80.22 & 81.65 & 2.848 & 83.02 & 83.48 \\
AdamW$^2$ (no warmup) &  80.58 & 81.82 & 2.840 & 83.14 & 83.44 \\
\bottomrule
\end{tabular}
\label{tab:experimental_results}
\end{table}

\section{Experiments}
\label{sec:experiments}

Compared to some previous works~\citep{rezero_bachlechner2021rezero, prenorm_wang2019learning, on_layer_norm_prenorm_xiong2020layer, deepnorm_wang2022deepnet, qi2023understanding} that focuse on improving the training stability of Transformer, our AdamW$^2$ does not need to adjust the network structure and we do not use learning rate warmup. For ViT, Swin-Transformer and GPT, we will use a warmup of 60 epochs, 20 epochs, and 2000 steps, respectively. 
In $\text{AdamW}^2$, we directly use a cosine learning rate schedule and decay the learning rate from maximum to minimum.

\begin{wrapfigure}{l}{0.76\textwidth}
\vspace{0pt}
  \centering
	\begin{subfigure}{0.48\linewidth}
		\includegraphics[width=1\textwidth]{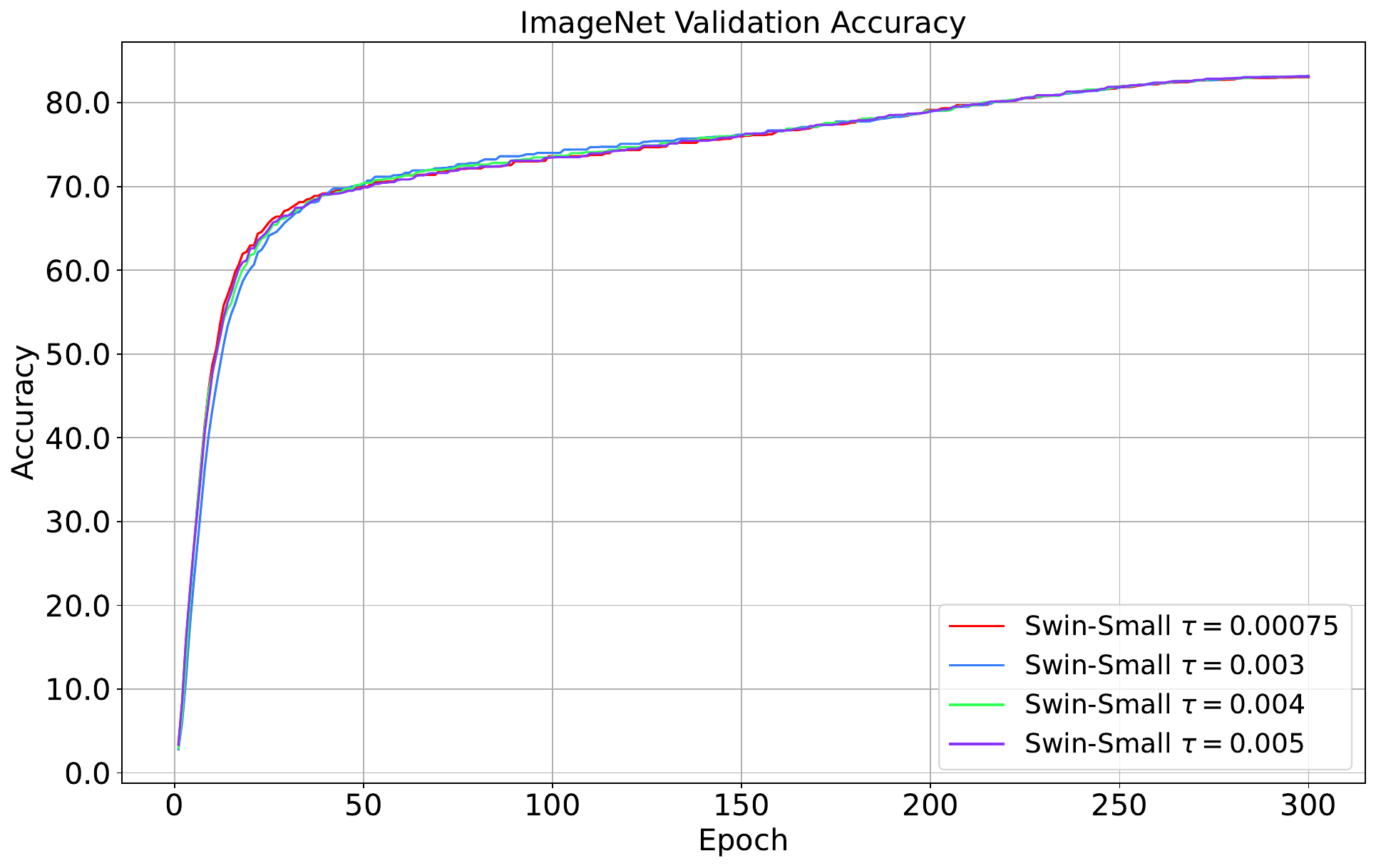}
		\caption*{(a) Swin-Transformer}
		
	\end{subfigure}
	\begin{subfigure}{0.48\linewidth}
		\includegraphics[width=1\textwidth]{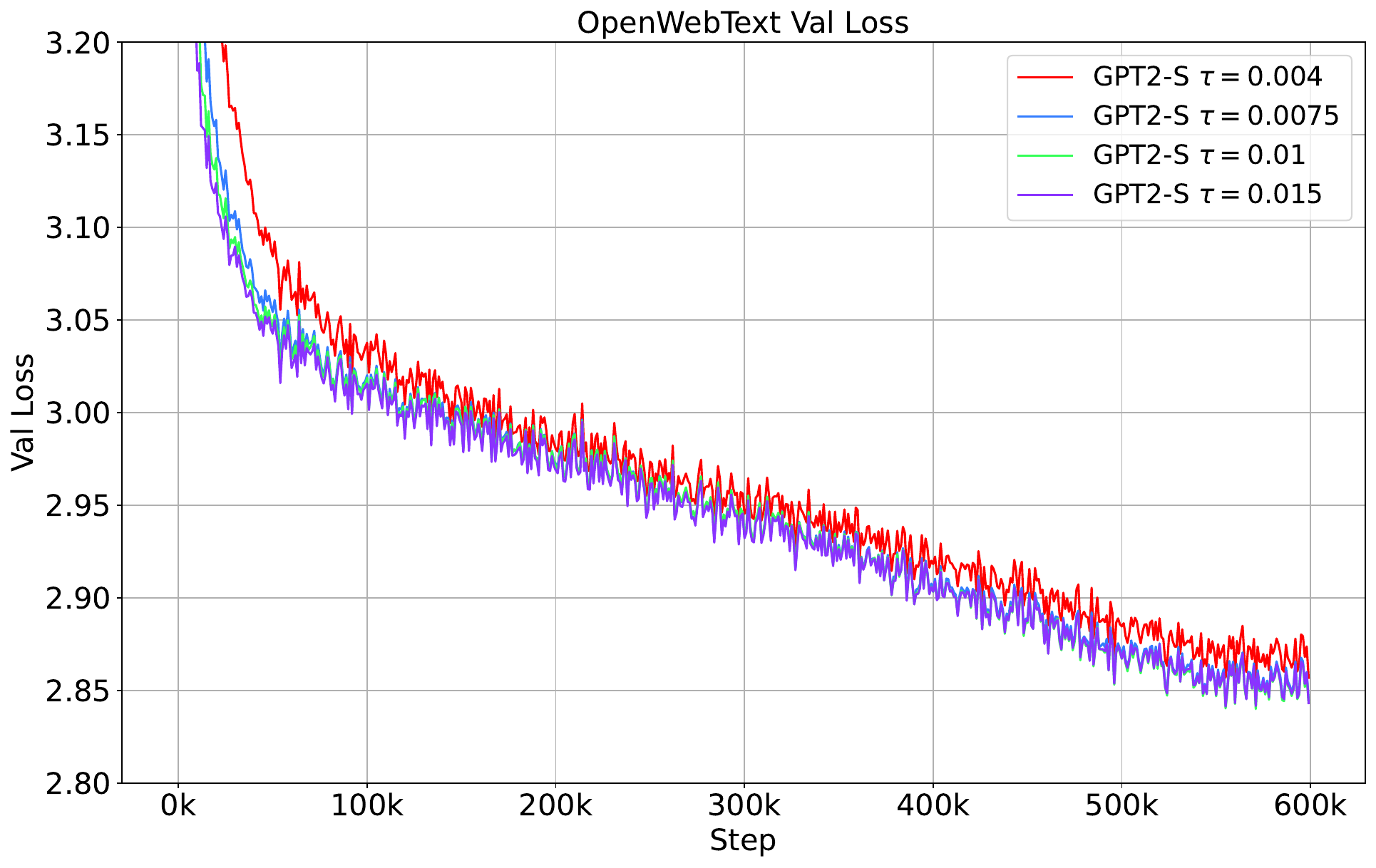}
		\caption*{(b) GPT}
		
	\end{subfigure}
 \caption{Ablation study of $\tau$ in $\text{AdamW}^2$ using Swin-S and GPT-S.} 
 \label{fig:ablation_study_tau}
\end{wrapfigure}

We conduct experiments on three  
popular Transformers, \ie, ViT~\citep{vit_dosovitskiy2020image}, GPT-2~\citep{gpt2_radford2019language} and Swin-Transformer~\citep{liu2021swin}, where ViT 
and Swin-Transformer are pure encoder architectures and GPT is a pure 
causal decoder. Note that we do not conduct any adjustments to the networks and directly use the original implementation. Our experiments include 
image classification on ImageNet~\citep{deng2009imagenet} and large language model on OpenWebText~\citep{openwebtext_Gokaslan2019OpenWeb} dataset. 
We list some training configurations in Appendix~\ref{appendix:training_configurations}.
The quantitative results are shown in Table~\ref{tab:experimental_results}. Our baseline model is the corresponding Transformer using a learning rate warmup; whereas baseline models without using learning rate warmup will crash. $\text{AdamW}^2$ demonstrates a very competitive performance compared to the baseline method. These experimental results 
verify that our  
understanding of the training dynamics of the Transformer is rational.

We also conduct an ablation study of the choice of $\tau$ 
in GPT and Swin-Transformer. The results are shown in Figure~\ref{fig:ablation_study_tau}. We can see that the performance of AdamW$^2$ varies slightly for different values of $\tau$, but overall, our approach is robust for different choices of $\tau$. The curves basically overlap in the later epochs because our steady updating rule is never broken in the later epochs.

\section{Conclusion}
\label{sec:conclusion}
In this paper, we revisited the training dynamics of the Transformers by visualizing the spectral norm of weight matrices, the activations and the attention map, presented a theoretical analysis for the Transformer training and identified two modes of attention entropy collapse, \ie, the benign collapse and the malignant collapse, in which the malignant collapse accompanies model crash. Moreover, we revealed that the \textit{spectral energy concentration} of ${\bW_q}^{\top} \bW_k$ is the reason 
behind the model crash, which causes the attention map to be sparse and simultaneously low-rank. Furthermore, we proposed a steady updating 
rule to resolve the problem of spectral energy concentration of ${\bW_q}^{\top} \bW_k$ by controlling the rapid growth of singular values, which can prevent the fast spectral energy concentration to a few directions and thus avoid the malignant entropy collapse. We conducted extensive experiments to verify the proposed strategy with ViT, Swin Transformer, and GPT, and demonstrated that the proposed strategy could effectively and stably train a model without using any learning rate warmup.

\newpage

\section*{Ethics Statement}
In this paper, we aim to provide a novel approach to train transformers without learning rate warmup. Our work does not involve any human subjects, and we have carefully ensured that it poses no potential risks or harms. Additionally, there are no conflicts of interest, sponsorship concerns, or issues related to discrimination, bias, or fairness associated with this study. We have taken steps to address privacy and security concerns, and all data used comply with legal and ethical standards. Our work fully adheres to research integrity
principles, and no ethical concerns have arisen during the course of this study.

\section*{Reproducibility Statement}
To ensure the reproducibility of our work, we provide all the details to reproduce the experiments. Theoretical proofs of the claims made in this paper, and detailed experimental settings and configurations are provided in the Appendices.

\section*{Acknowledgments}
Chun-Guang Li was partially supported by the National Natural Science Foundation of China under Grant 61876022. Qin Zou was partially funded by the National Natural Science Foundation of China under Grant 62171324.

\bibliography{iclr2025_conference}
\bibliographystyle{iclr2025_conference}

\newpage 

\appendix

\section{Kronecker Product and Vectorization}
\label{appendix:kronecker_product}
Kronecker product~\citep{kronecker_product_and_matrix_calculus_graham2018kronecker, matrix_cookbook_petersen2008matrix}, also called as matrix direct product, is an operation defined on two matrices of arbitrary size. The specific definition is as follows.

\

\begin{definition}[Kronecker Product]
Let \(\boldsymbol{A}\) be an \(n \times p\) matrix and \(\boldsymbol{B}\) an \(m \times q\)
matrix. The \(m n \times p q\) matrix
$$
\boldsymbol{A} \otimes \boldsymbol{B}=\left[\begin{array}{cccc}a_{1,1} \boldsymbol{B} & a_{1,2} \boldsymbol{B} & \cdots & a_{1, p} \boldsymbol{B} \\ a_{2,1} \boldsymbol{B} & a_{2,2} \boldsymbol{B} & \cdots & a_{2, p} \boldsymbol{B} \\ \vdots & \vdots & \vdots & \vdots \\ a_{n, 1} \boldsymbol{B} & a_{n, 2} \boldsymbol{B} & \cdots & a_{n, p} \boldsymbol{B}\end{array}\right]
$$
is called the Kronecker product of \(\boldsymbol{A}\) and \(\boldsymbol{B}\). It is also called
the direct product or the tensor product.
\end{definition}

For instance, if $\bA = \begin{bmatrix} 1 & 2 \\ 3 & 4 \end{bmatrix}$, and $\bB = \begin{bmatrix} 1 & 2 & 3 \\ 3 & 4 & 5 \end{bmatrix}$, then $\boldsymbol{A} \otimes \boldsymbol{B} = \begin{bmatrix} 1 & 2 & 3 &2 & 4 & 6 \\ 3 & 4 & 5 & 6& 8 & 10 \\ 3&6 & 9 & 4 & 8 & 12 \\ 9 &12 & 15 & 12 & 16 & 20 \end{bmatrix}$.

Some basic properties of the Kronecker product include: 
\begin{align*}
    \centering
    \boldsymbol{A} \otimes(\boldsymbol{B} \otimes \boldsymbol{C}) &= (\boldsymbol{A} \otimes \boldsymbol{B}) \otimes \boldsymbol{C}, \\
    \boldsymbol{A} \otimes(\boldsymbol{B}+\boldsymbol{C}) &= (\boldsymbol{A} \otimes \boldsymbol{B})+(\boldsymbol{A} \otimes \boldsymbol{C}), \quad 
    (\boldsymbol{A}+\boldsymbol{B}) \otimes \boldsymbol{C} = (\boldsymbol{A} \otimes \boldsymbol{C})+(\boldsymbol{B} \otimes \boldsymbol{C}), \\
    (\boldsymbol{A} \otimes \boldsymbol{B})^{T} &= \boldsymbol{A}^{T} \otimes \boldsymbol{B}^{T}.
\end{align*}

{For a matrix $\boldsymbol{A}$, the rank of $\boldsymbol{A} \otimes \boldsymbol{A}$ can be computed as,}
$$
\operatorname{rank}(\boldsymbol{A} \otimes \boldsymbol{A}) = \operatorname{rank}(\boldsymbol{A}) \cdot \operatorname{rank}(\boldsymbol{A}).
$$

{It means if the rank of the matrix $\boldsymbol{A}$ is small, then the rank of $\boldsymbol{A} \otimes \boldsymbol{A}$ will also be very small.
}

\
In mathematics, \textit{Vectorization}~\citep{kronecker_product_and_matrix_calculus_graham2018kronecker, matrix_cookbook_petersen2008matrix} is usually used together with the Kronecker product to express matrix multiplication as a linear transformation on matrices.
After vectorization, we can calculate the Jacobian matrix of the matrix product more conveniently. A property of vectorization for the matrix product is defined below.

\begin{proposition}[Property of Vectorization for Matrix Product]
     Let $\bA \in \mathcal{R}^{m \times n}, \bB\in \mathcal{R}^{n \times k}, \bC\in \mathcal{R}^{k \times l}$, then we have
\[\operatorname{vec}(\bA\bB\bC) = (\bC^{\top} \otimes \bA)\operatorname{vec}(\bB).\]
\end{proposition}
\begin{proof}
{
Let $\bC_i$ 
be the $i$-th row of $\bC$. Then we have: 
\begin{align*}
\operatorname{vec}(\bA \bB \bC) &= \sum_{i=1}^n \sum_{j=1}^k b_{ij}\operatorname{vec}(\ba_i\bC_j) \\[1ex]
&= \sum_{i=1}^n \sum_{j=1}^k b_{ij}(\bC_j^{\top} \otimes \ba_i) \\[1ex]
&= \sum_{j=1}^k (\bC_j^{\top} \otimes \bA)\bb_j \\[1ex]
&= (\bC^{\top} \otimes \bA)\operatorname{vec}(\bB). 
\end{align*}
}
\end{proof}

Furthermore, we have the following properties:
\begin{align*}
    \centering
    \operatorname{vec}(\bA \bB) &= (\bI_k \otimes \bA) \operatorname{vec}(\bB) = \left(\bB^{\top} \otimes \bI_m\right) \operatorname{vec}(\bA), \\
    \operatorname{vec}(\bA \bB \bC) &=\left(\bC^{\top} \bB^{\top} \otimes \bI_{m}\right) \operatorname{vec}(\bA) \\
    &= \left(\bC^{\top} \otimes \bA\right) \operatorname{vec}(\boldsymbol{B}) \\
    &= \left(\bI_{l} \otimes \bA \bB\right) \operatorname{vec}(\bC)
\end{align*}

\noindent where $\bI_k \in \mathcal{R}^{k\times k}, \bI_l \in \mathcal{R}^{l\times l}, \bI_m \in \mathcal{R}^{m\times m}$ are all identity matrices.

Together with the Kronecker product, vectorization is an effective tool to compute matrix calculus. We can see the following two examples.

Let $\bP = \bA \bB$ where $\bA \in \mathcal{R}^{m\times n}, \bB \in \mathcal{R}^{n\times k}$, we have: 
\begin{align*}
    \centering
    \frac{\partial \operatorname{vec} (\bP)}{\partial \operatorname{vec} (\bA)} = \boldsymbol{B}^{\top} \otimes \boldsymbol{I}_m, \quad
    \frac{\partial \operatorname{vec} (\bP)}{\partial \operatorname{vec} (\bB)} = \boldsymbol{I}_k \otimes \boldsymbol{A}.
\end{align*}
    
Let $\bP = \bA \bB \bC$ where $\bA \in \mathcal{R}^{m\times n}, \bB \in \mathcal{R}^{n\times k}, \bC \in \mathcal{R}^{k\times l} $, we have,
\begin{align*}
    \centering
    \frac{\partial \operatorname{vec} (\bP)}{\partial \operatorname{vec} (\bA)} = \bC^{\top} \bB^{\top} \otimes \bI_{m}, \quad
    \frac{\partial \operatorname{vec} (\bP)}{\partial \operatorname{vec} (\bB)} = \bC^{\top} \otimes \bA, \quad
    \frac{\partial \operatorname{vec} (\bP)}{\partial \operatorname{vec} (\bC)} = \bI_{l} \otimes \bA \bB.
\end{align*}

Vectorization and Kronecker product provide us a convenient way to analyze the self-attention module. We can compute the Jacobian matrix of the output with respect to the input or the weight matrix more conveniently.
For more introduction to Kronecker product and vectorization, the readers can refer to~\citep{matrix_cookbook_petersen2008matrix, kronecker_product_and_matrix_calculus_graham2018kronecker}

\ 

\ 

\

\section{Derivation of Jacobian Matrix for Single-head Self-Attention}
\label{appendix:single_head_attention}
A  single-head self-attention can be defined as
\begin{equation*}
    \bY = \bW_v \bX \bA,
\end{equation*}
\noindent where $\bA = \operatorname{softmax}(\frac{\bP}{\sqrt{d_q}})$ in which $\bP = \bX^{\top} {\bW_q}^{\top} {\bW_k} {\bX}$.
The matrix $\bA$ is called the attention matrix and $\frac{\bP}{\sqrt{d_q}}$ is called the logit,  $\bA \in \mathcal{R}^{n\times n}, \bX \in \mathcal{R}^{d\times n}, \bW_v \in \mathcal{R}^{d_v\times d}$ . Here, our goal is to calculate $\frac{\partial \text{vec}(\bY)}{\partial \text{vec}(\bX)} $.

In the main body, we have derived $\frac{\partial \operatorname{vec}(\bP)}{\partial \operatorname{vec}(\bX)}$.
Here, let us calculate the matrix calculus of $\bA = \operatorname{softmax}(\frac{\bP}{\sqrt{d_q}})$ with respect to $\bP$ using Kronecker products and vectorization. We can rewrite it as $\bA = \exp(\frac{\bP}{\sqrt{d_q}}) \oslash (\boldsymbol{1}_n\boldsymbol{1}_n^\top \exp(\frac{\bP}{\sqrt{d_q}}))$, where $\boldsymbol{1}_n$ denotes an $n$-dimensional vector of 1's
in $\mathcal{R}^{n}$. 
Note that $\bA$ is obtained by conducting a softmax operation in each column individually.

First, let us define two intermediate variables:
\begin{equation*}
    \bB = \exp(\frac{\bP}{\sqrt{d_q}}), \quad \bC  = \boldsymbol{1}\boldsymbol{1}^\top \exp(\frac{\bP}{\sqrt{d_q}})  = \boldsymbol{1}\boldsymbol{1}^\top \bB.
\end{equation*}
In this way, we can represent the attention matrix $\bA$ as $\bA = \bB \oslash \bC$.

Then, we can vectorize the equation $\bA = \bB \oslash \bC$ as follows:
\begin{equation*}
       \text{vec}(\bA) = \text{vec}(\bB \oslash \bC) = \text{vec}(\bB) \oslash \text{vec}(\bC).
\end{equation*}
According to  the chain rule, we have
\begin{equation*}
       \frac{\partial \text{vec}(\bA)}{\partial \text{vec}(\bP)} = \frac{\partial \text{vec}(\bA)}{\partial \text{vec}(\bB)} \frac{\partial \text{vec}(\bB)}{\partial \text{vec}(\bP)} + \frac{\partial \text{vec}(\bA)}{\partial \text{vec}(\bC)} \frac{\partial \text{vec}(\bC)}{\partial \text{vec}(\bP)}.
\end{equation*}

Let us calculate each individual term. We have
\begin{align*}
    \frac{\partial \text{vec}(\bA)}{\partial \operatorname{vec}(\bB)} &= \boldsymbol{1}_{n^2} \oslash \text{diag}(\text{vec}(\bC)), \\
   \frac{\partial \text{vec}(\bB)}{\partial \text{vec}(\bP)} &= \frac{\text{diag}(\text{vec}(\bB))}{\sqrt{d_q}}, \\ 
   \frac{\partial \text{vec}(\bA)}{\partial \text{vec}(\bC)} &= -\text{diag}(\text{vec}(\bB) \oslash (\text{vec}(\bC) \odot \text{vec}(\bC) )), \\
   \frac{\partial \text{vec}(\bC)}{\partial \text{vec}(\bP)} &=  \frac{(\boldsymbol{I}_n \otimes \boldsymbol{1}_n\boldsymbol{1}_n^\top) \text{diag}(\text{vec}(\bB))}{\sqrt{d_q}},
\end{align*}
\noindent where $\boldsymbol{I}_n$ is an identity matrix of $n \times n$ and $\otimes$ is the Kronecker product, $\boldsymbol{1}_{nn}$ denotes an $n^2$-dimensional vector of 1's in $\mathcal{R}^{n^2}$. 

By substituting the above four terms into the chain rule, we have
\begin{align*}
    \frac{\partial \text{vec}(\bA)}{\partial \text{vec}(\bP)} &= \frac{\left( \boldsymbol{1}_{n^2} \oslash \text{diag}(\text{vec}(\bC))\right) \text{diag}(\text{vec}(\bB)) -\text{diag}(\text{vec}(\bB) \oslash (\text{vec}(\bC) \odot \text{vec}(\bC) )) (\boldsymbol{I}_n \otimes \boldsymbol{1}_n\boldsymbol{1}_n^\top ) \text{diag}(\text{vec}(\bB)) }{\sqrt{d_q}} \\
    &=\frac{ \text{diag}(\text{vec}(\bA)) -\text{diag}(\text{vec}(\bB) \oslash (\text{vec}(\bC) \odot \text{vec}(\bC) )) (\boldsymbol{I}_n \otimes \boldsymbol{1}_n \boldsymbol{1}_n^\top ) \text{diag}(\text{vec}(\bB)) }{\sqrt{d_q}} \\
    &= \frac{ \text{blockdiag}(\text{diag}(\bA_{:,1}) - \bA_{:,1} \bA_{:,1}^\top, \dots, \text{diag}(\bA_{:,n}) - \bA_{:,n} \bA_{:,n}^\top) }{\sqrt{d_q}}.
\end{align*}

For the 
simplicity, we 
denote
$$
\bJ  = \text{blockdiag}(\text{diag}(\bA_{:,1}) - \bA_{:,1} \bA_{:,1}^\top, \dots, \text{diag}(\bA_{:,n}) - \bA_{:,n} \bA_{:,n}^\top).
$$

When $\bA$ and $\bP$ are vectorized into 
vectors, 
we use $\ba$ and $\bp$ to denote them, respectively. 
Then we see that 
\begin{equation*}
     \frac{\partial \text{vec}(\ba)}{\partial \text{vec}(\bp)} = \frac{\text{diag}(\ba) - \ba \ba^{\top}}{\sqrt{d_q}}.
\end{equation*}
If $\ba$ approaches to a unit vector $\be$, then the Jabobian matrix $\frac{\partial \text{vec}(\ba)}{\partial \text{vec}(\bp)}$ will tend to $\boldsymbol{0}$.

In Section~\ref{sec:matrix_calculus_transformer}, we have the following Jacobian matrix
\begin{equation*} 
    \frac{\partial \operatorname{vec}(\bP)}{\partial \operatorname{vec}(\bX)} = 
    (\bX^{\top}{\bW_k}^{\top}{\bW_q} \otimes \bI_n)\bK + (\bI_n \otimes \bX^{\top}{\bW_q}^{\top}{\bW_k}). 
    \label{eq:dp_dx}
\end{equation*}

By vectorization of $ \bY = \bW_v \bX \bA$, we have 
\begin{equation*}
     \partial \text{vec}(\bY) =  (\bA^\top \otimes \bW_v) \partial \text{vec}(\bX)  + (\bI_n \otimes \bW_v\bX) \partial \text{vec}(\bA).
\end{equation*}

Therefore, according to the product rule and chain rule, we can denote the Jacobian matrix of $\bY$ with respect to $\bX$ as follows: 
\begin{align*}
  \frac{\partial \text{vec}(\bY)}{\partial \text{vec}(\bX)} &= (\bA^\top \otimes \bW_v) + (\bI_n \otimes \bW_v\bX) \frac{\partial \text{vec}(\bA)}{\partial \text{vec}(\bX)}, \\
  &= (\bA^\top \otimes \bW_v) + (\bI_n \otimes \bW_v\bX) \frac{\partial \text{vec}(\bA)}{\partial \text{vec}(\bP)} \frac{\partial \text{vec}(\bP)}{\partial \text{vec}(\bX)}.
\end{align*}

Bringing in all these terms, we get the following formula: 
{\footnotesize
\begin{equation}
  \frac{\partial \text{vec}(\bY)}{\partial \text{vec}(\bX)} = (\bA^\top \otimes \bW_v) + (\bI_n \otimes \bW_v\bX)  \frac{\bJ}{\sqrt{d_q}}  \left( 
  (\bX^{\top}{\bW_k}^{\top}{\bW_q} \otimes \bI_n)\bK + (\bI_n \otimes \bX^{\top}{\bW_q}^{\top}{\bW_k}) 
 \right).
\label{eq:self_attention_bp}
\end{equation}
}

Let us analyze Equation~\ref{eq:self_attention_bp}. If a malignant entropy mode happens, $\frac{\bJ}{\sqrt{d_q}}$ will approach $\boldsymbol{0}$ because each $\ba$ in $\bA$ will be a unit vector $\be$. From the perspective of the forward process, the features $\bY$ will collapse to several directions. From the perspective of the backward process, $\frac{\bJ}{\sqrt{d_q}}$ will become $\boldsymbol{0}$, and $\frac{\partial \text{vec}(\bY)}{\partial \text{vec}(\bX)} $ will be a sparse and simultaneously low-rank matrix. Through $\bA^\top \otimes \bW_v$, most of the positions in $\bX$ will get zero gradient, and only very few columns will obtain some large noisy gradients. In a malignant entropy mode, the learned feature is invalid and useless. Similarly, if a benign entropy mode happens, the attention map $\bA$ will approach an identity matrix $\bI$ and $\frac{\bJ}{\sqrt{d_q}}  \approx 0$ when $\bA \approx \bI$. Therefore, we have $\frac{\partial \text{vec}(\bY)}{\partial \text{vec}(\bX)} \approx (\bI^\top \otimes \bW_v)$. In this way, a self-attention module degenerates to a linear layer.

\ 

\ 

\

\section{Proof of Benign Entropy Collapse}
\label{appendix:benign_entropy_collapse}
Recall that $\bA = \operatorname{softmax}(\frac{\bP}{\sqrt{d_q}})$ where $\bP = \bX^{\top} {\bW_q}^{\top} {\bW_k} {\bX}$. Here, let $\bW =  {\bW_q}^{\top} {\bW_k} $ and $\bW \in \mathcal{R}^{d\times d}$. In this way, we have that $\bA = \operatorname{softmax}({\frac{\bX^{\top} \bW {\bX}}{\sqrt{d_q}}})$. We know that $\operatorname{rank}(\bW) \leq d_q$ and $d_q < d$. 
To prove $\bA$ will always collapse to an identity matrix when $\bW$ is a non-symmetric positive quasi-definite square matrix, it is equivalent to prove $\mathbb{E}\left[{\bx_i}^{\top} \bW \bx_i\right] \gg \mathbb{E}\left[{\bx_i}^{\top} \bW \bx_j\right]$ for any $i \neq j$.
It will be very hard to prove it mathematically if $\bW$ is a form of a non-symmetric positive quasi-definite square matrix. Therefore, let us make some simplification assumptions. 
Assume $\bW$ is a real symmetric positive semi-definite square matrix and its trace is  in direct proportion to the dimension $d_q$, and any $\bx_i$ is a high-dimension random vector and each element in $x_{i,j} \widesim[2]{\text{iid}} \mathcal{N}(0, 1)$. 

We break our proof into two sub-problems. 
\begin{proposition}[Expectation of ${\bx_i}^{\top} \bW \bx_i$]
\label{pro:xi_xi}
    Let $\bW$ be a real symmetric positive semi-definite matrix, and any $\bx_i$ be a high-dimensional random vector. Then, we have $\mathbb{E}\left[{\bx_i}^{\top} \bW \bx_i\right] = \operatorname{trace} (\bW)$.
\end{proposition}
\begin{proof}
Let $\bW$ be a real symmetric positive semi-definite, thus it can be decomposed into $\bW = \bU \Sigma \bU^{\top}$. In this way, we have
\begin{align*}
\mathbb{E}\left[{\bx_i}^{\top} \bW \bx_i\right] &= \mathbb{E}\left[{\bx_i}^{\top} \bU \Sigma \bU^{\top} \bx_i\right] \\
&= \mathbb{E}\left[\bz^{\top} \Sigma \bz\right] \quad \text{(let $\bz = \bU^{\top} \bx_i$)} \\
&=\mathbb{E}\left[\sum_{i=1}^{d} \sigma_i z_i^{2}\right] \quad \text{($\Sigma$ is a diagonal matrix)}\\
&=\sum_{i=1}^{d_q} \sigma_i\times (0+1)  \quad \text{(by independence, mean 0)} \\
&= \sum_{i=1}^{d_q} \sigma_i  = \operatorname{trace}(\bW). 
\end{align*}
For a real symmetric positive semi-definite, all its singular values are larger or equal to 0. Thus, we know $\sum_{i=1}^{d_q} \sigma_i > 0$ considering $\bW$ is not a matrix of all zeros.
\end{proof}

\begin{proposition}[Expectation of ${\bx_i}^{\top} \bW \bx_j$ for $i \neq j$]
\label{pro:xi_xj}
    Let $\bW$ be a real symmetric positive semi-definite square matrix, and any $\bx_i$ is a high-dimension random vector, $\mathbb{E}_{i\neq j}\left[{\bx_i}^{\top} \bW \bx_j\right] = 0$.
\end{proposition}
\begin{proof}
Let $\bW$ be a real symmetric positive semi-definite. Thus, it can be decomposed into $\bW = \bU \Sigma \bU^{\top}$. In this way, we have
\begin{align*}
\mathbb{E}_{i\neq j}\left[{\bx_i}^{\top} \bW \bx_j\right] &= \mathbb{E}_{i\neq j}\left[{\bx_i}^{\top} \bU \Sigma \bU^{\top} \bx_j\right] \\
&= \mathbb{E}\left[\bz^{\top} \Sigma \bv \right] \quad \text{(let $\bz = \bU^{\top} \bx_i$, and $\bv = \bU^{\top} \bx_j$)} \\
&=\mathbb{E}\left[\sum_{i=1}^{d_q} \sigma_{ij} z_i v_j \right] \quad \text{($\Sigma$ is a diagonal matrix. $\bz$ and $\bv$ are independent)}\\
&= 0. 
\end{align*}
\end{proof}

According to Proposition~\ref{pro:xi_xi} and Proposition~\ref{pro:xi_xj}, we can have that $\mathbb{E}\left[{\bx_i}^{\top} \bW \bx_i\right] > \mathbb{E}\left[{\bx_i}^{\top} \bW \bx_j\right]$ for any $i \neq j$. Considering that $\bW$  is usually a high-dimensional matrix and some of its singular values are significantly larger than 0. Thus, after the softmax operation, $\bA$ will always collapse to an identity matrix. In this way, the self-attention module degenerates into a linear projection module. The model fitting ability will decline, but model training will not crash. Our proof is based on matrix computations~\citep{golub2013matrix} and high-dimensional probability~\citep{vershynin2018high}.

\ 

\

\

\section{Proof of Malignant Entropy Collapse}
\label{appendix:malignant_entropy_collapse}
\begin{proof}

{\bf Step 1}. To prove the sparsity of \( \bA \), we must show that the number of non-zero elements in each column of \( \bA \) is small.

Now, let’s consider the properties of the matrix \( \bP = \bX^T \bW \bX \). Since \( \bX \) and \( \bW \) are low-rank matrices, the matrix \( \bP \) will also be low-rank. Specifically, the rank of \( \bP \) is bounded by the rank of \( \bX \) and \( \bW \), which is much smaller than the dimension. In particular, \( \bP \) has only a small number of significant singular values. This implies that the entries in \( \bP \) are concentrated in a lower-dimensional subspace.

When we apply the softmax function to the rows of \( \bP \), the function concentrates most of the probability mass on a few components of each column. This is because the softmax function is sharply peaked around the largest values in each row. The smaller values in each row contribute less to the sum in the denominator of the softmax function, and therefore, their corresponding entries in \( \bA \) will be small.

Thus, in each column of \( \bA \), only a small number of entries will be non-zero with high probability, and the rest will be close to zero. This establishes the sparsity of \( \bA \).

{\bf Step 2}. To prove the low-rankness of \( \bA \), we turn to proving that \( \bA \) is approximately low-rank.

As noted earlier, the matrix \( \bP = \bX^T \bW \bX \) is low-rank. Specifically, the rank of \( \bP \) is bounded by the ranks of \( \bX \) and \( \bW \), which are both small. Therefore, \( \bP \) has only a few dominant singular values. Assume that the softmax function does not significantly change the rank of the matrix. The rank of \( \bA \) is controlled by the rank of \( \bP \), as the softmax operation only introduces nonlinearities that do not increase the rank.

Thus, since \( \bP \) has a small number of dominant singular values, \( \bA \), formed by the softmax of \( \bP \), will also have a small number of significant singular values. This implies that \( \bA \) is approximately low-rank.

In conclusion, under the assumptions that \( \bX \) and \( \bW \) are low-rank matrices with sufficiently large singular values, the matrix \( \bA \) formed by applying the softmax function to \( \bP \) will exhibit the properties of both sparsity and low-rankness with high probability.

\end{proof}

\section{Proof of Weyl’s Inequality on Singular Values}
\label{appendix_weyl_inequality_proof}
Our derivation depends on~\cite{topics_in_matrix_analysis_horn1991topics, matrix_analysis_horn2012matrix}. Readers can refer to these material for more background information.

\begin{proof}
Before we prove Weyl's Inequality on singular values, let us review  Courant-Fischer min-max principle that is 
important for analyzing the singular values of matrix. 

\begin{theorem}[Courant-Fischer Min-max Principle for Singular Values]
Let $\bW \in \mathcal{R}^{m \times n}$ be a matrix where $m \geq n$. $\bW$ has ordered singular values 
$\sigma_1 \geq \sigma_2 \geq \cdots \geq \sigma_n \geq 0$. 
Then, for $i = 1, 2, \ldots, n$, we have
\[\sigma_i(\bW) = \max_{\substack{S \subset \mathcal{R}^n \\ \dim(S) = i}} \  \min_{\substack{\bx \in S \\ \|\bx\| = 1}} \|\bW \bx\| = \min_{\substack{S' \subset \mathcal{R}^n \\ \dim(S') = n-i+1}} \  \max_{\substack{\bx \in S' \\ \|\bx\| = 1}} \|\bW \bx\| \]
where the maximum is taken over all $i$-dimensional subspaces $S$ of $\mathcal{R}^n$, 
and the minimum is taken over all unit vectors $\bx$ in $S$.
\end{theorem}

Let $\bW_1 = \bU \boldsymbol{\Sigma}_{1} \bV^{\top}$ and $\bW_2 = \cU \boldsymbol{\Sigma}_{2} \cV^{\top}$ be singular value decompositions
of $\bW_1$ and $\bW_2$ with unitary matrix 
$\bV = [\bv_1, \ldots, \bv_n], 
\cV = [\fv_1, \ldots, \fv_n]$ where $\bv_i, \fv_i \in \mathcal{R}^n$ and unitary matrix
$\bU = [\bu_1, \ldots, \bu_m], 
\cU = [\fu_1, \ldots, \fu_m]$, where $\bu_j,\fu_j \in  \mathcal{R}^m$. 

Let $i$ and $j$ be positive integers with
$1 \leq i, j \leq n$ and $i + j \leq n + 1$.
Let $S_1 \equiv \text{Span} \{\bv_i,\ldots,
\bv_n\}$ and $S_2 \equiv \text{Span} \{\fv_j,\ldots, \fv_n\}$; notice that $\dim (S_1) = n - i + 1$ and $\dim (S_2) =
n - j + 1$. Let $k \equiv \dim(S_1 \cap S_2)$, then we have
\begin{align*}
\dim(S_1 \cap S_2) = \dim (S_1) + \dim (S_2) - \dim(S_1 + S_2) & = (n - i + 1) + (n - j + 1) - \dim(S_1 + S_2) \\
& \geq (n - i + 1) + (n - j + 1) - n = n - (i + j - 1) + 1 \geq 1.
\end{align*}

Because of the bounds assumed for $i$ and $j$. Thus, the subspace $S_1 \cap S_2$ has
positive dimension $k$, $n - k + 1 \leq i + j - 1$, and we have
\begin{align*}
\sigma_{i+j-1}(\bW_1 + \bW_2) & \leq \sigma_{n-k+1}(\bW_1 + \bW_2) \\
&= \min_{\substack{S \subset \mathcal{R}^n \\ \dim(S)=k}} \max_{\substack{\bx\in S \\ \|\bx\|_2=1}} \|(\bW_1 + \bW_2)\bx\|_2 \\[2ex]
&\leq \max_{\substack{\bx\in S_1 \cap S_2 \\ \|\bx\|_2=1}} \|(\bW_1 + \bW_2) \bx\|_2 \\[2ex]
&\leq \max_{\substack{\bx \in S_1 \cap S_2 \\ \|\bx\|_2=1}} \| \bW_1 \bx\|_2 + \max_{\substack{\bx\in S_1 \cap S_2 \\ \|\bx\|_2=1}} \|\bW_2 \bx\|_2 \\[2ex]
&\leq \max_{\substack{\bx \in S_1 \\ \|\bx\|_2=1}} \|\bW_1\bx\|_2 + \max_{\substack{x\in S_2 \\ \|\bx\|_2=1}} \|\bW_2 \bx\|_2 = \sigma_i(\bW_1) + \sigma_j(\bW_2).
\end{align*}
As a special case, the second part of the theorem follows directly from the general result of part (a). Specifically, for $i=j=1$, we have:
\begin{align*}
\sigma_1(\bW_1 + \bW_2) \leq \sigma_{1}(\bW_1) + \sigma_{1}(\bW_2).
\end{align*}
This completes the proof.
\end{proof}

\ 

\ 

\newpage

\section{Related Works}
\label{sec:related_works}
\textbf{Training Dynamics of Transformer.} Previous works have delved into understanding the training dynamics of Transformers from two different perspectives: a high-level perspective and a low-level perspective. From a high-level perspective, Scan\&Snap~\citep{scanandsnap_tian2023scan} unveiled complex phenomena, particularly in single-layer architectures, relating to frequency and discriminative bias. These studies linked sparse attention patterns to token co-occurrence frequencies and observed two-stage behaviors in attention logits.
JoMA~\citep{joma_tian2023joma} further improved upon previous models by incorporating residual connections and MLP nonlinearity, analyzing joint training of MLP and self-attention layers, and offering qualitative explanations for multilayer Transformer dynamics. From a low-level perspective, two critical challenges in Transformer training have been identified: rank collapse~\citep{rank_collapse_dong2021attention, rank_collapse_noci2022signal}, where attention output converges to a rank 1 matrix, potentially causing vanishing gradients; and entropy collapse~\citep{stabilizing_transformer_zhai2023stabilizing}, which denotes pathologically low attention entropy, corresponding to highly concentrated attention scores. \textit{In this work, we analyze and prove two different entropy collapse modes and identify the key reason for model failure is spectral energy
concentration. Finally, we introduce a simple but effective solution to address this problem.}

\textbf{Training Stability of Transformer.}  
ReZero~\citep{rezero_bachlechner2021rezero} introduces a simple yet effective mechanism for improving training stability. The key innovation lies in initializing residual connections to zero, which allows networks to learn identity mappings more easily. 
Admin~\citep{admin_liu2020understanding} introduces a new network initialization strategy tailored for Transformers to make the network train stable.
DeepNorm~\citep{deepnorm_wang2022deepnet} extends the concept of normalization to accommodate increasingly deeper networks. By dynamically adjusting normalization parameters, DeepNorm ensures stability even as network depth increases. 
LipsFormer~\cite{lipsformer_qilipsformer} addresses the specific challenge of stability in transformer networks. By introducing a Lipschitz continuity constraint, Lipsformer effectively mitigates the issue of exploding gradients - a common problem in deep transformer architectures. This approach ensures that the network's output changes smoothly with respect to its input, promoting overall stability. ReZero, Admin, and DeepNorm can all be considered as an approach to control the Lipschitz constant of the network in the initial stage.
\textit{In this work, by revisiting the training dynamics of Transformer, we can achieve a stable training only by modifying the optimizer instead of using learning rate warmup or changing the network structures as LipsFormer~\citep{lipsformer_qilipsformer} and QKNorm~\cite{qk_norm_henry2020query, scaling22b_dehghani2023scaling}.}

\textbf{Learning Rate Schedule.} Warmup~\citep{warmup_loshchilov2016sgdr} has emerged as a must-have technique for ensuring a stable network training, especially in the initial phases of the optimization process. This method involves gradually increasing the learning rate from a small value to the desired initial learning rate over a certain number of training steps or epochs. 
The cosine learning rate scheduler~\citep{warmup_loshchilov2016sgdr} has gained popularity due to its smooth annealing properties. This schedule decreases the learning rate following a cosine curve, starting from an initial value and decaying to a minimum value over a set number of epochs or iterations. 
Cyclic learning rates~\citep{cyclical_smith2017cyclical} involve systematically varying the learning rate between boundary values. The learning rate oscillates between a lower and upper bound, either linearly or following other patterns (e.g., triangular, cosine). 
The above-mentioned learning rate schedules require specification of a stopping time step $T$, \citet{road_less_defazio2024road} introduces a Schedule-Free approach that avoids the need for this stopping time by eschewing the use of schedules entirely.

Compared to the up-mentioned works, the core novel contributions of our work lie on follows. 
\begin{enumerate}[leftmargin=*]
    \item  We present a theoretical analysis for Transformer training and point out two entropy collapse modes, \ie the benign collapse and the malignant collapse. 

    \item We reveal that \textit{ \textit{spectral energy concentration} (SEC) of $\ {\bW_q}^{\top} \bW_k$} is the main reason of model crash.

    \item We introduce AdamW$^2$, a new optimization strategy motivated by Weyl's Inequality.
\end{enumerate}

We also observe there are two works~\citep{kosson2023rotational, kosson2024analyzing} discussing the needs of learning rate warmup  by explicitly controlling the angular
updates via Rotational Optimizer Variants and by limiting the Frobenius norm of the update relative to that of the weights. They provided  some different perspectives on the necessity
 of the learning rate warmup.

\newpage

\section{Simulation of Three Attention Modes}
\label{appendix:three_attention_modes}
We provide a simple simulation code to simulate three attention modes, but it is important to note that the real picture is more complicated. In real case, in the benign attention entropy mode, ${\bW_q}^{\top} {\bW}$ is a non-symmetric positive quasi-definite square matrix instead of a symmetric positive definite matrix in our simulation. The code is just to demonstrate the core ideas behind three attention modes.

\ 

\renewcommand{\lstlistingname}{Code}
\label{cpde:1}
\begin{lstlisting}[language=Python, caption=Simulation of Three Attention Modes.]
import torch
import torch.nn
import matplotlib.pyplot as plt
import numpy as np


#Randomly generate data and weight matrices
d_q, d, num_tokens= 64, 768, 197
Wq = torch.randn(d_q, d)
Wk = torch.randn(d_q, d)
X = torch.randn(d, num_tokens)
W = torch.mm(Wq.T, Wk)


# Normal attention mode
W1 = W
P = torch.mm(torch.mm(X.T, W1), X)
attn_map1 = P.softmax(dim=1)


# Malignant attention entropy collapse mode
u,s,v = torch.svd(W)
s[0:3] = torch.tensor([3., 2., 1.])*s[0:3] 
s[3:]  = 0.0
W2 = torch.mm(torch.mm(u, torch.diag(s)), v.T)
P = torch.mm(torch.mm(X.T, W2), X)
attn_map2 = P.softmax(dim=1)


# Benign attention entropy collapse mode
u,s,v = torch.svd(W)
W3 = torch.mm(torch.mm(u, torch.diag(s)), u.T)
P = torch.mm(torch.mm(X.T, W3), X)
attn_map3 = P.softmax(dim=1)


# Plot figures
fig, axs = plt.subplots(1, 3, figsize=(15, 5))
axs[0].imshow(attn_map1.detach().numpy())
axs[1].imshow(attn_map2.detach().numpy())
axs[2].imshow(attn_map3.detach().numpy())
plt.show()
\end{lstlisting}

\ 

\newpage

\section{Attention Map Visualization of GPT}
\label{appendix:visualization_gpt}
Figure~\ref{fig:gpt_attention_map_gif} visualizes the dynamic process of attention map as the number of training steps increases for a successful and unsuccessful GPT-Small model. It should be noted that the GPT model uses a lower triangular attention mask.

\begin{figure}[h]
\centering
\begin{subfigure}{0.30\textwidth}
      \animategraphics[poster=last, width=1\textwidth]{6}{final_figures/nanogpt/baseline/image}{1}{21}
      \subcaption*{(a) \emph{\footnotesize Block 11 (successful).}}     
\end{subfigure}
\begin{subfigure}{0.30\textwidth}
      \animategraphics[poster=last, width=1\textwidth]{6}{final_figures/nanogpt/collapse/image}{1}{21}
      \subcaption*{(b) \emph{\footnotesize Block 11 (unsuccessful).}}
\end{subfigure}
\caption{Visualization of the dynamic process of attention map as the number of training steps increases for a successful and unsuccessful GPT-Small model.
Attention map gradually becomes sparse and low-rank along with the training process in a failure case. \emph{Please click the images to play the flash. Best viewed with Acrobat Reader.}}
\label{fig:gpt_attention_map_gif}
\end{figure}

In Figure~\ref{fig:gpt_attention_map_gif}, the attention values in a successful case distribute to different position, but the attention values in a unsuccessful case will only concentrate into several directions.

\ 

\

\section{More Training Dynamics of ViT and GPT}
\label{sec:appendix_vit_gpt}

Figure~\ref{fig:why_vit_succusss} visualizes a successful ViT training process. Compared with Figure~\ref{fig:why_vit_fail}, we find several significant differences as follows. 
\begin{itemize}[leftmargin=*]
    \item In a successful ViT training process, the value of $\sigma_1({\bW_q}^{\top}{\bW_k})$ increases to 16,000, then starts to oscillate smoothly. But for an unsuccessful training, the value suddenly increases to a very large value, around 300,000, it triggers the model crash,
    \item  The $\gamma_1$ and $\beta_1$ in a successful ViT training process are very smooth, but they change a lot in an unsuccessful case,
    \item  The fast increase of $\sigma_1({\bW_q}^{\top}{\bW_k})$ is accompanied by a fast increase of $\bW_q$ and  $\bW_k$.
\end{itemize}

We can observe similar phenomenon in Figure~\ref{fig:why_gpt_succusss} and Figure~\ref{fig:why_gpt_fail}. In a successful GPT training process, the value of $\sigma_1({\bW_q}^{\top}{\bW_k})$ increases to 60, then starts to oscillate smoothly. But for an unsuccessful GPT training, the value increases to 20,000. The difference between the sclae of value between GPT and ViT may be due to the density and sparsity of the supervision signal. In GPT, each token will contribute a gradient, but in ViT, only one class label in an image provides a supervision information.

\begin{figure}[htbp]

	\centering
	\begin{subfigure}{0.33\linewidth}
		\centering
		\includegraphics[width=4.5cm,height=3.8cm]{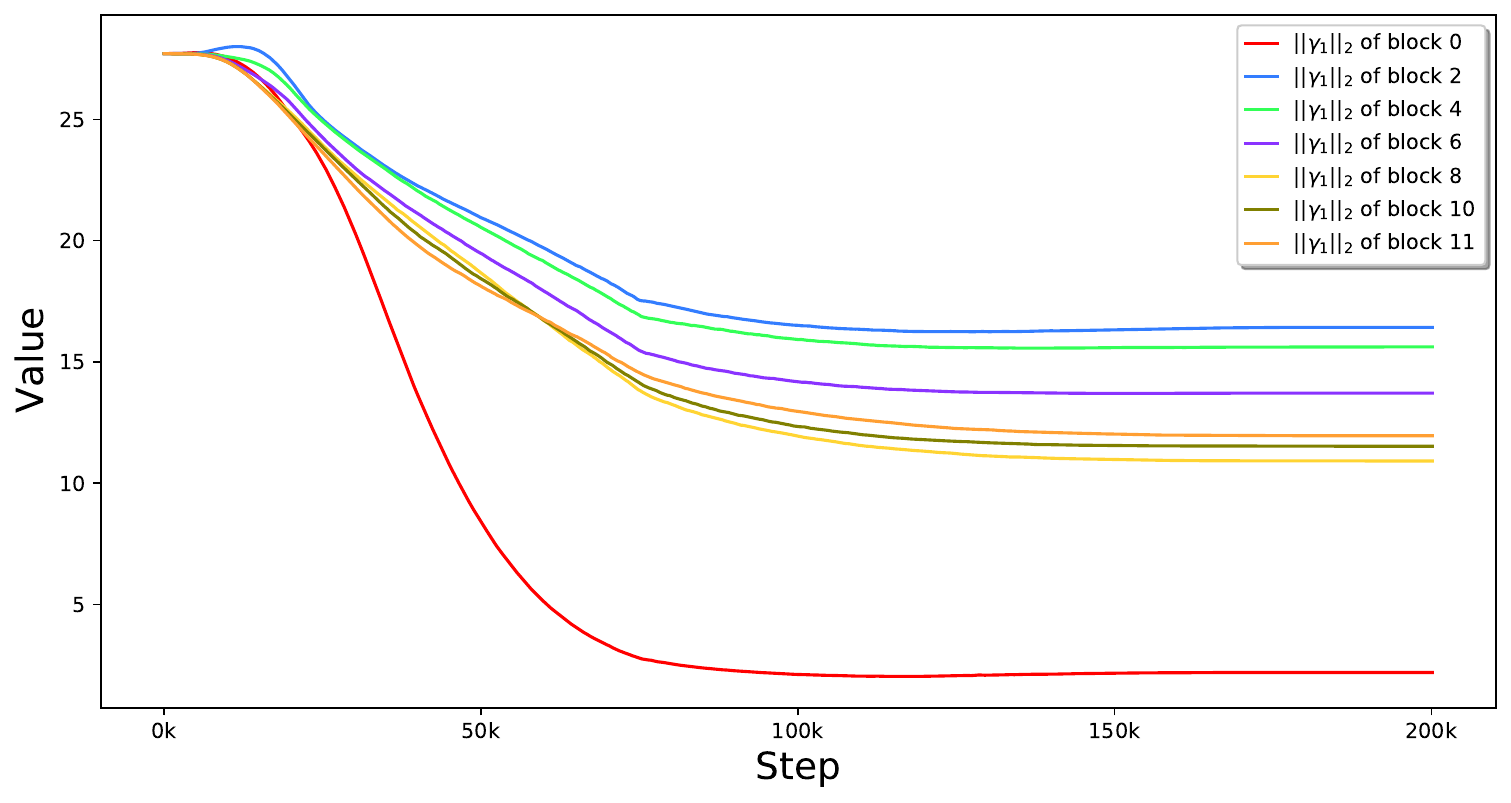}
		\caption*{(a) $\|\boldsymbol{\gamma_1}\|_2$}
		
	\end{subfigure}
	\begin{subfigure}{0.33\linewidth}
		\centering
		\includegraphics[width=4.5cm,height=3.8cm]{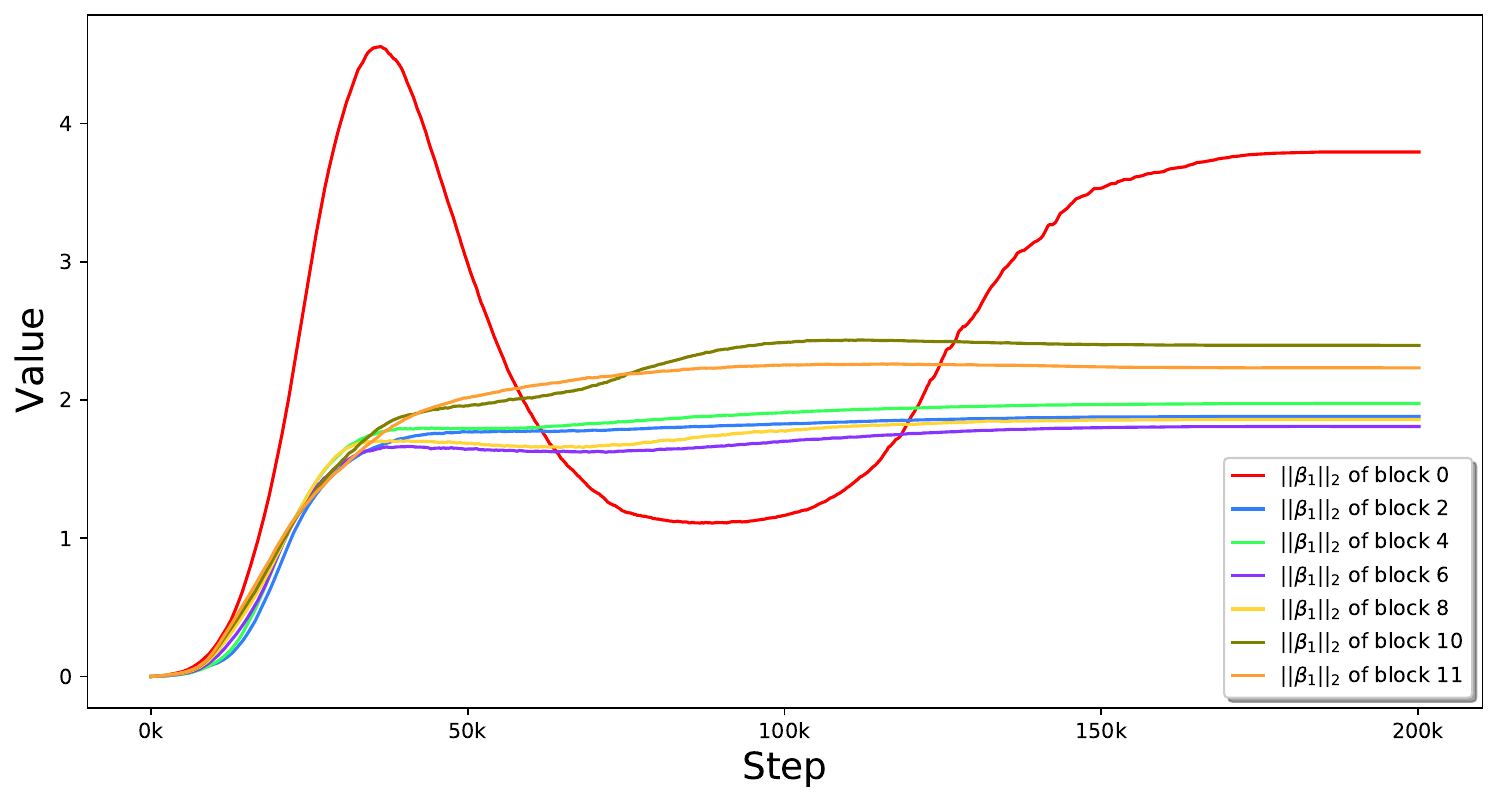}
		\caption*{(b) $\|\boldsymbol{\beta_1}\|_2$}
		
	\end{subfigure}
	\begin{subfigure}{0.33\linewidth}
		\centering
		\includegraphics[width=4.5cm,height=3.8cm]{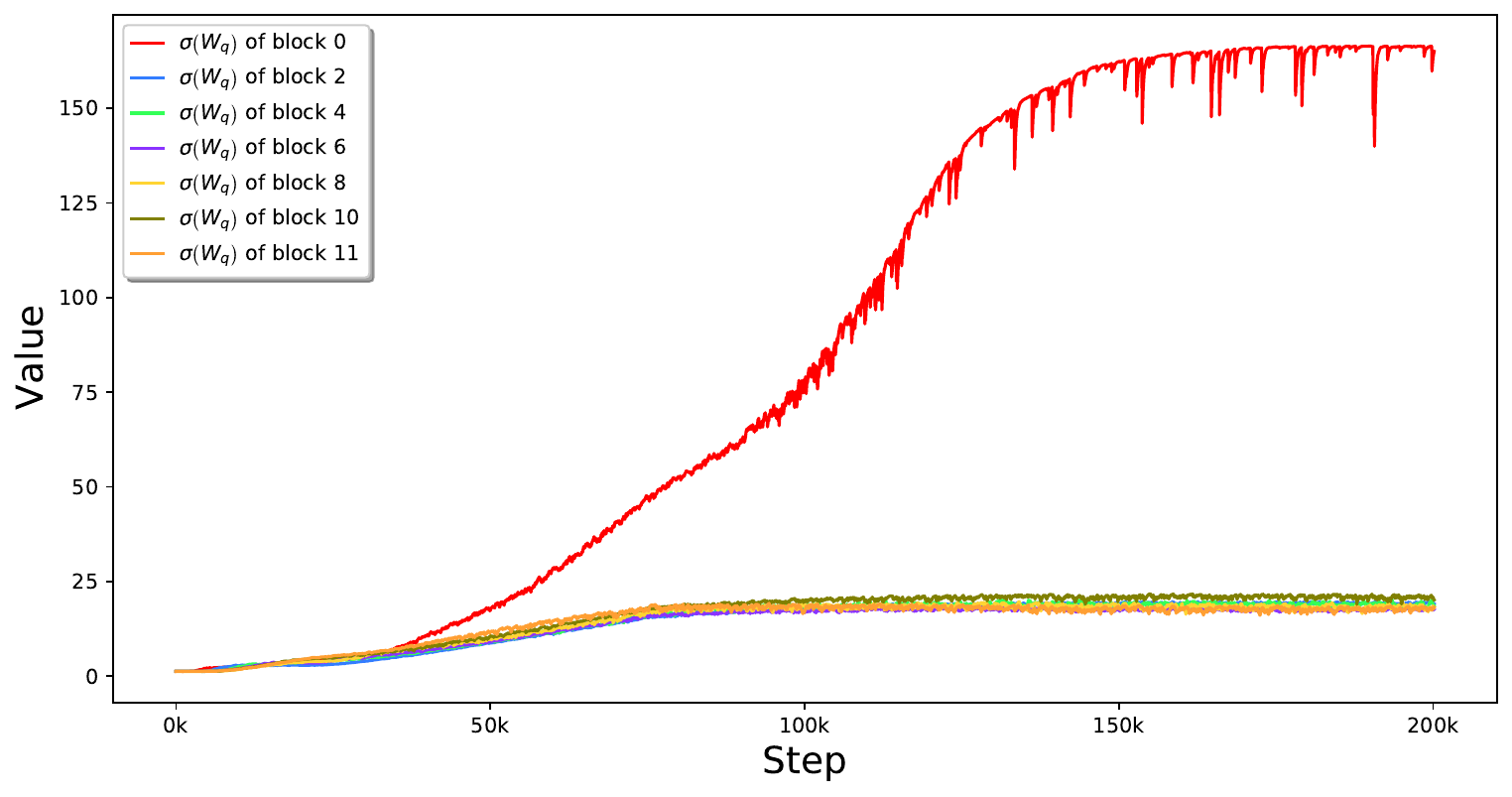}
		\caption*{(c) $\sigma_1\left({\bW_q}\right)$}
		
	\end{subfigure}

        \begin{subfigure}{0.33\linewidth}
		\centering
		\includegraphics[width=4.5cm,height=3.8cm]{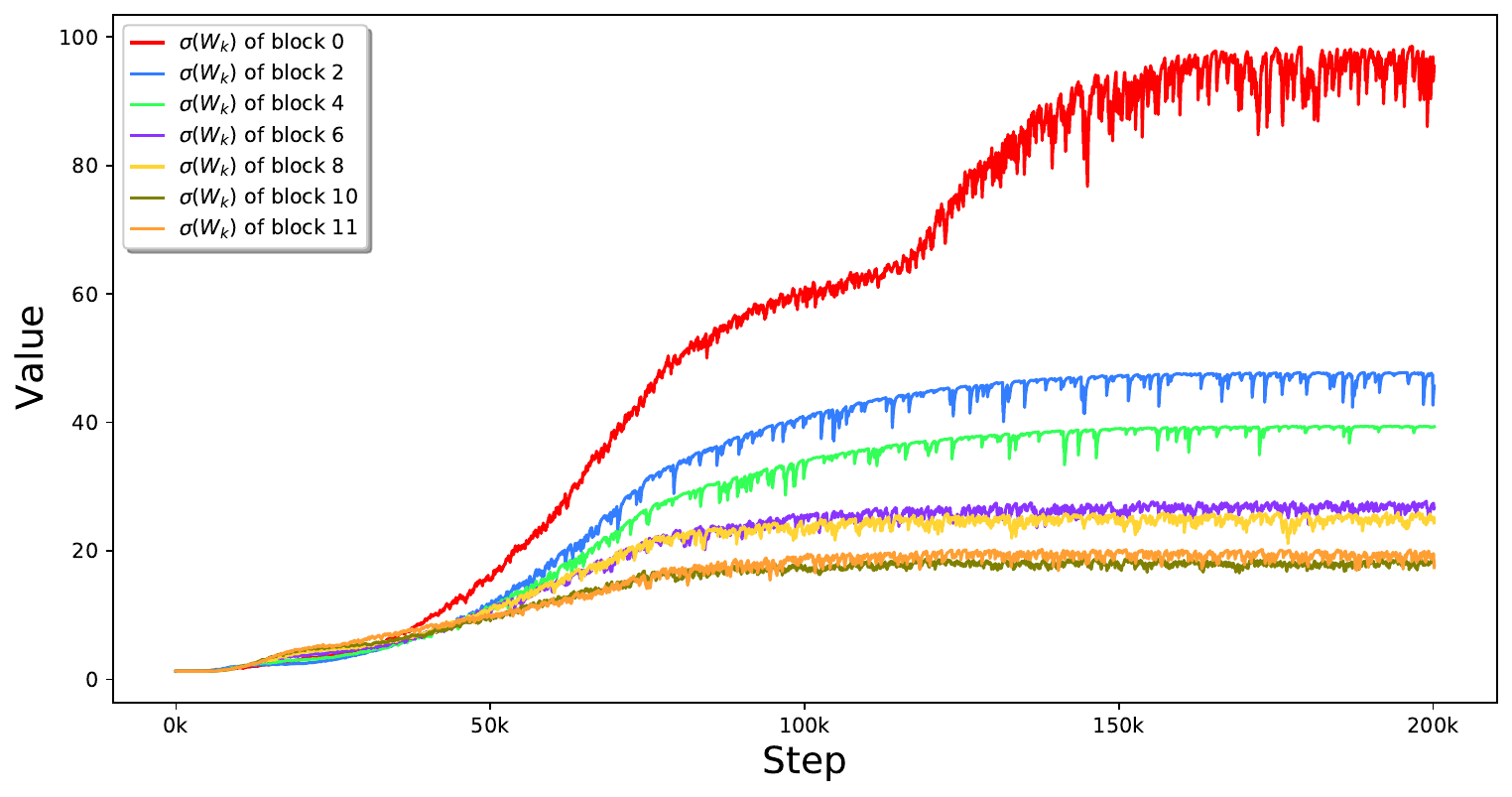}
		\caption*{(d) $\sigma_1\left({\bW_k}\right)$}
		
	\end{subfigure}
	\begin{subfigure}{0.33\linewidth}
		\centering
		\includegraphics[width=4.5cm,height=3.8cm]{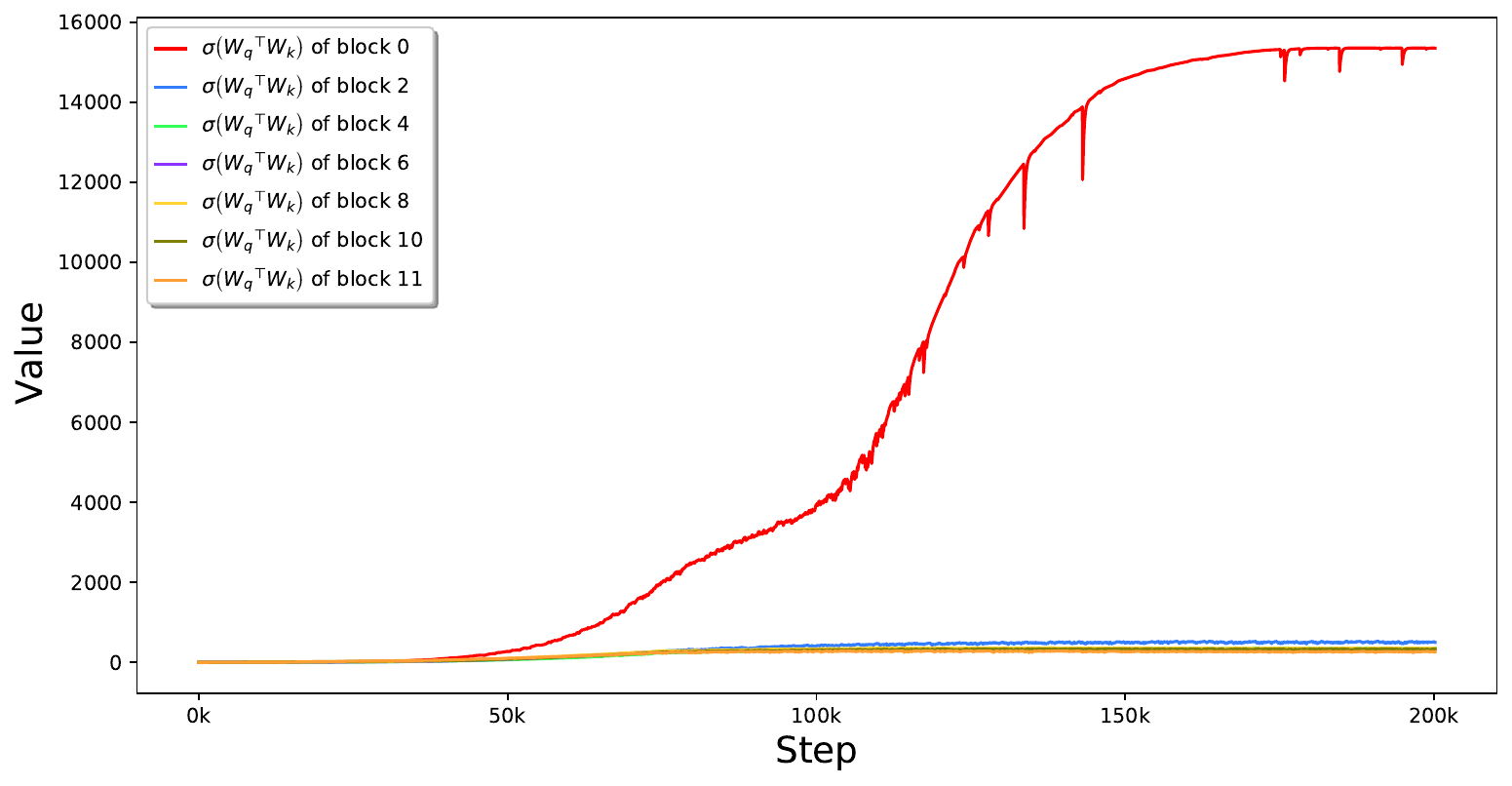}
		\caption*{(e) $\sigma_1\left({{\bW_q}^{\top}{\bW_k}}\right)$}
		\label{fig:vit_sub_success_wqwk}
	\end{subfigure}
	\begin{subfigure}{0.33\linewidth}
		\centering
		\includegraphics[width=4.5cm,height=3.8cm]{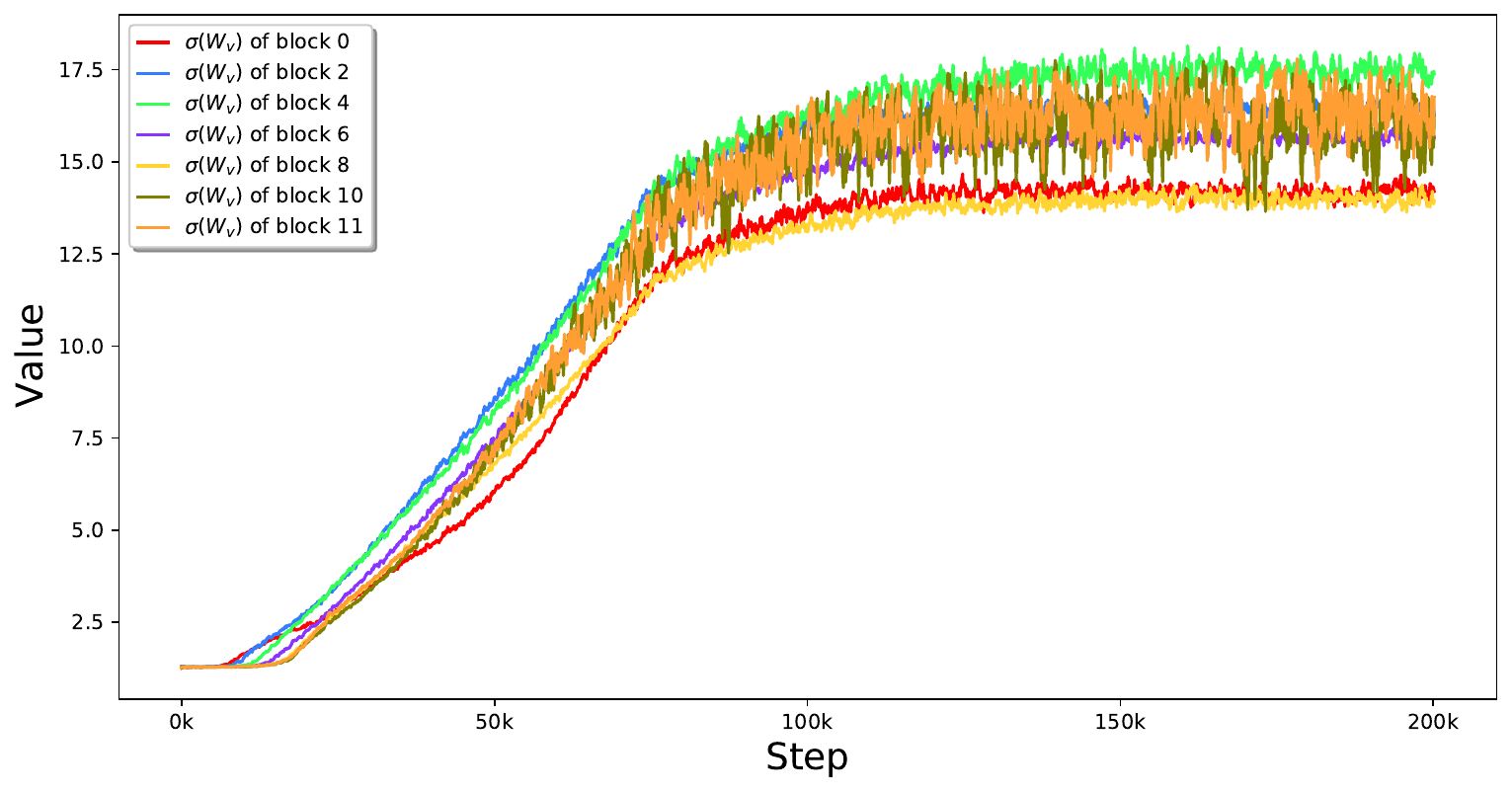}
		\caption*{(f) $\sigma_1\left({\bW_v}\right)$}
		
	\end{subfigure}

        \begin{subfigure}{0.33\linewidth}
		\centering
		\includegraphics[width=4.5cm,height=3.8cm]{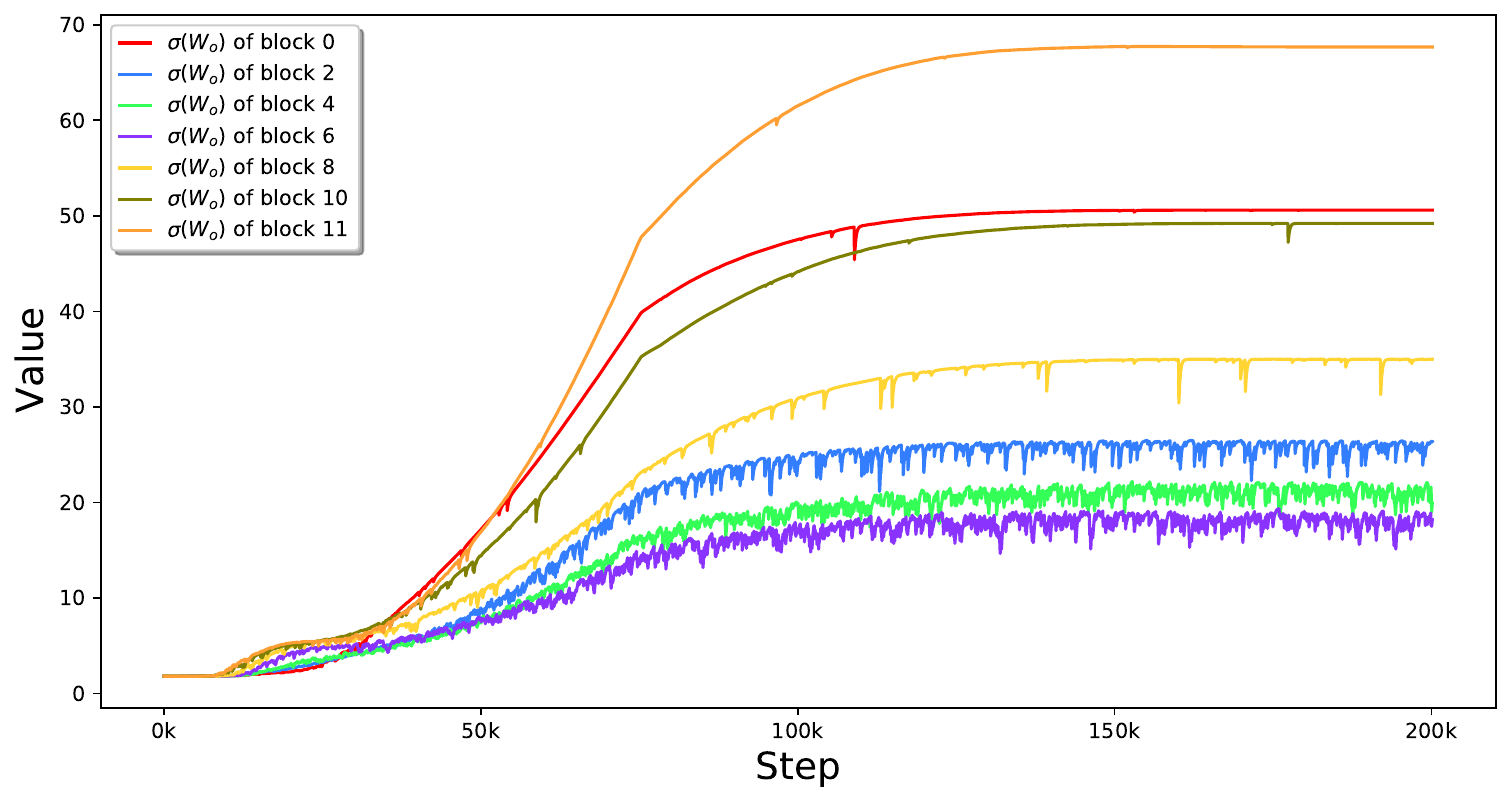}
		\caption*{(g) $\sigma_1\left({\bW_o}\right)$}
		
	\end{subfigure}
	\begin{subfigure}{0.33\linewidth}
		\centering
		\includegraphics[width=4.5cm,height=3.8cm]{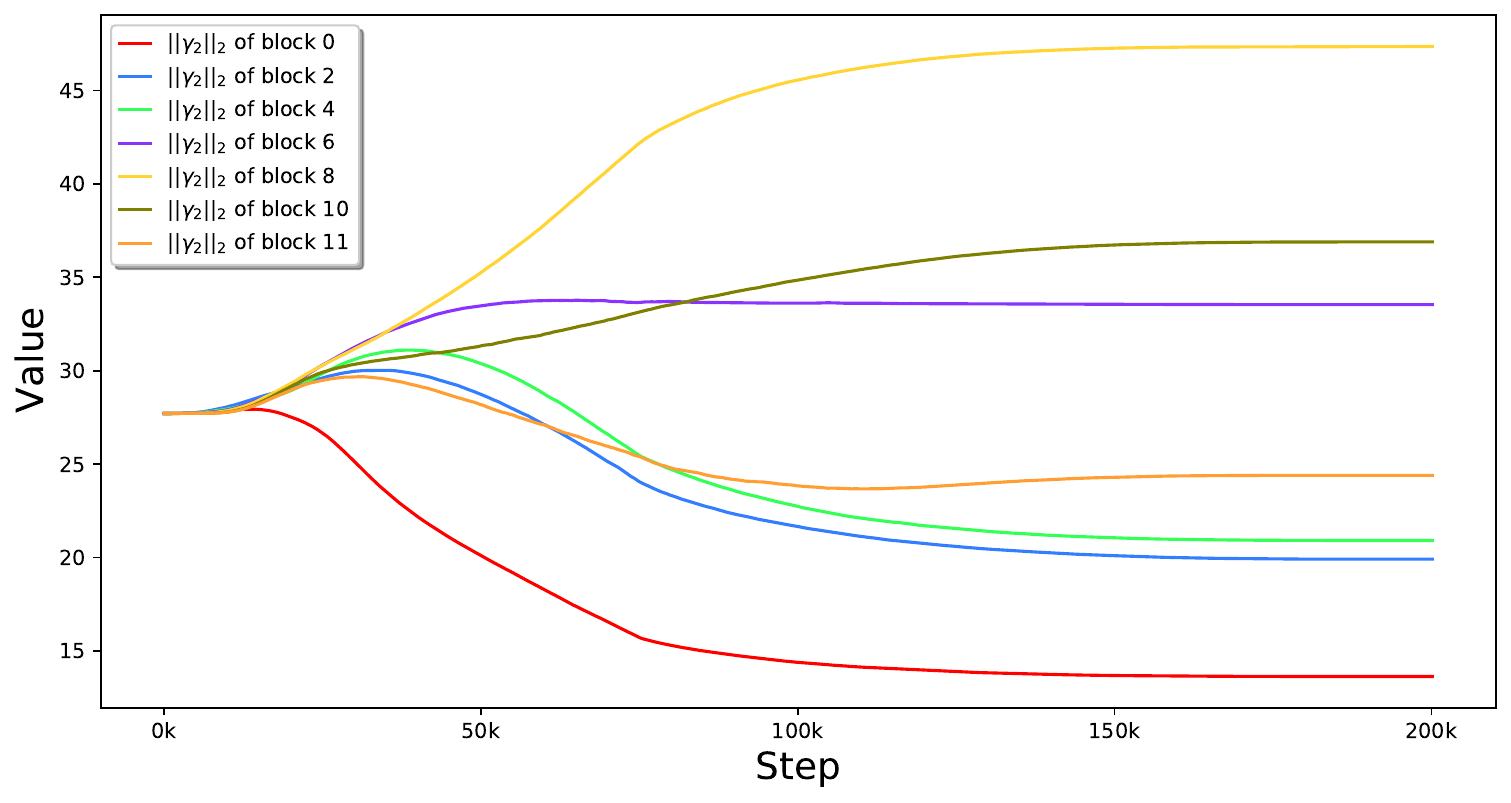}
		\caption*{(h) $\|\boldsymbol{\gamma_2}\|_2$}
		
	\end{subfigure}
	\begin{subfigure}{0.33\linewidth}
		\centering
		\includegraphics[width=4.5cm,height=3.8cm]{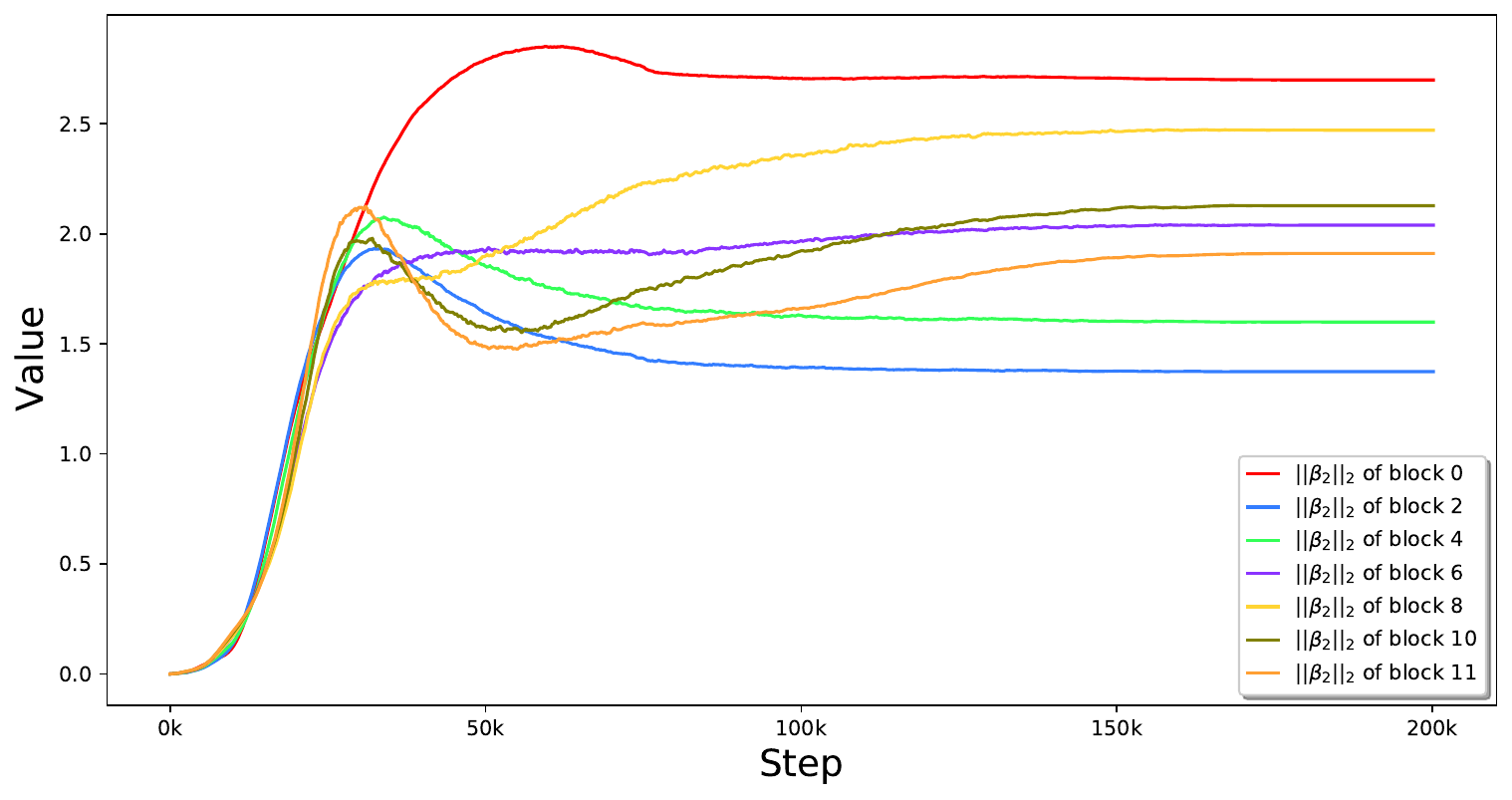}
		\caption*{(i) $\|\boldsymbol{\beta_2}\|_2$}
		
	\end{subfigure}
	
	\begin{subfigure}{0.33\linewidth}
		\centering
		\includegraphics[width=4.5cm,height=3.8cm]{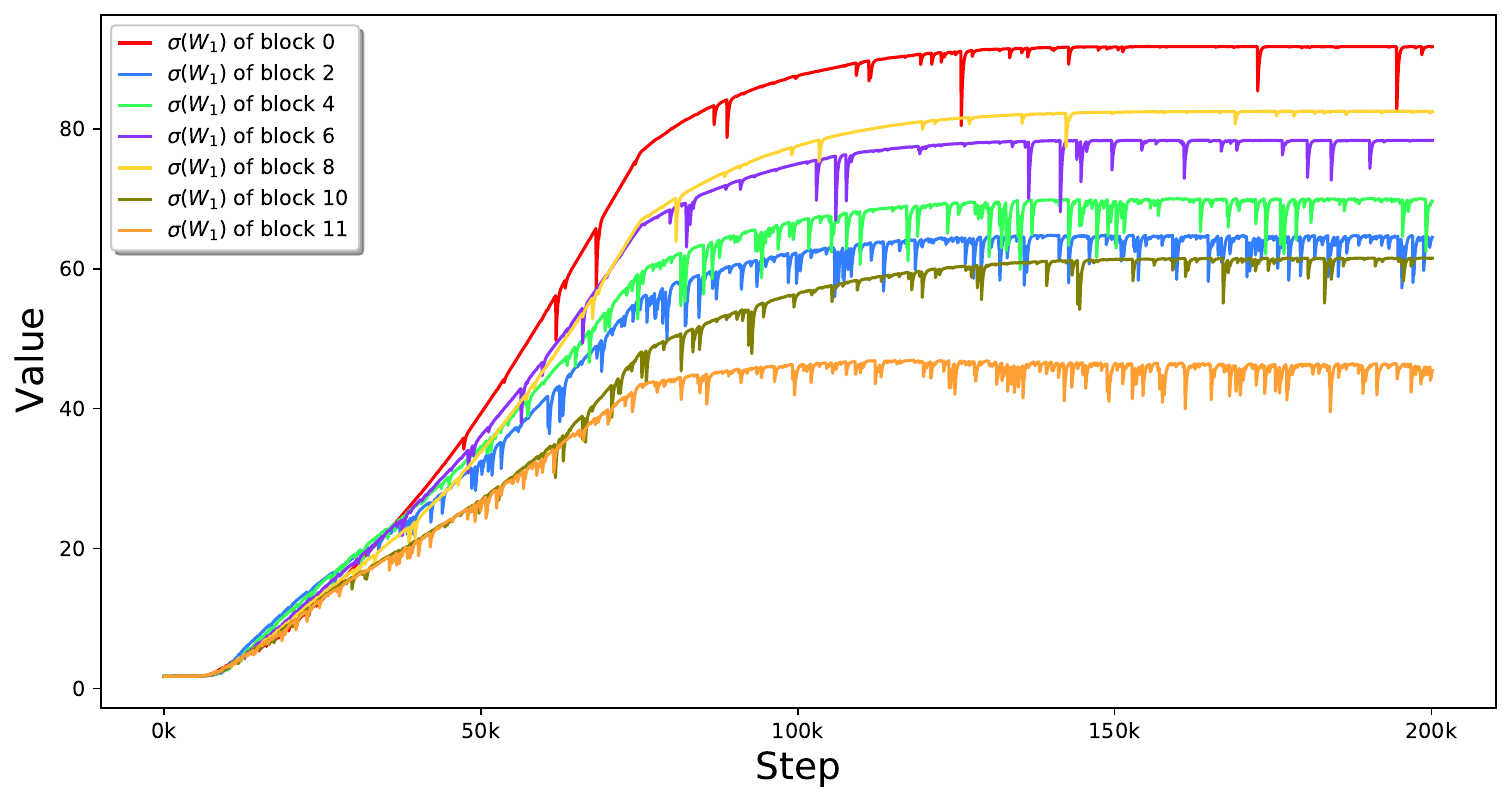}
		\caption*{(j) $\sigma_1\left({\bW_1}\right)$}
		
	\end{subfigure}
	\begin{subfigure}{0.33\linewidth}
		\centering
            \includegraphics[width=4.5cm,height=3.8cm]{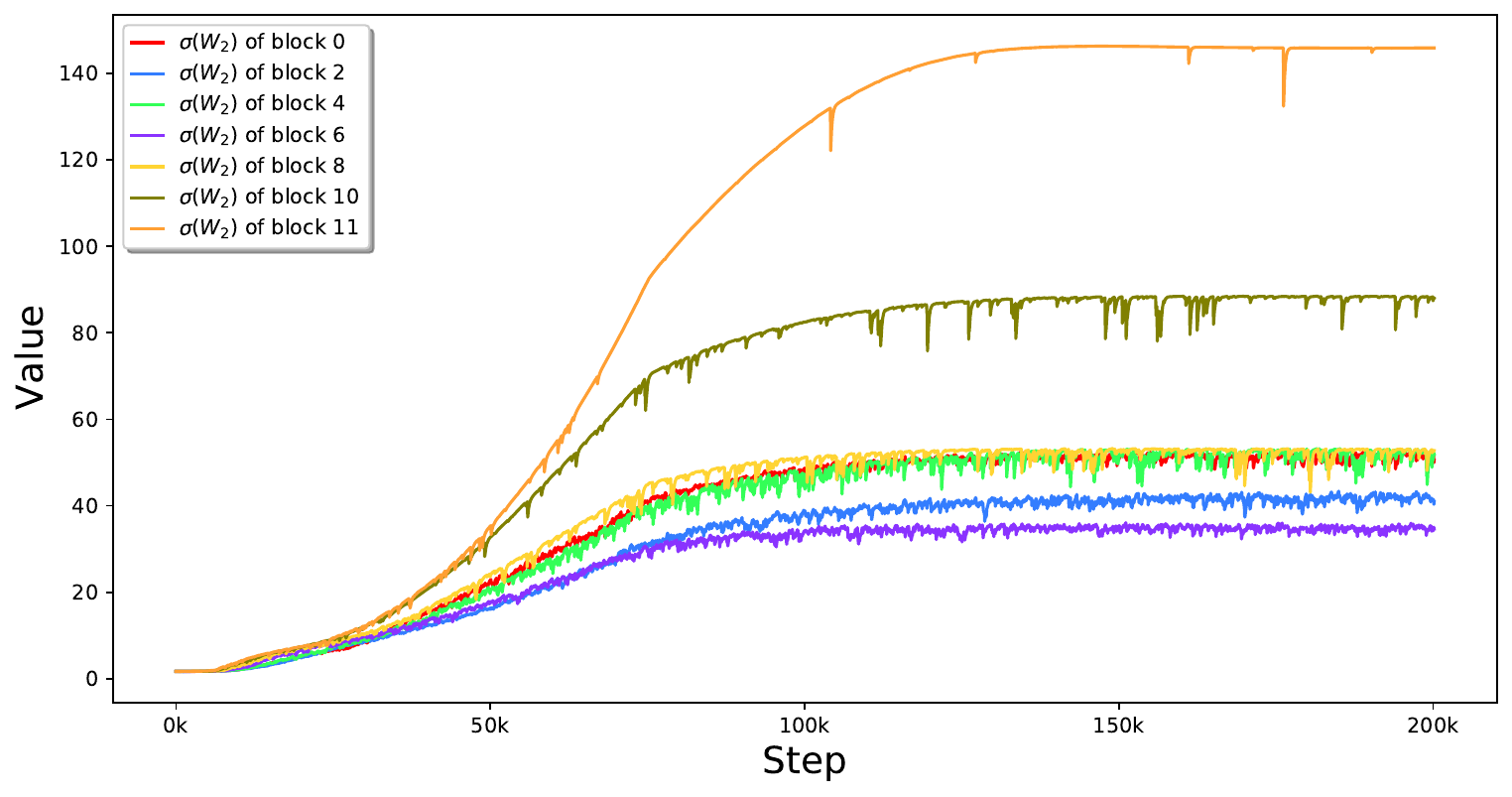}
		\caption*{(k) $\sigma_1\left({\bW_2}\right)$}
		
	\end{subfigure}
        \begin{subfigure}{0.33\linewidth}
		\centering
            \includegraphics[width=4.5cm,height=3.8cm]{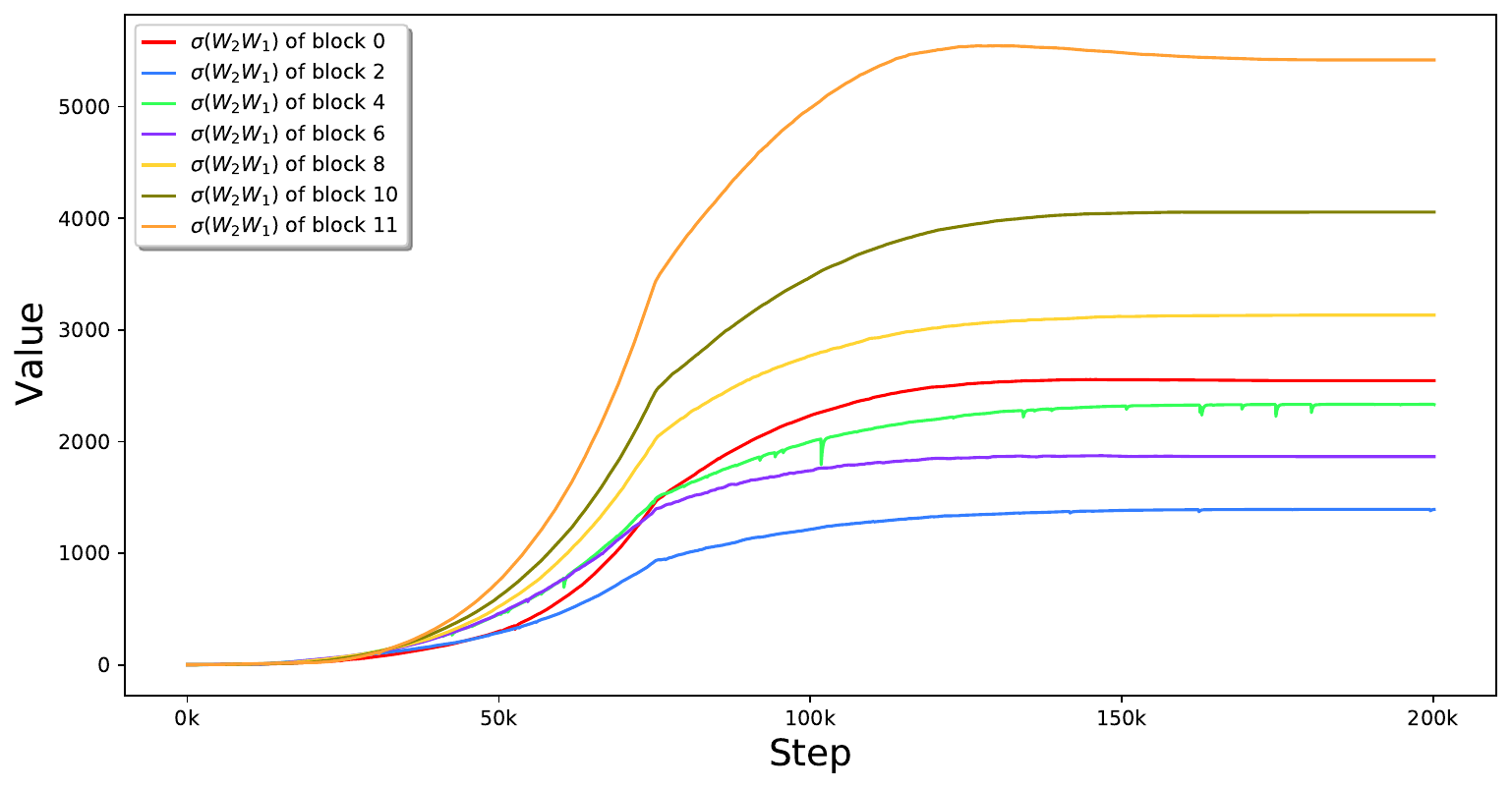}
  
		\caption*{(l) $\sigma_1\left({\bW_2\bW_1}\right)$}
		
	\end{subfigure}
        \begin{subfigure}{0.33\linewidth}
		\centering
            \includegraphics[width=4.5cm,height=3.8cm]{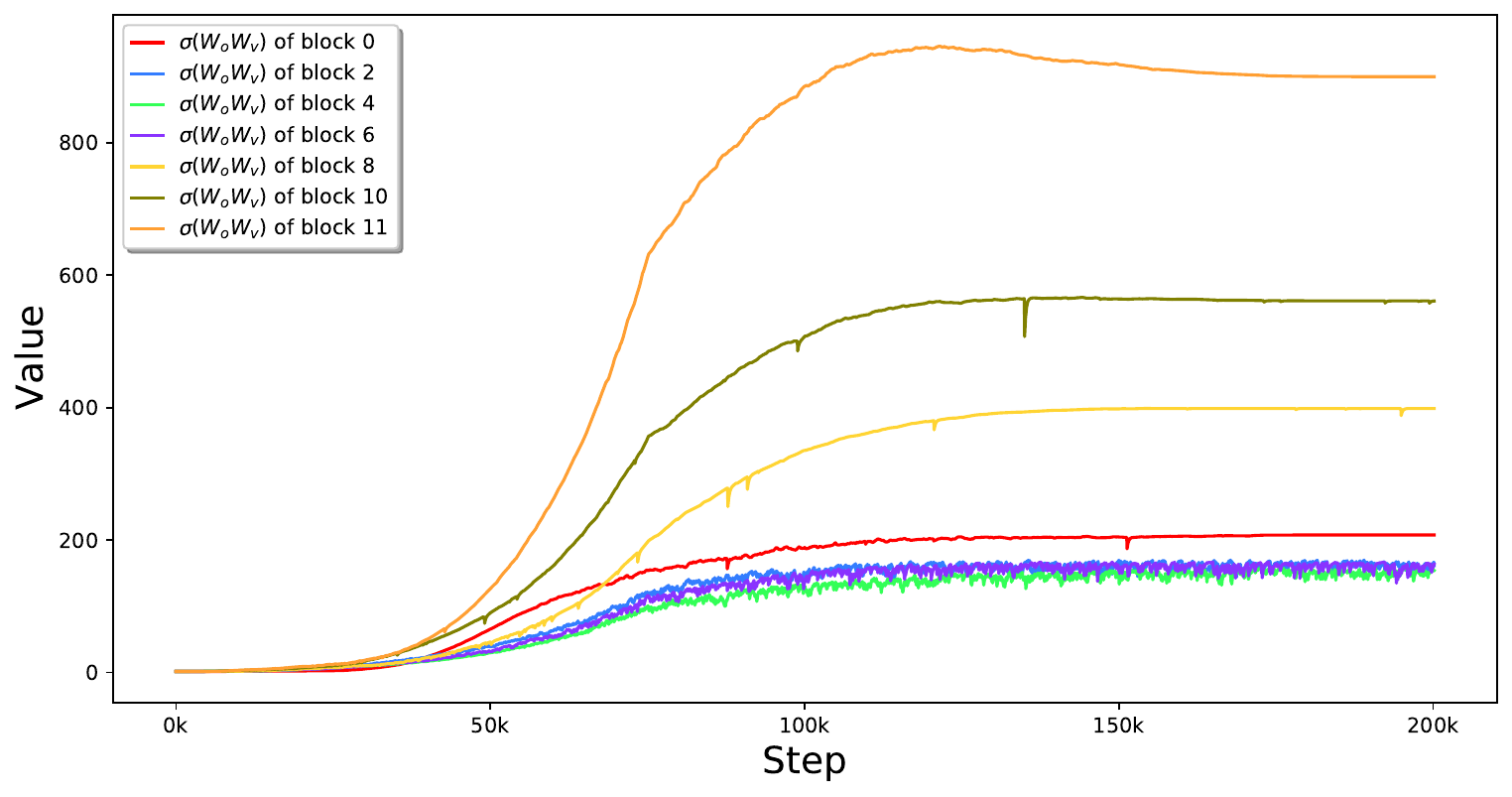}
		\caption*{(m) $\sigma_1\left({\bW_o\bW_v}\right)$}
		
	\end{subfigure}
        \begin{subfigure}{0.33\linewidth}
		\centering
            \includegraphics[width=4.5cm,height=3.8cm]{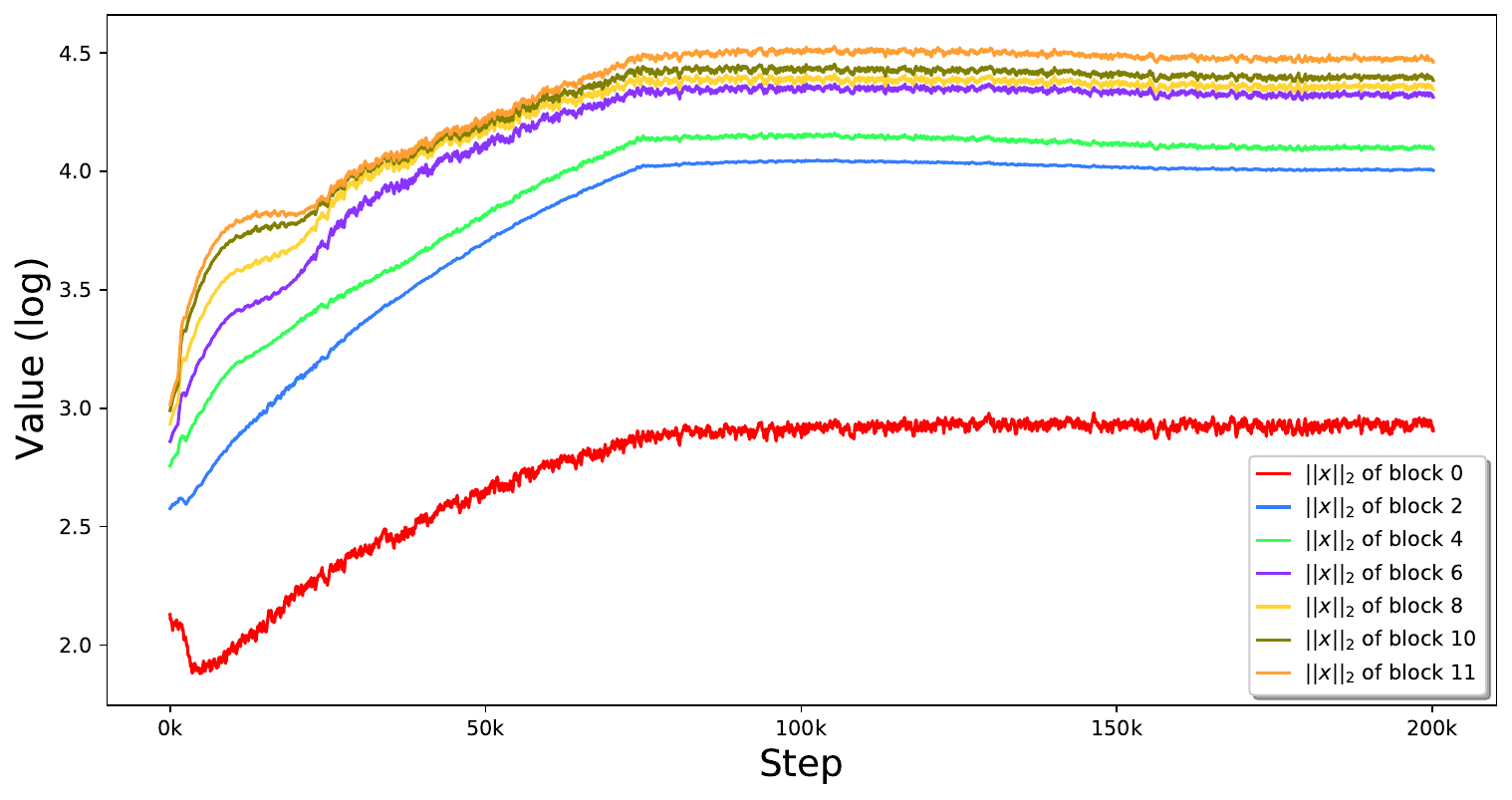}
		\caption*{(n) $\left\|{\bx}\right\|_2$}
		
	\end{subfigure}
        \begin{subfigure}{0.33\linewidth}
		\centering
            \includegraphics[width=4.5cm,height=3.8cm]{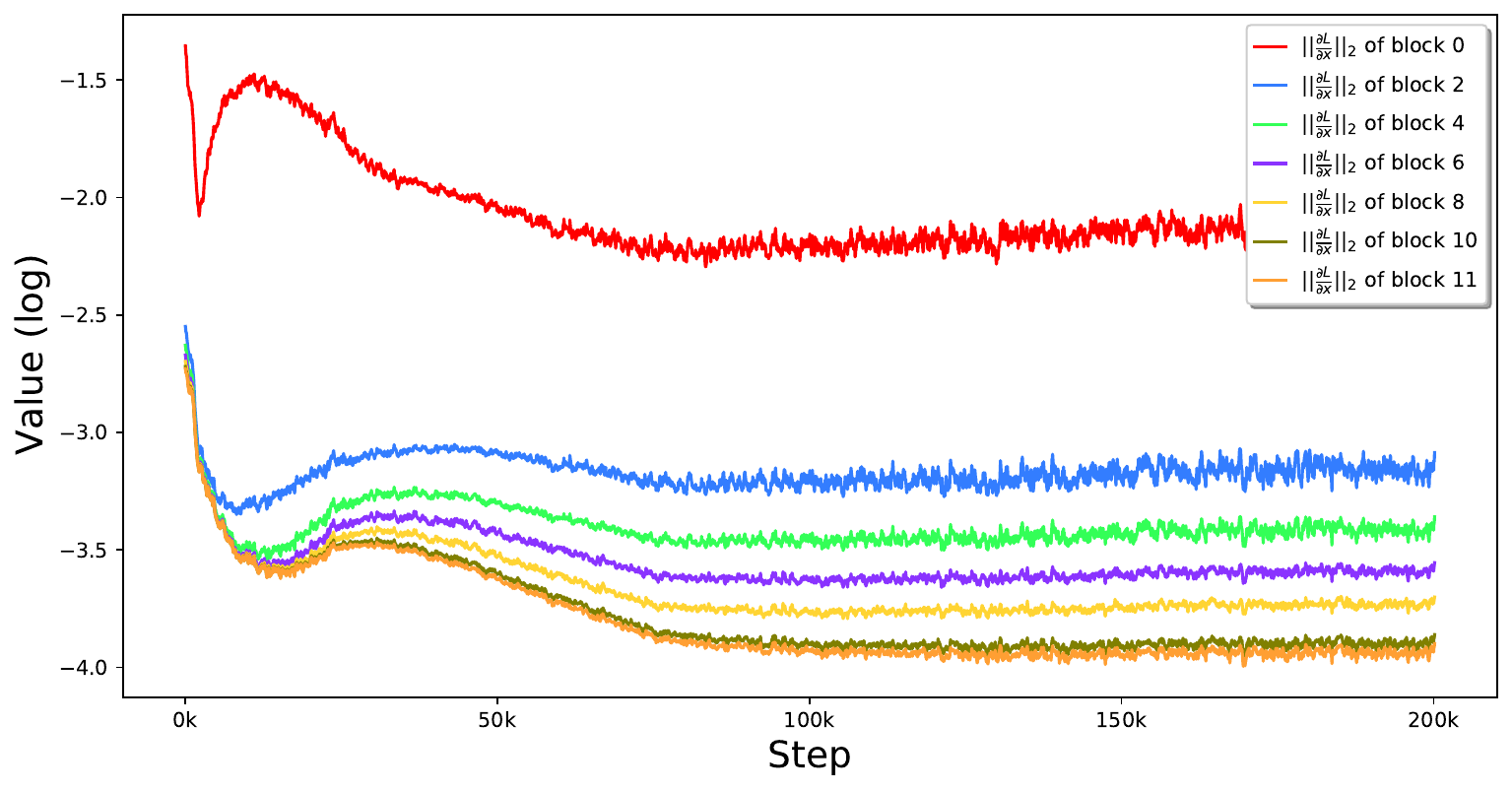}
            \caption*{(o) $\|{\frac{\partial L}{\partial \bx}}\|_2$}
		
	\end{subfigure}

 \caption{Training dynamics of a successful ViT training.} 
 \label{fig:why_vit_succusss}
\end{figure}

\begin{figure}[htbp]
	\centering
	\begin{subfigure}{0.33\linewidth}
		\centering
		\includegraphics[width=4.5cm,height=3.8cm]{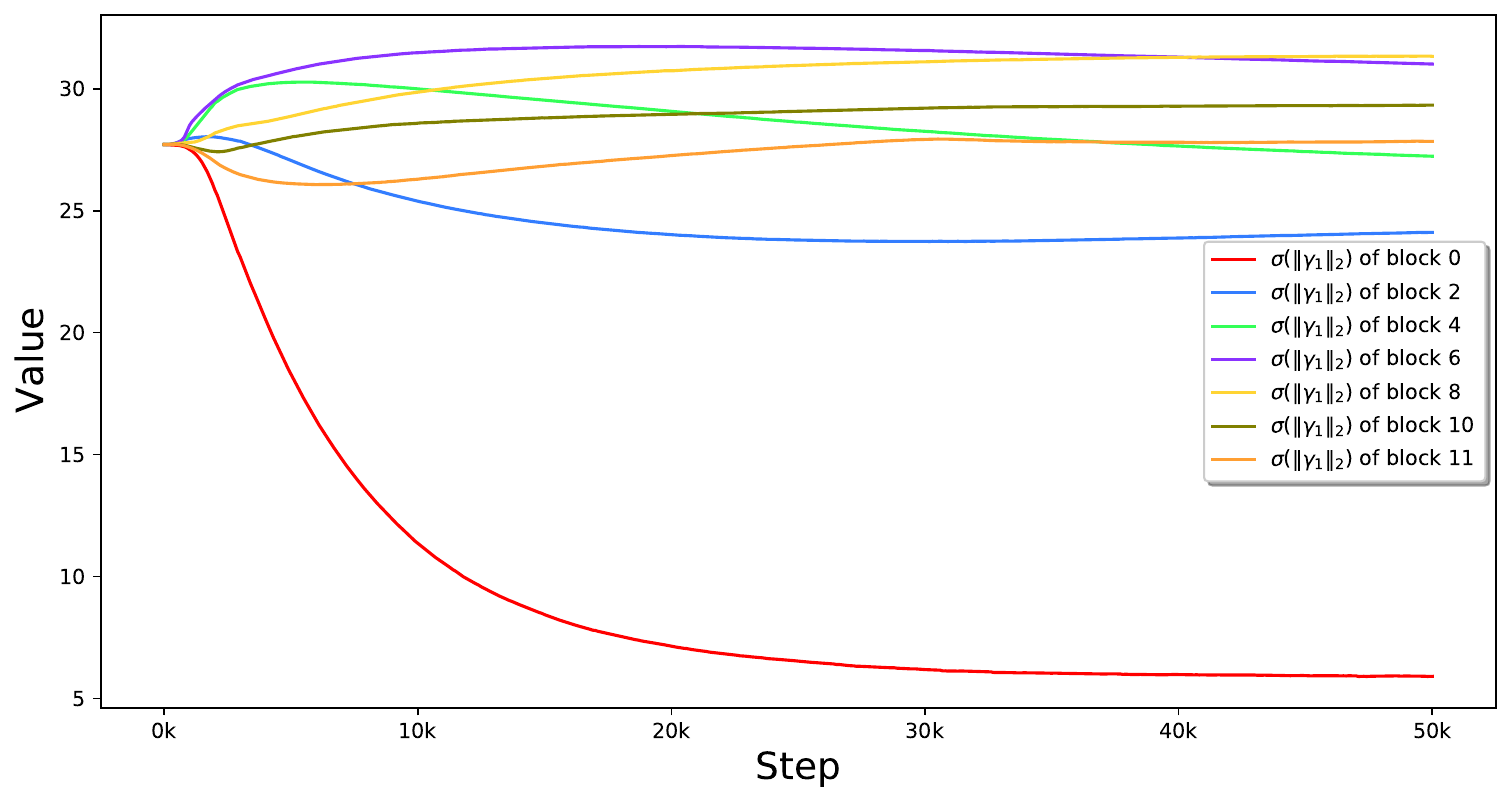}
		\caption*{(a) $\|\boldsymbol{\gamma_1}\|_2$}
		
	\end{subfigure}
	\begin{subfigure}{0.33\linewidth}
		\centering
		\includegraphics[width=4.5cm,height=3.8cm]{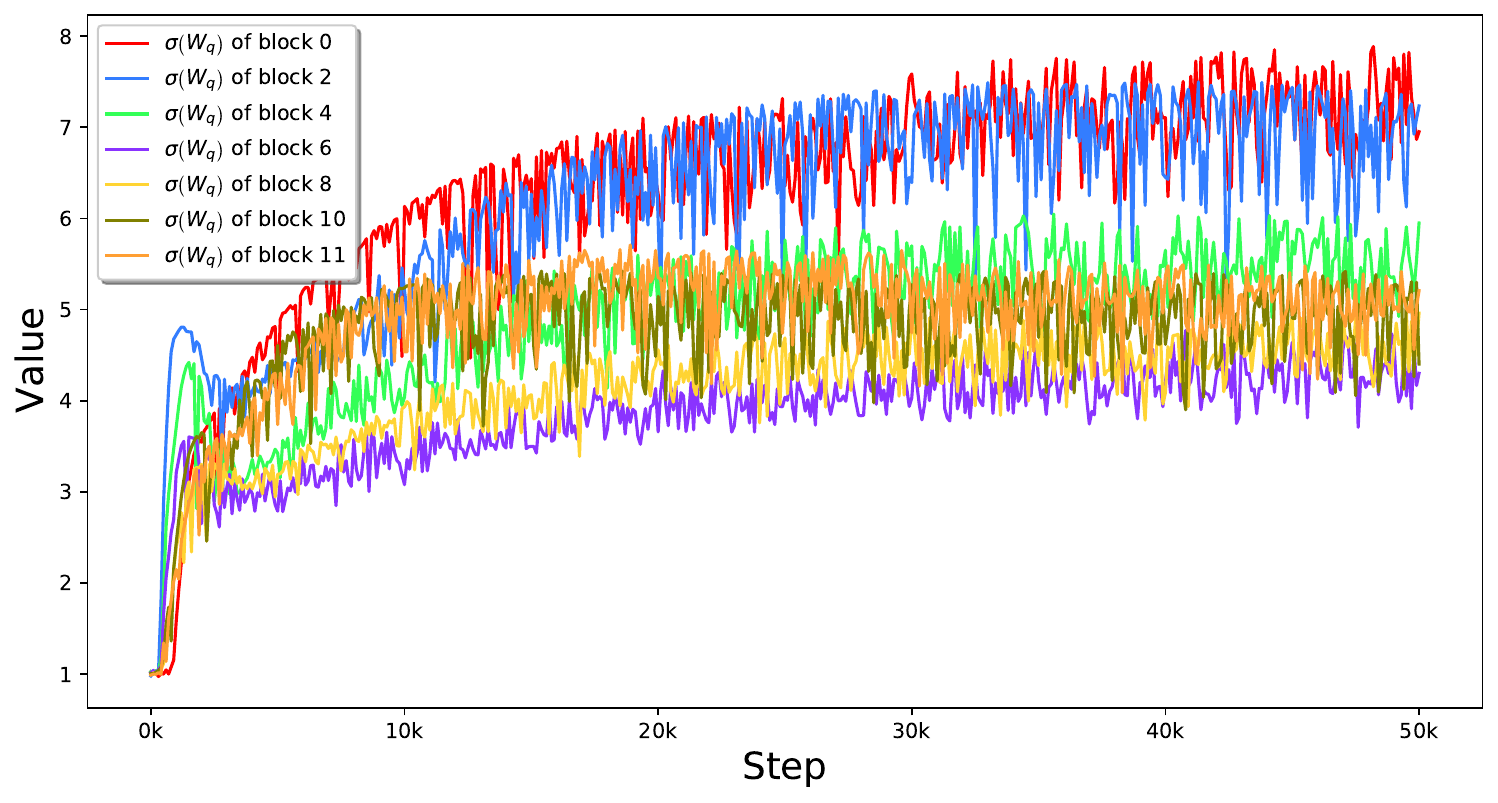}
		\caption*{(b) $\sigma_1\left({\bW_q}\right)$}
		
	\end{subfigure}
        \begin{subfigure}{0.33\linewidth}
		\centering
		\includegraphics[width=4.5cm,height=3.8cm]{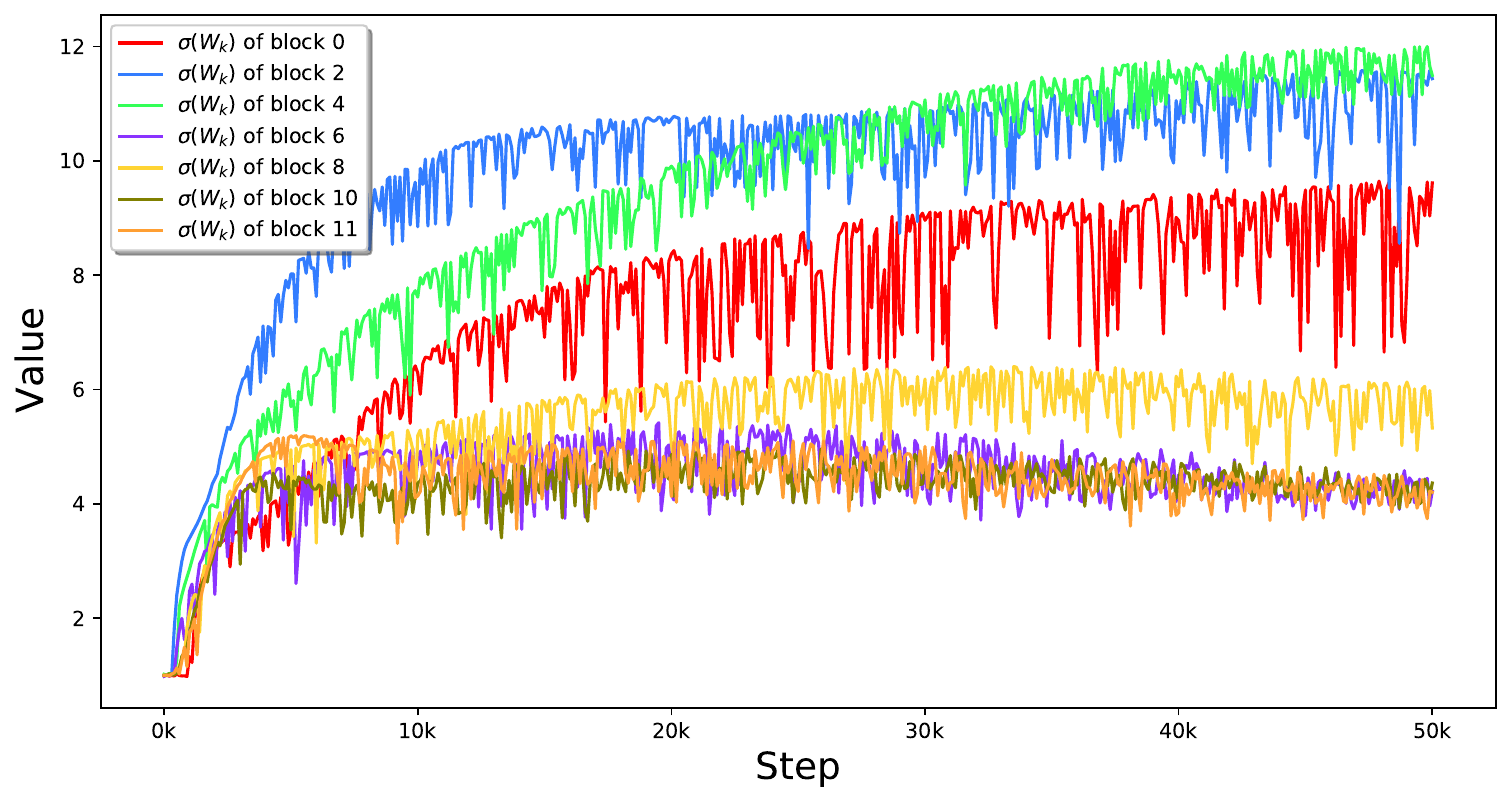}
		\caption*{(c) $\sigma_1\left({\bW_k}\right)$}
		
	\end{subfigure}
	\begin{subfigure}{0.33\linewidth}
		\centering
		\includegraphics[width=4.5cm,height=3.8cm]{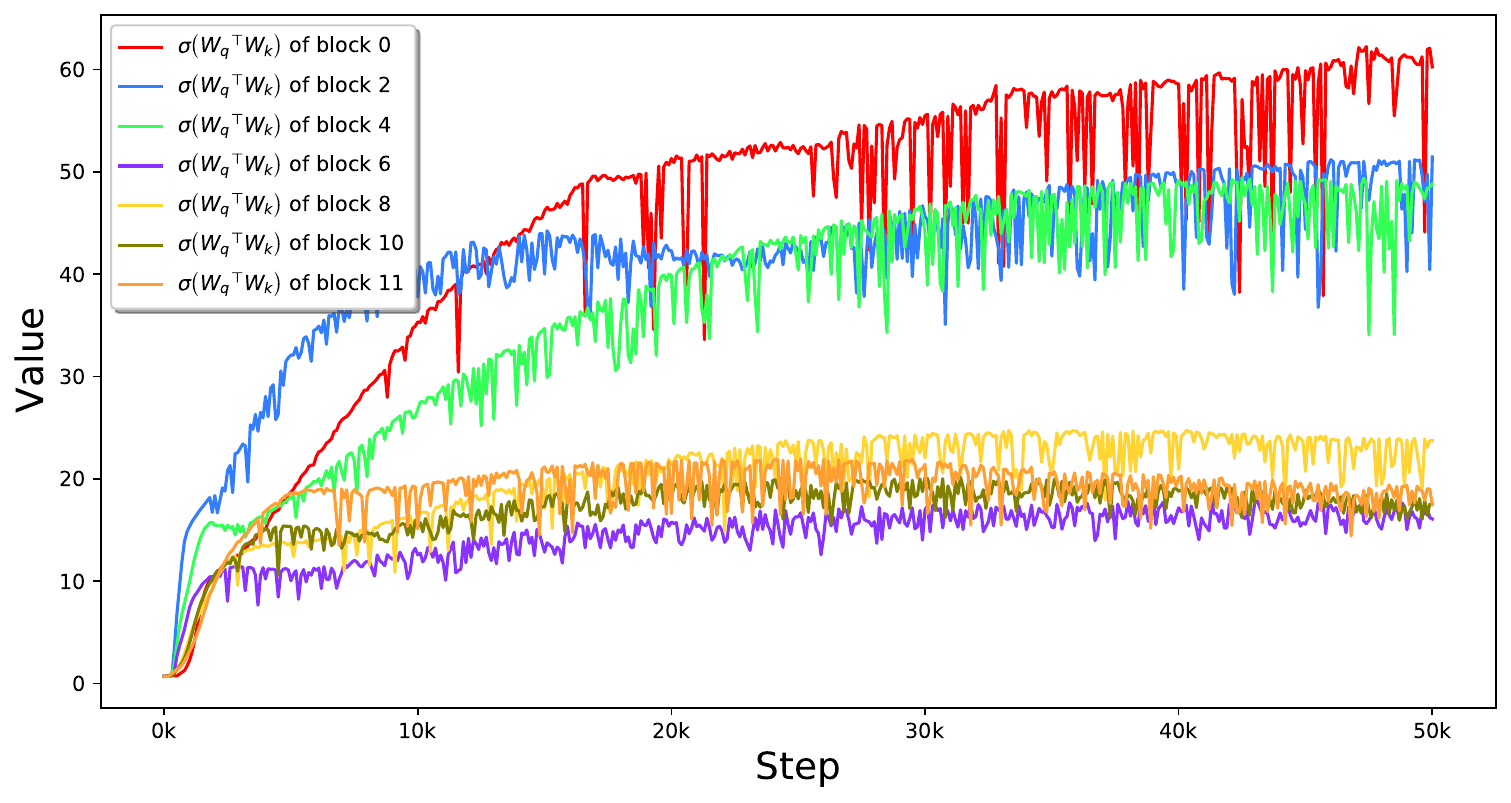}
		\caption*{(d) $\sigma_1\left({{\bW_q}^{\top}{\bW_k}}\right)$}
		\label{fig:vit_sub_success_wqwk}
	\end{subfigure}
	\begin{subfigure}{0.33\linewidth}
		\centering
		\includegraphics[width=4.5cm,height=3.8cm]{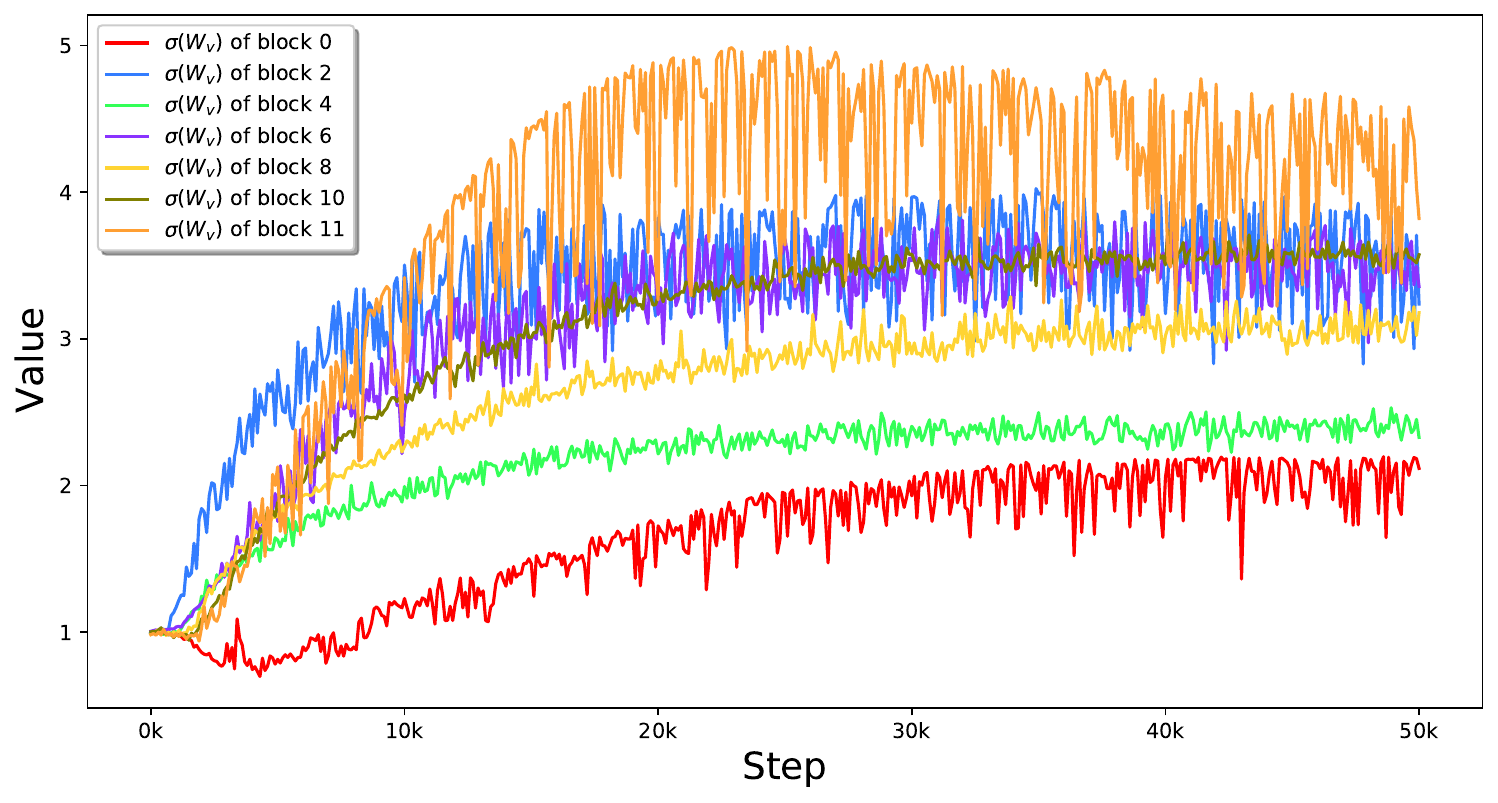}
		\caption*{(e) $\sigma_1\left({\bW_v}\right)$}
		
	\end{subfigure}

        \begin{subfigure}{0.33\linewidth}
		\centering
		\includegraphics[width=4.5cm,height=3.8cm]{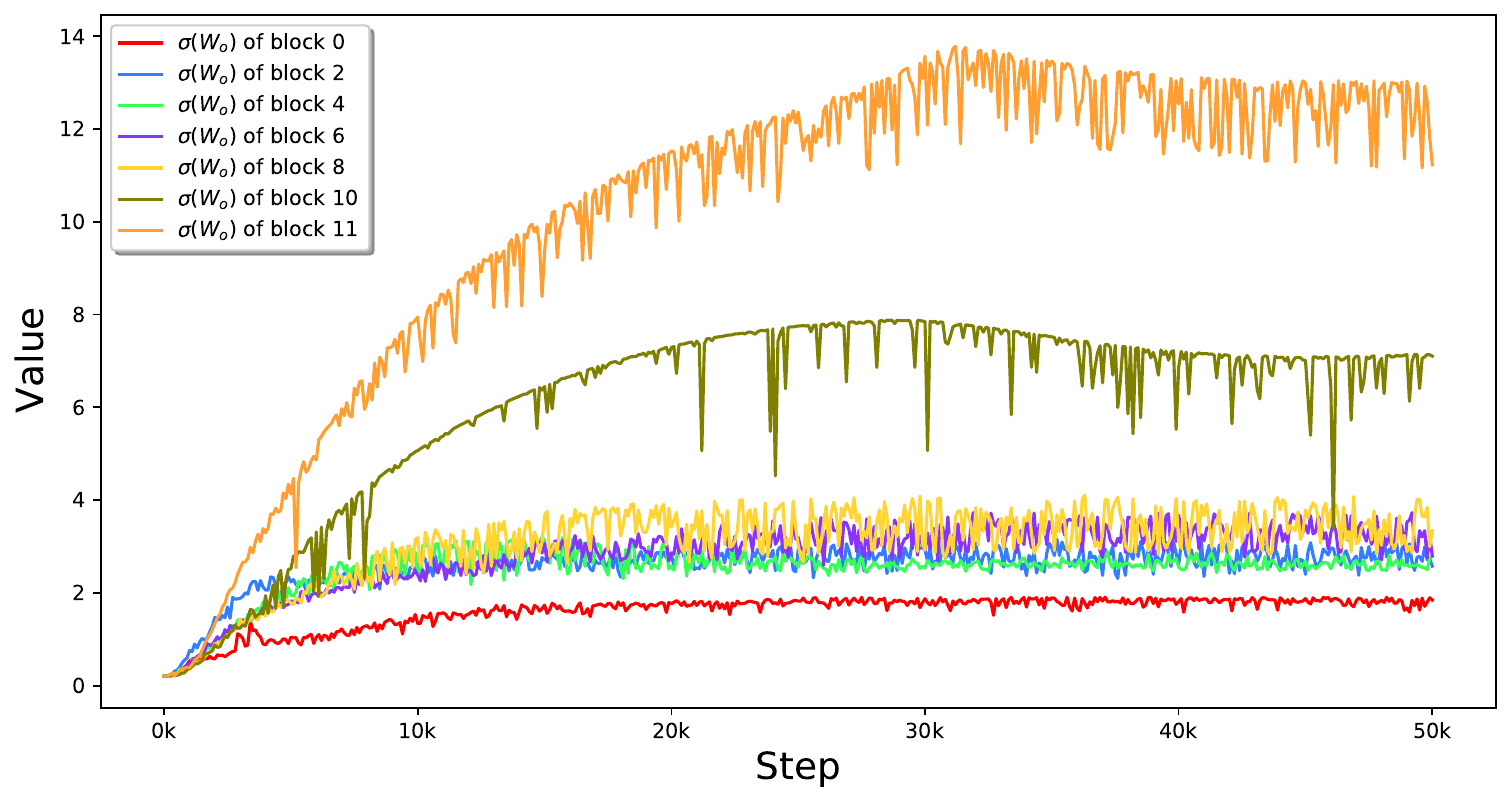}
		\caption*{(f) $\sigma_1\left({\bW_o}\right)$}
		
	\end{subfigure}
	\begin{subfigure}{0.33\linewidth}
		\centering
		\includegraphics[width=4.5cm,height=3.8cm]{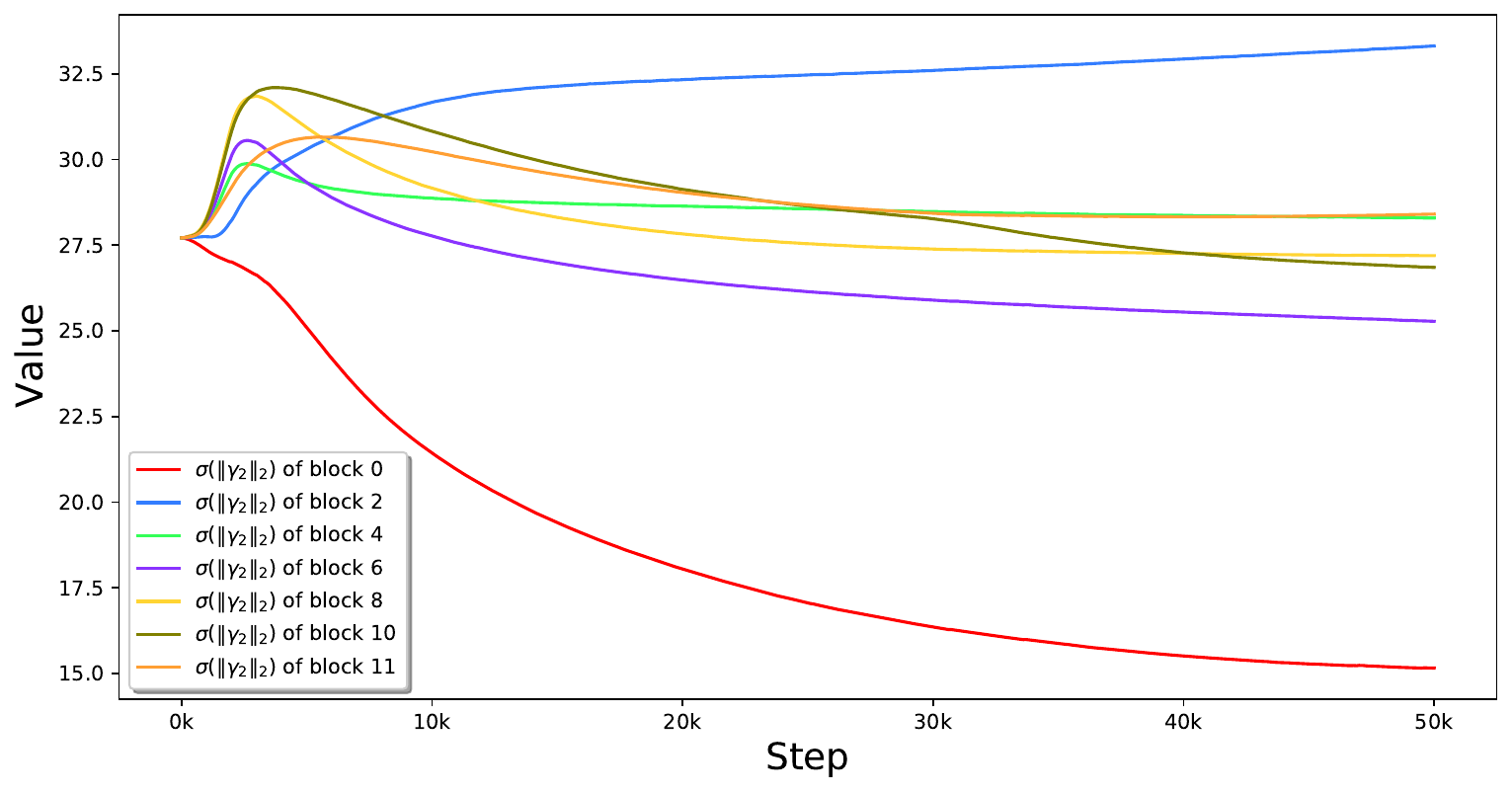}
		\caption*{(g) $\|\boldsymbol{\gamma_2}\|_2$}
		
	\end{subfigure}
	\begin{subfigure}{0.33\linewidth}
		\centering
		\includegraphics[width=4.5cm,height=3.8cm]{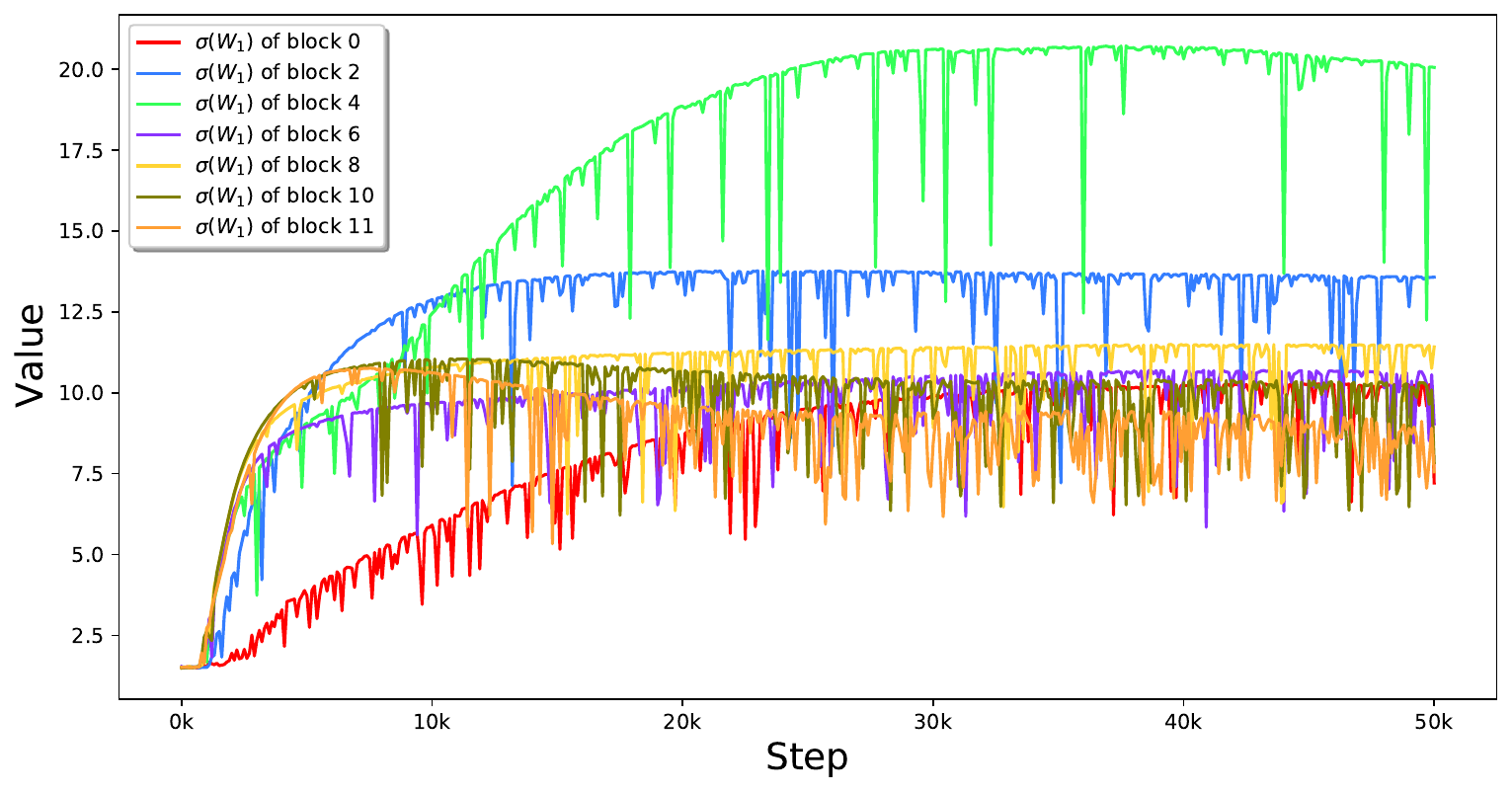}
		\caption*{(h) $\sigma_1\left({\bW_1}\right)$}
		
	\end{subfigure}
	\begin{subfigure}{0.33\linewidth}
		\centering
            \includegraphics[width=4.5cm,height=3.8cm]{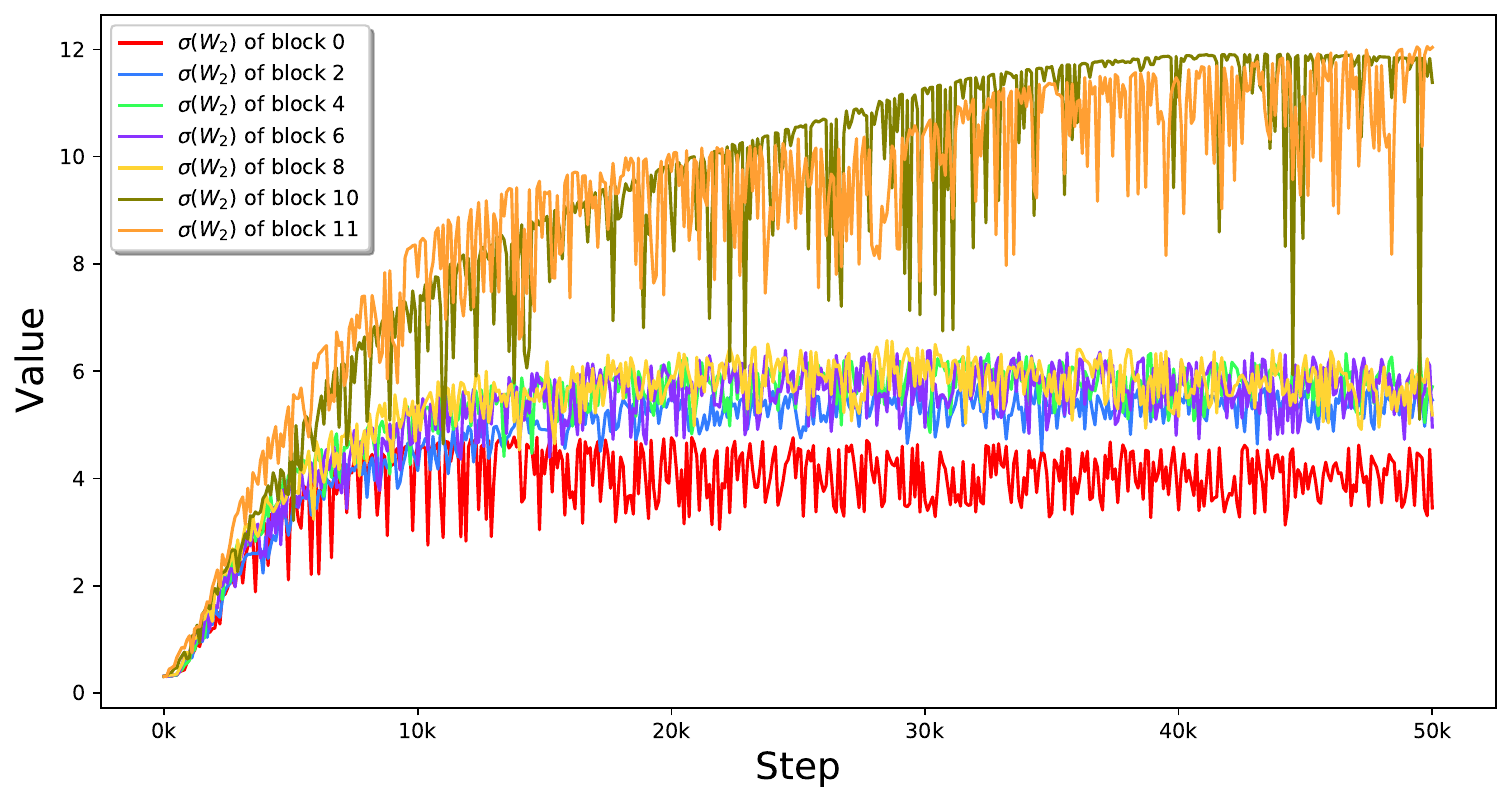}
		\caption*{(i) $\sigma_1\left({\bW_2}\right)$}
		
	\end{subfigure}
        \begin{subfigure}{0.33\linewidth}
		\centering
            \includegraphics[width=4.5cm,height=3.8cm]{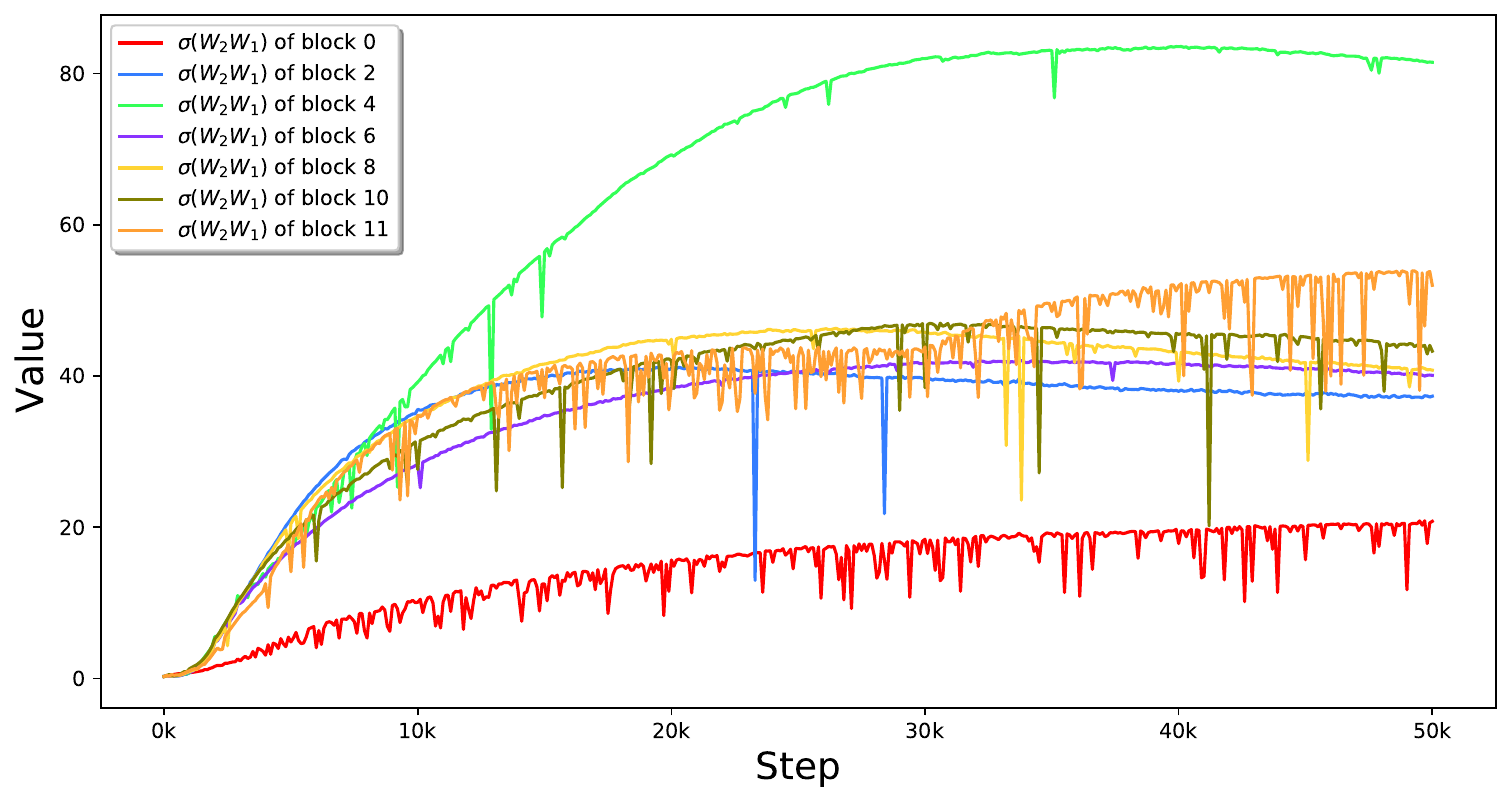}
  
		\caption*{(j) $\sigma_1\left({\bW_2\bW_1}\right)$}
		
	\end{subfigure}
        \begin{subfigure}{0.33\linewidth}
		\centering
            \includegraphics[width=4.5cm,height=3.8cm]{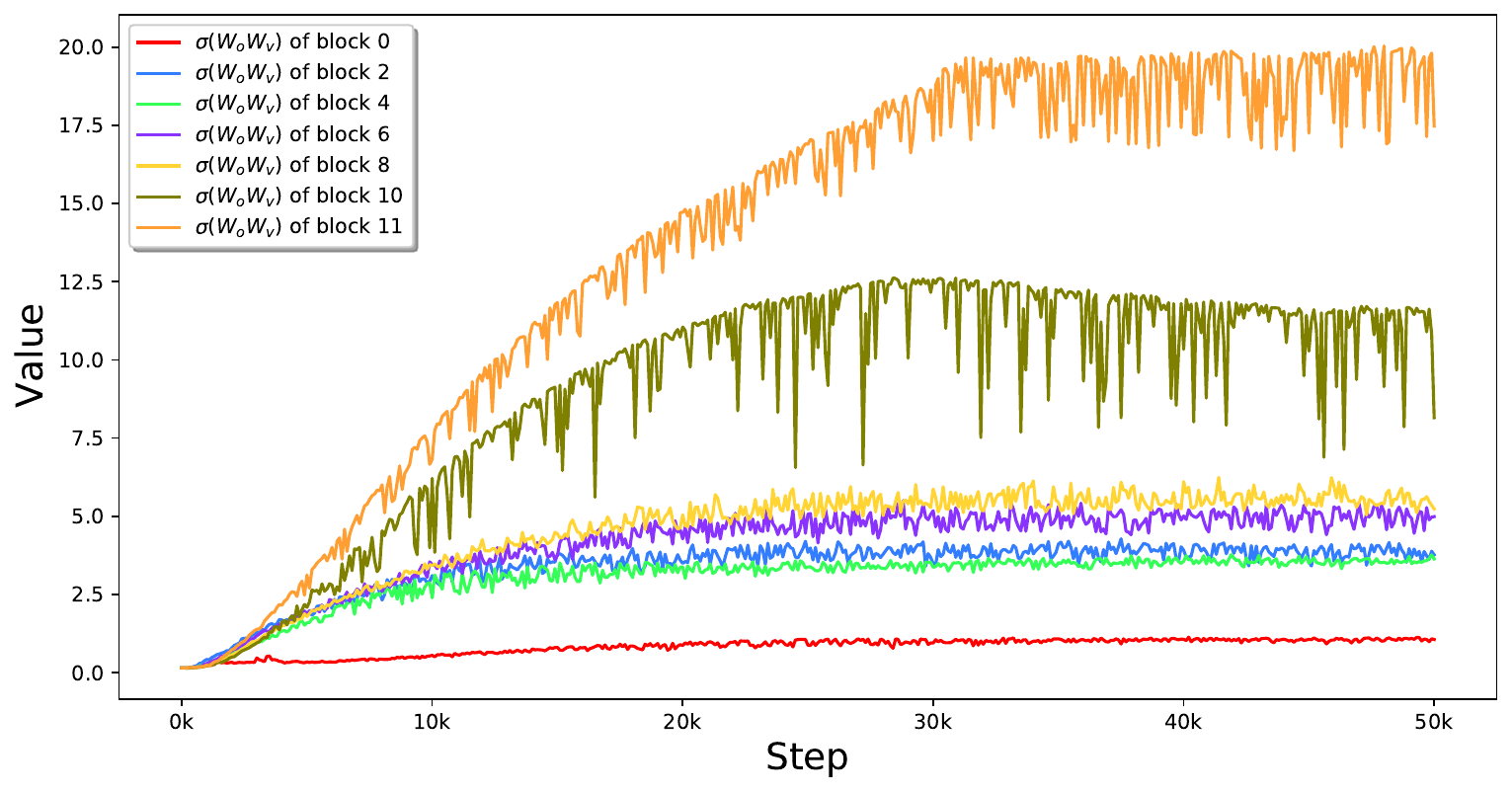}
		\caption*{(k) $\sigma_1\left({\bW_o\bW_v}\right)$}
		
	\end{subfigure}
        \begin{subfigure}{0.33\linewidth}
		\centering
            \includegraphics[width=4.5cm,height=3.8cm]{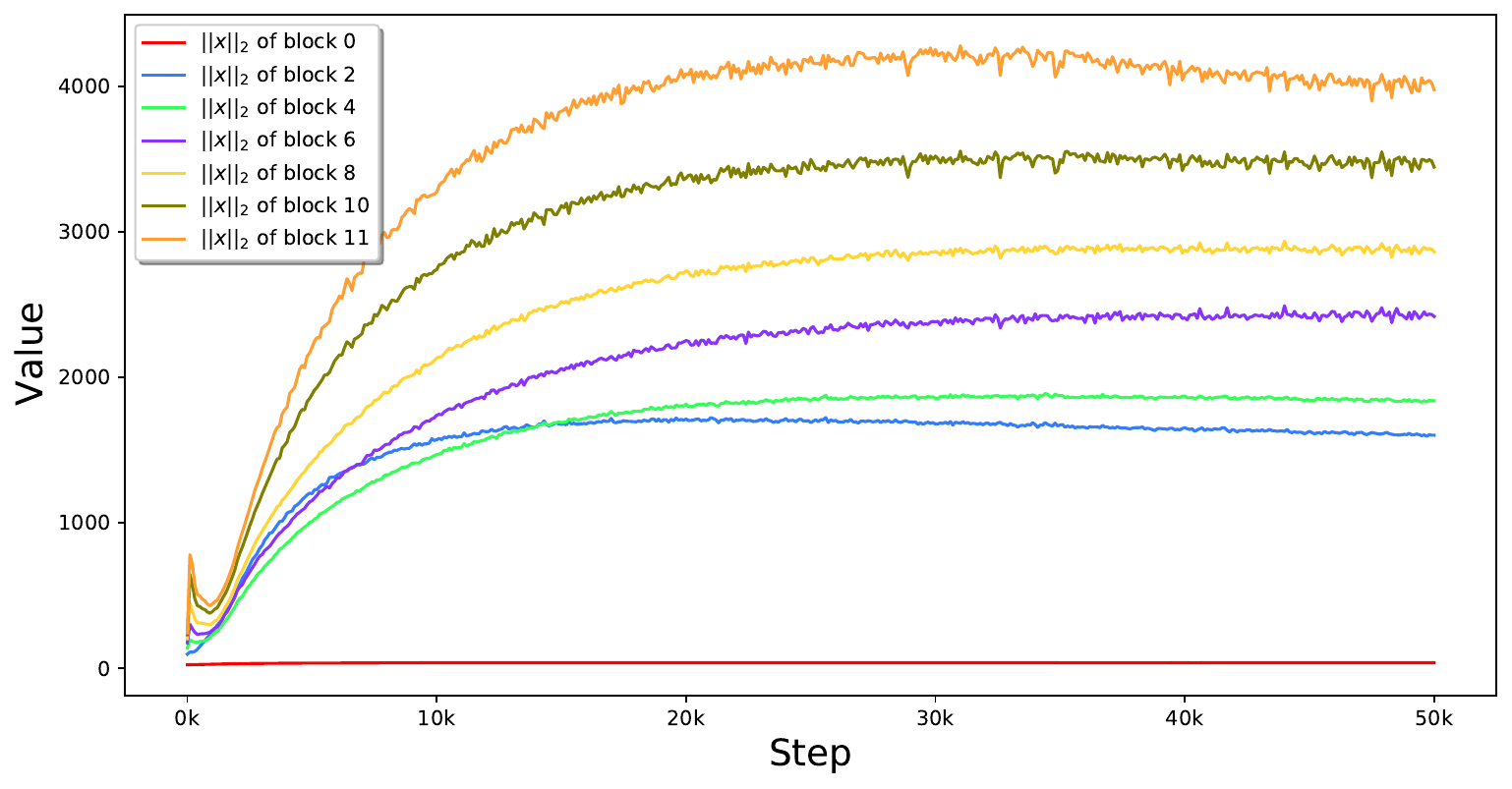}
		\caption*{(l) $\left\|{\bx}\right\|_2$}
		
	\end{subfigure}
        \begin{subfigure}{0.33\linewidth}
		\centering
            \includegraphics[width=4.5cm,height=3.8cm]{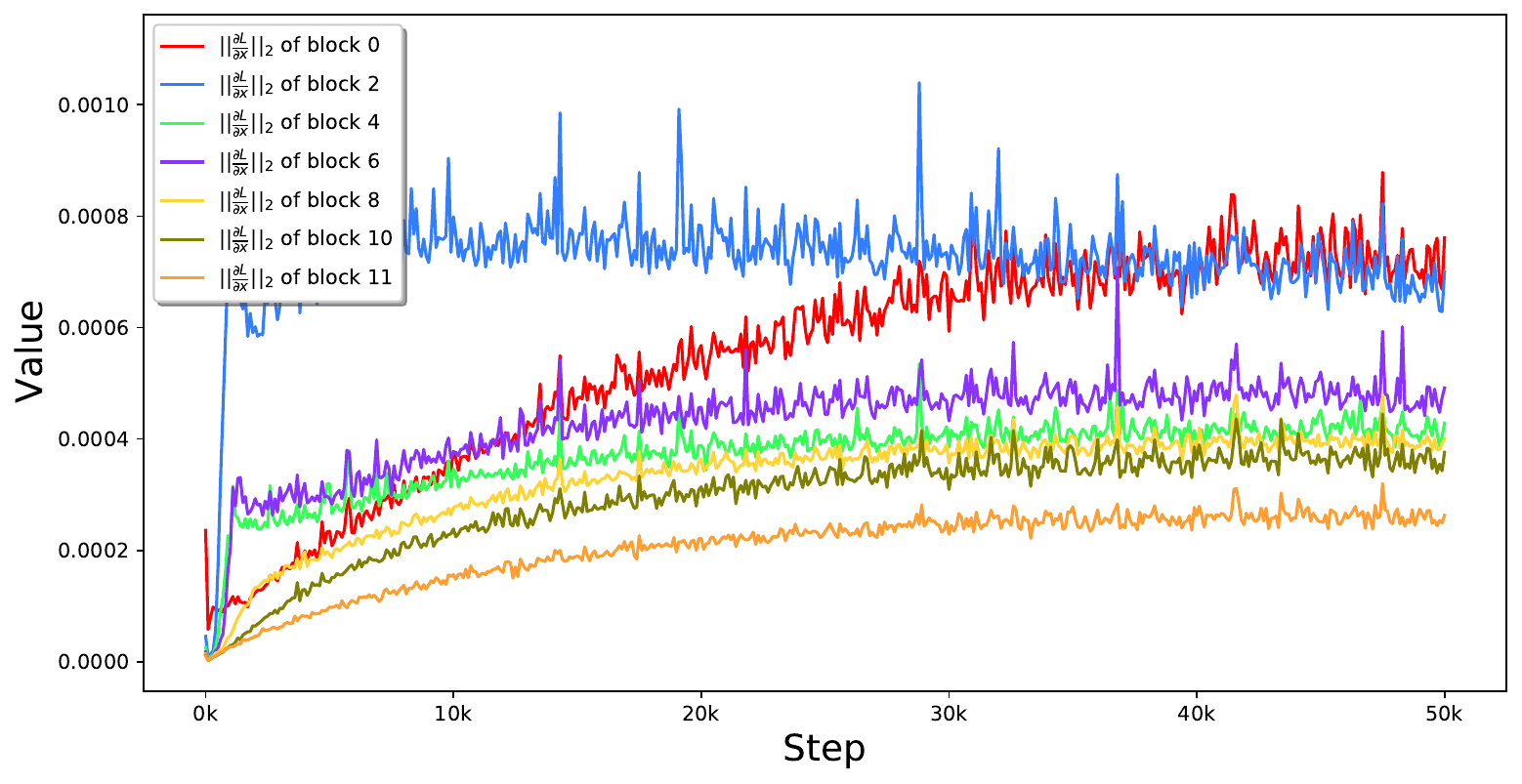}
            \caption*{(m) $\|{\frac{\partial L}{\partial \bx}}\|_2$}
		
	\end{subfigure}
 \caption{Training dynamics of a successful GPT training.} 
 \label{fig:why_gpt_succusss}
\end{figure}

\begin{figure}[htbp]
	\centering
	\begin{subfigure}{0.33\linewidth}
		\centering
		\includegraphics[width=4.5cm,height=3.8cm]{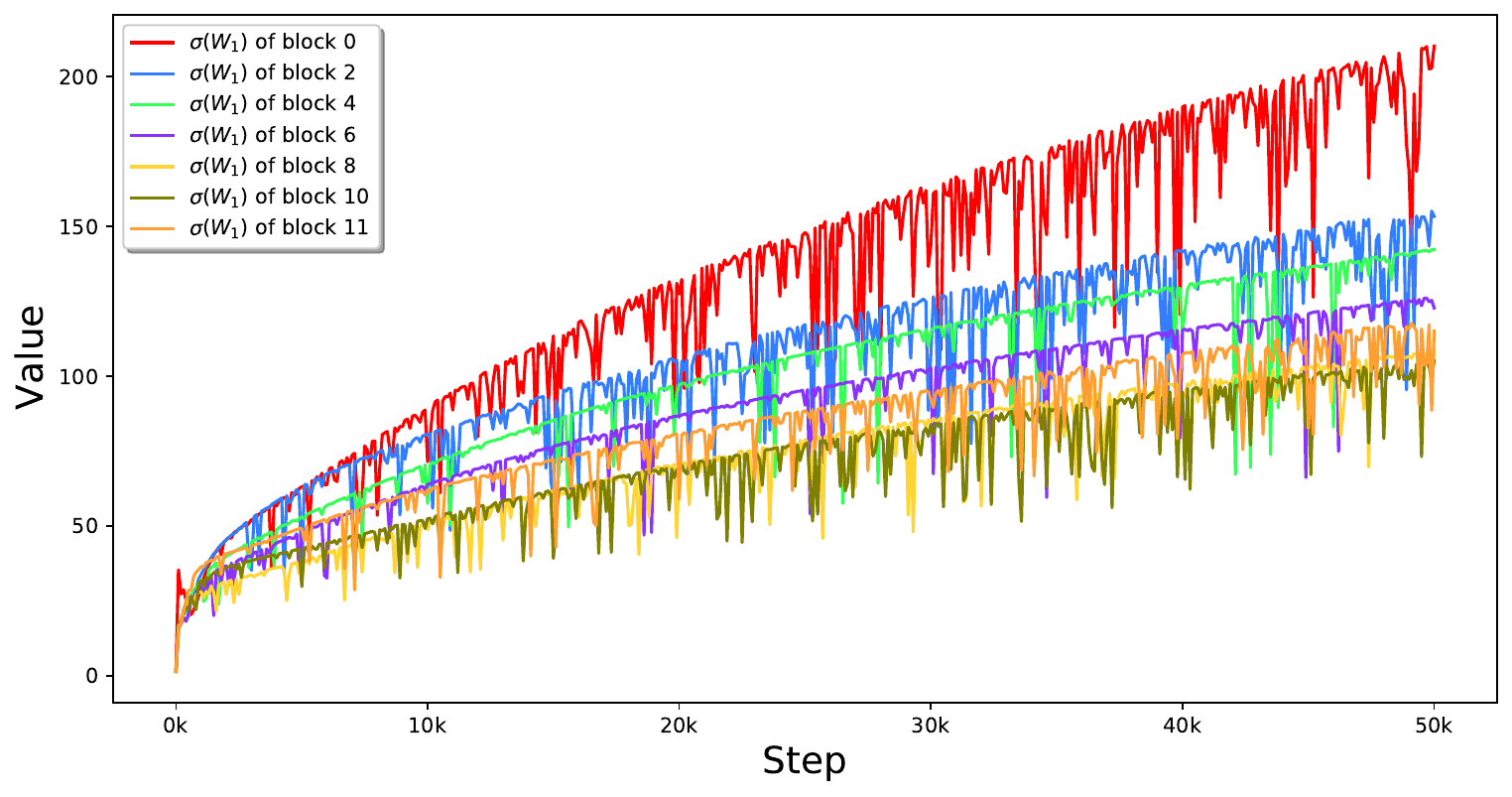}
		\caption*{(a) $\|\boldsymbol{\gamma_1}\|_2$}
		
	\end{subfigure}
	\begin{subfigure}{0.33\linewidth}
		\centering
		\includegraphics[width=4.5cm,height=3.8cm]{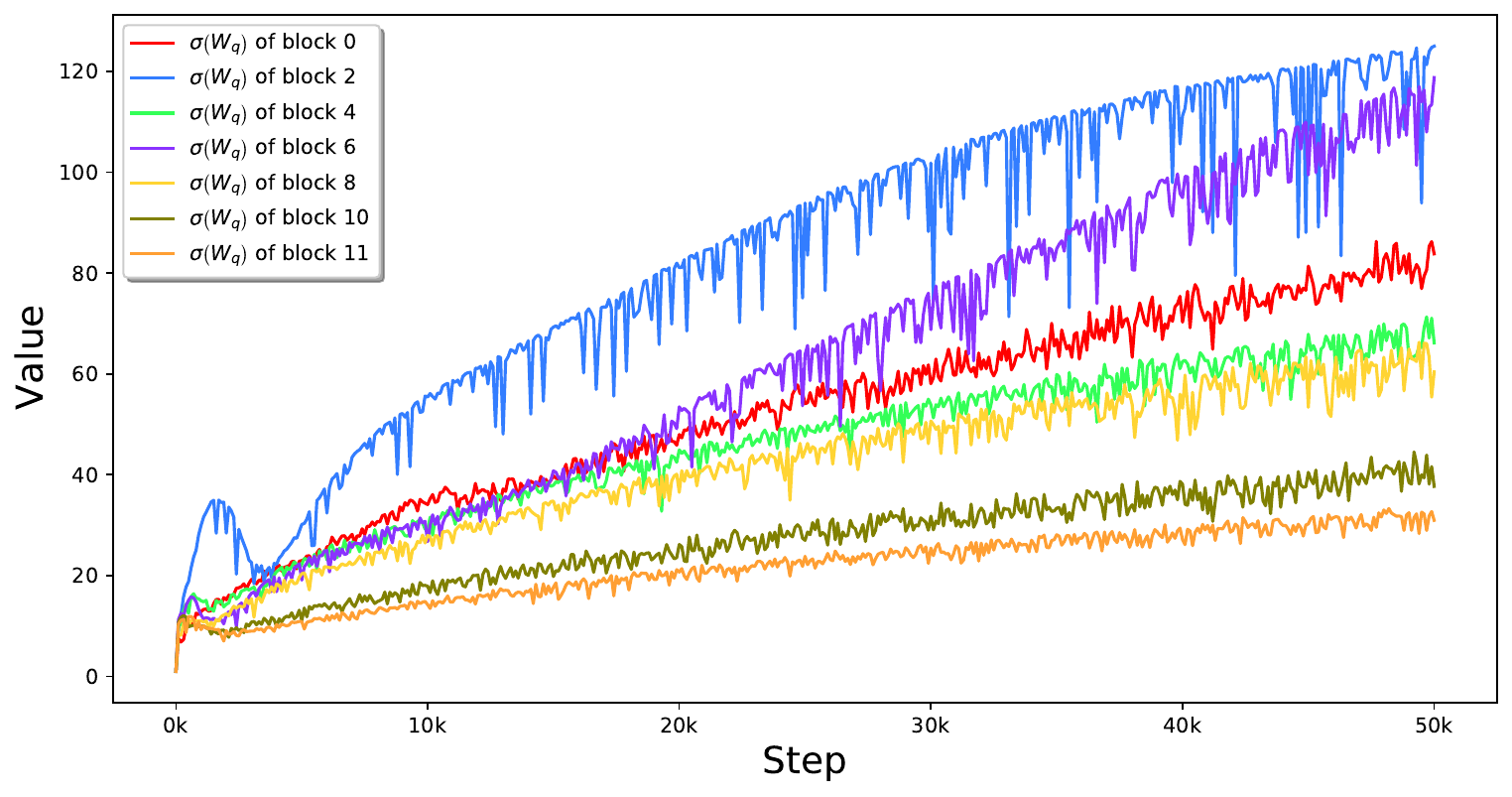}
		\caption*{(b) $\sigma_1\left({\bW_q}\right)$}
		
	\end{subfigure}
        \begin{subfigure}{0.33\linewidth}
		\centering
		\includegraphics[width=4.5cm,height=3.8cm]{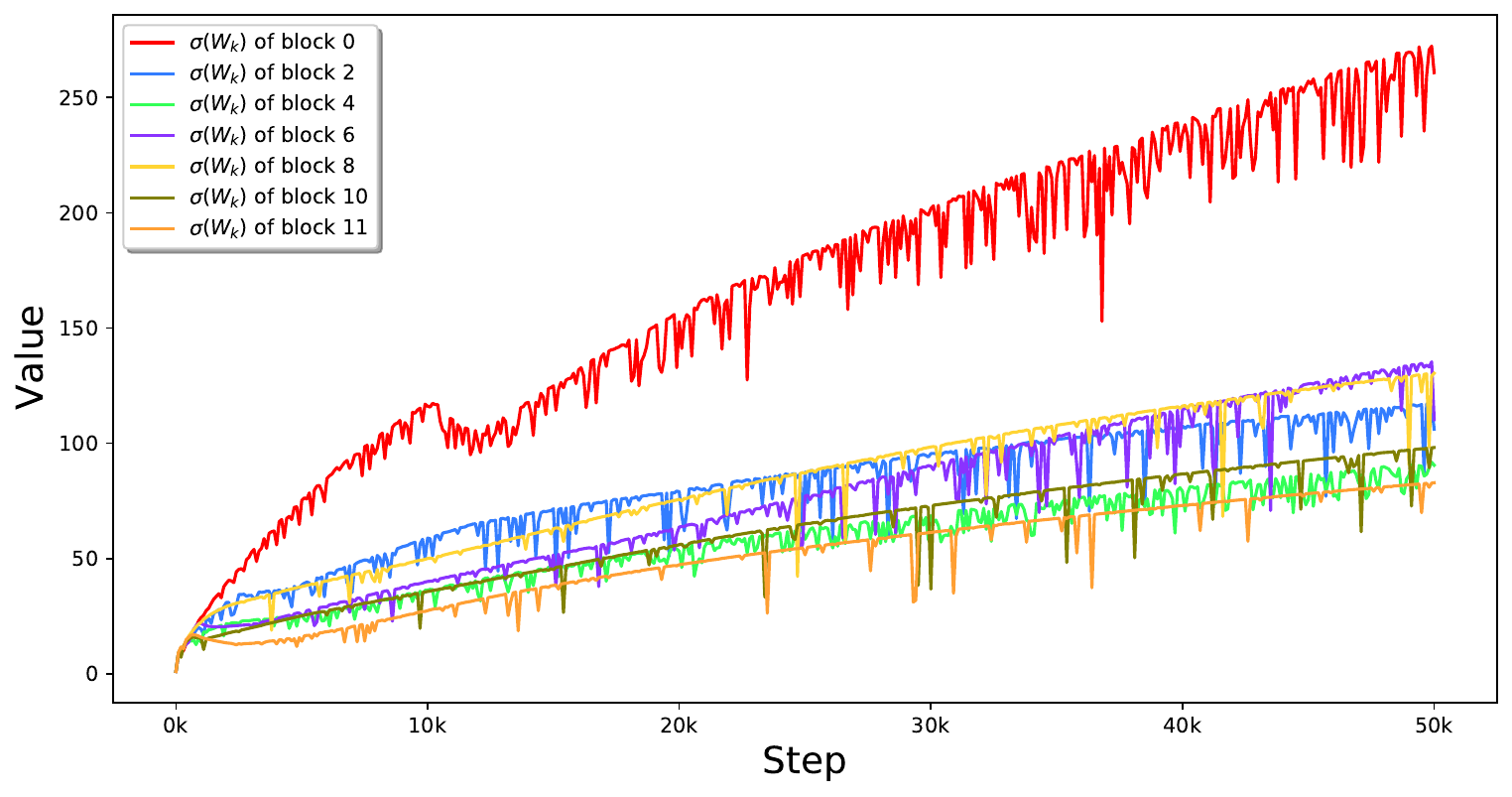}
		\caption*{(c) $\sigma_1\left({\bW_k}\right)$}
		
	\end{subfigure}
	\begin{subfigure}{0.33\linewidth}
		\centering
		\includegraphics[width=4.5cm,height=3.8cm]{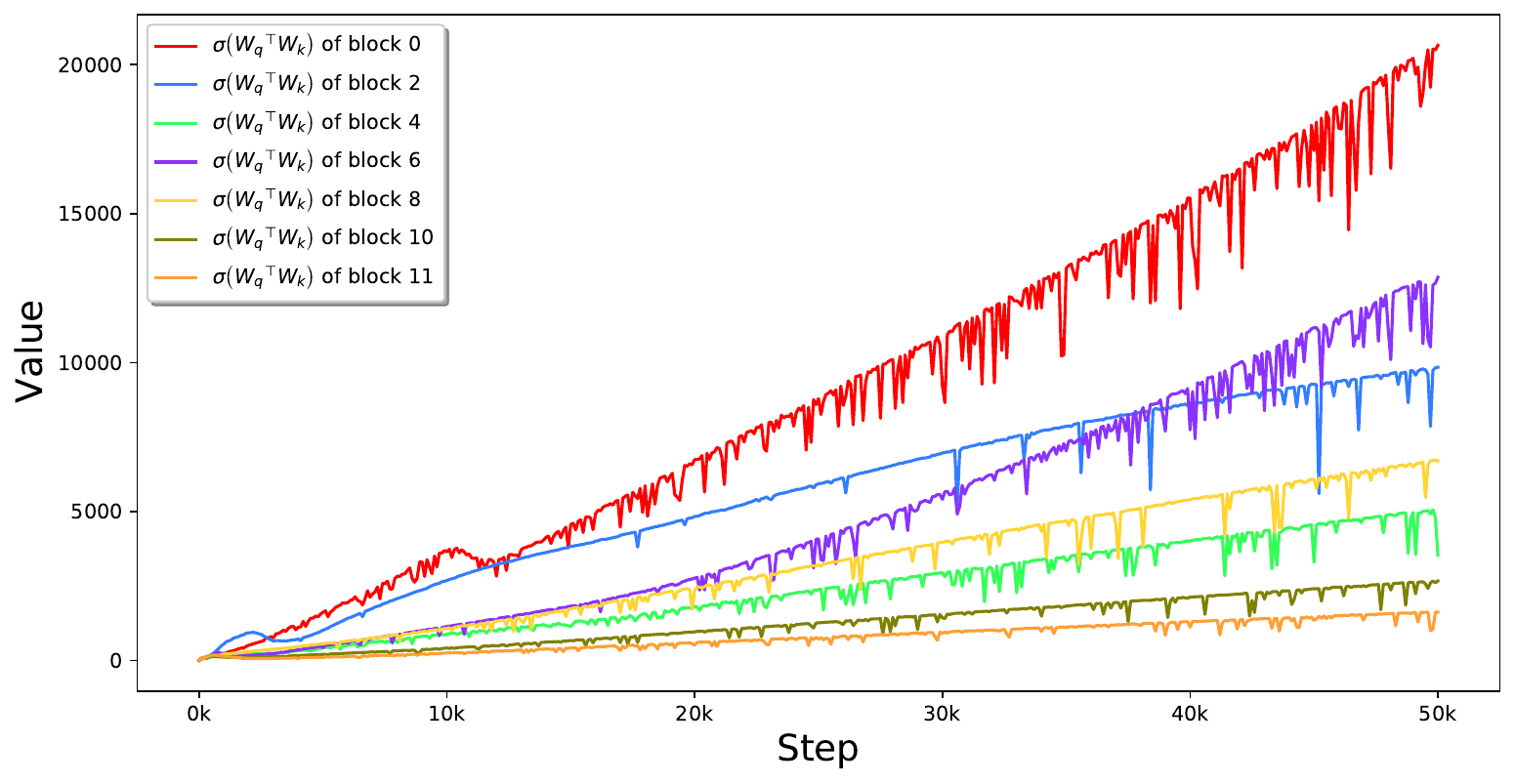}
		\caption*{(d) $\sigma_1\left({{\bW_q}^{\top}{\bW_k}}\right)$}
		\label{fig:vit_sub_success_wqwk}
	\end{subfigure}
	\begin{subfigure}{0.33\linewidth}
		\centering
		\includegraphics[width=4.5cm,height=3.8cm]{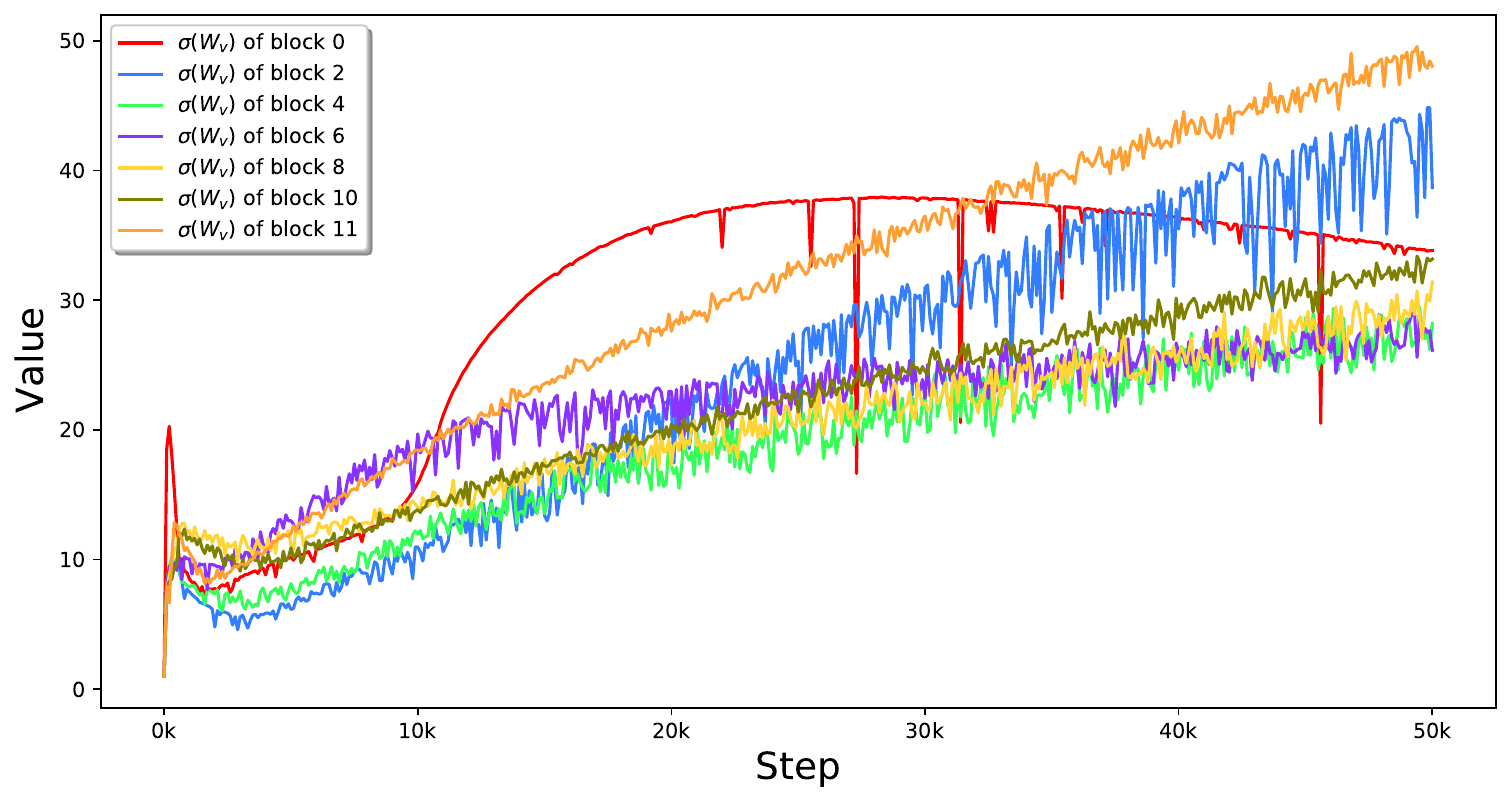}
		\caption*{(e) $\sigma_1\left({\bW_v}\right)$}
		
	\end{subfigure}

        \begin{subfigure}{0.33\linewidth}
		\centering
		\includegraphics[width=4.5cm,height=3.8cm]{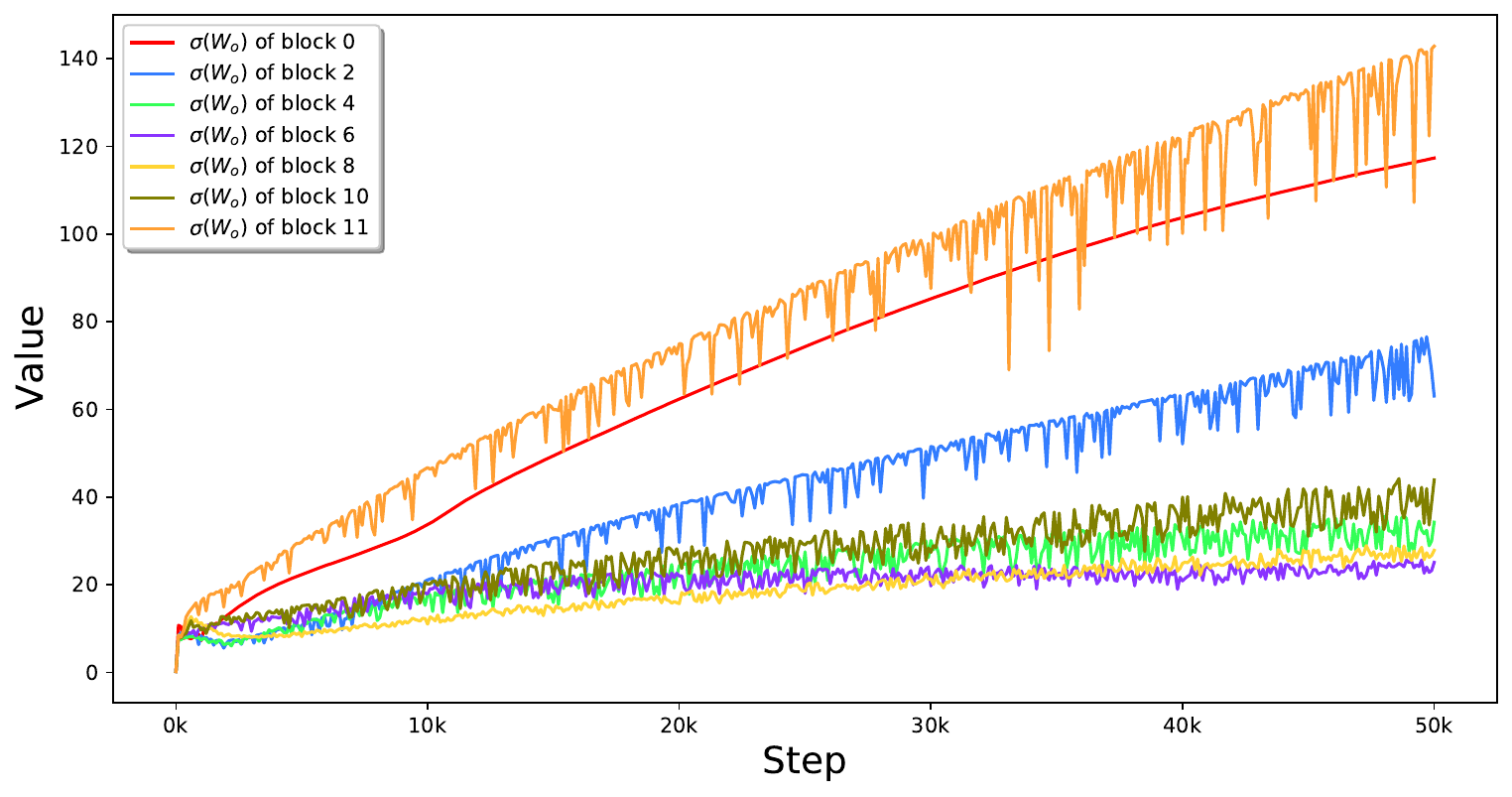}
		\caption*{(f) $\sigma_1\left({\bW_o}\right)$}
		
	\end{subfigure}
	\begin{subfigure}{0.33\linewidth}
		\centering
		\includegraphics[width=4.5cm,height=3.8cm]{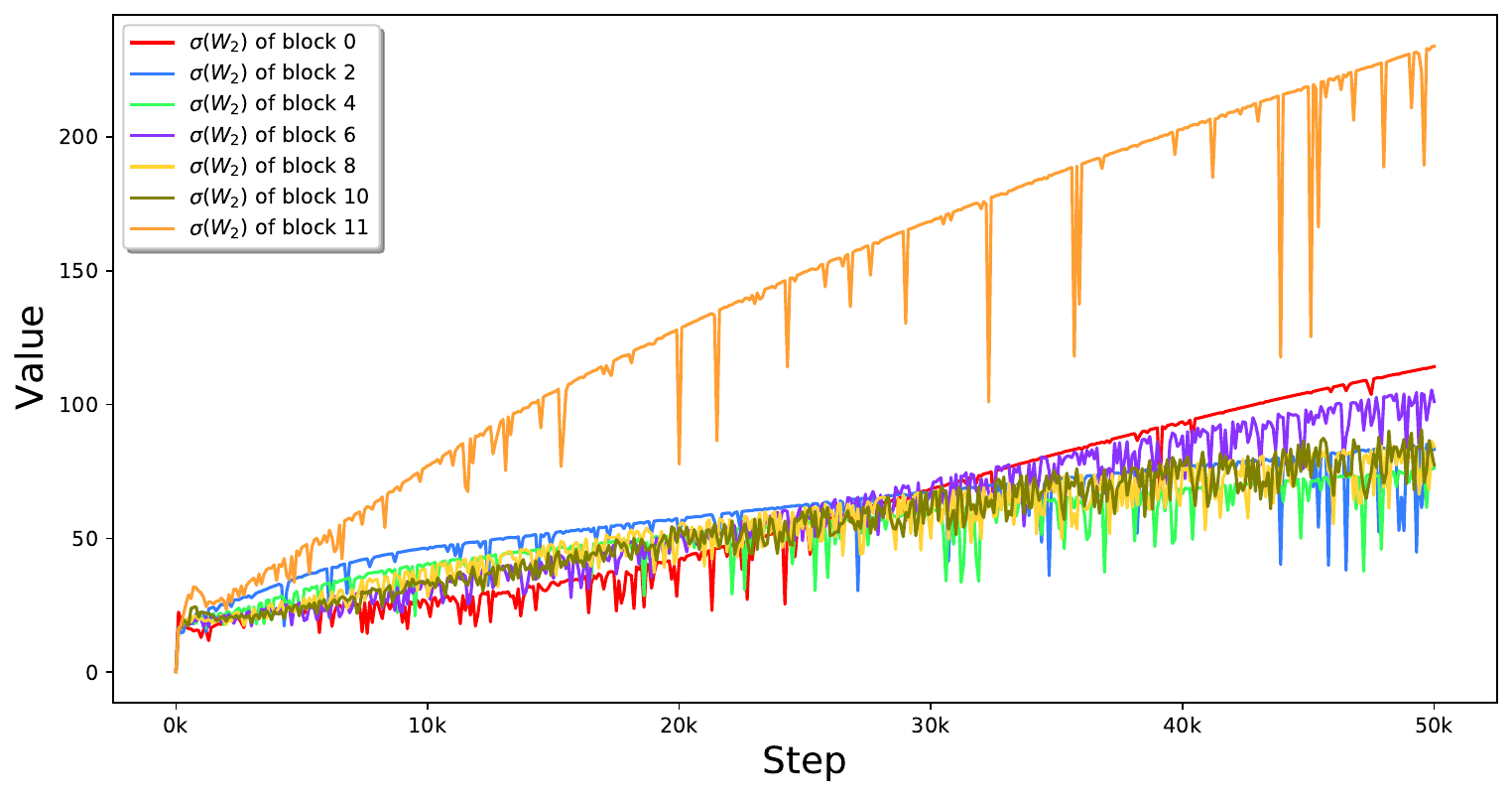}
		\caption*{(g) $\|\boldsymbol{\gamma_2}\|_2$}
		
	\end{subfigure}
	\begin{subfigure}{0.33\linewidth}
		\centering
		\includegraphics[width=4.5cm,height=3.8cm]{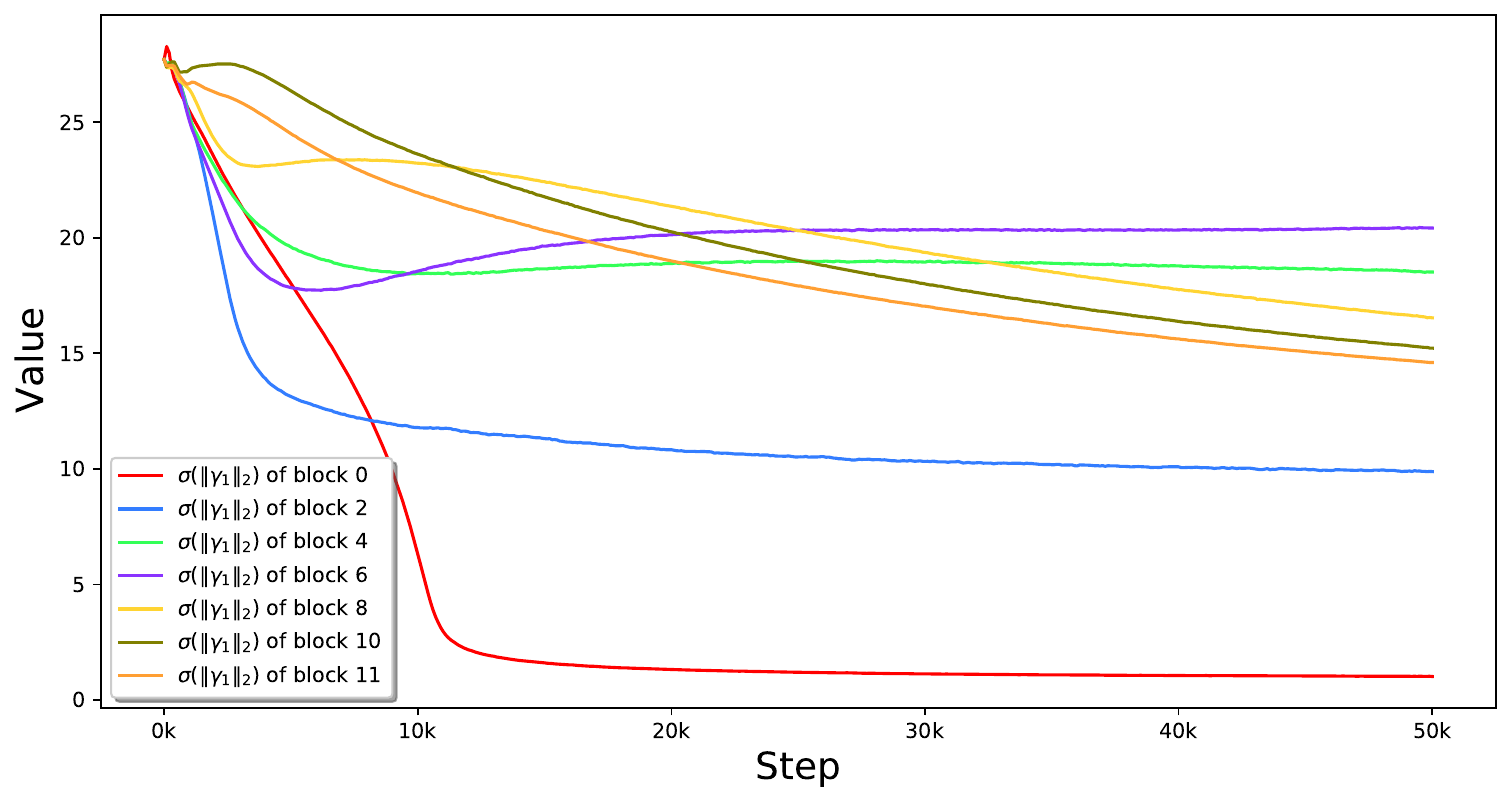}
		\caption*{(h) $\sigma_1\left({\bW_1}\right)$}
		
	\end{subfigure}
	\begin{subfigure}{0.33\linewidth}
		\centering
            \includegraphics[width=4.5cm,height=3.8cm]{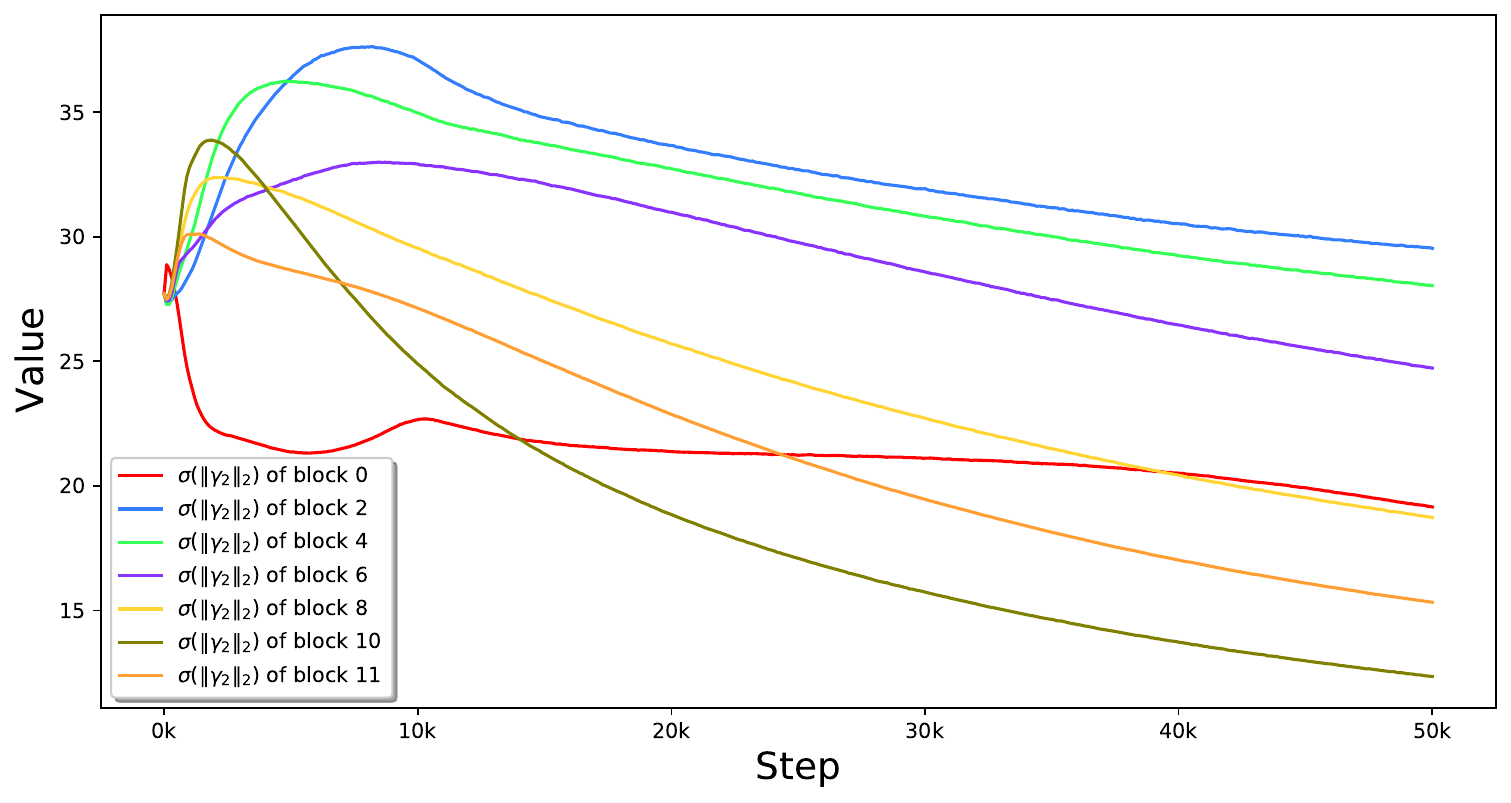}
		\caption*{(i) $\sigma_1\left({\bW_2}\right)$}
		
	\end{subfigure}
        \begin{subfigure}{0.33\linewidth}
		\centering
            \includegraphics[width=4.5cm,height=3.8cm]{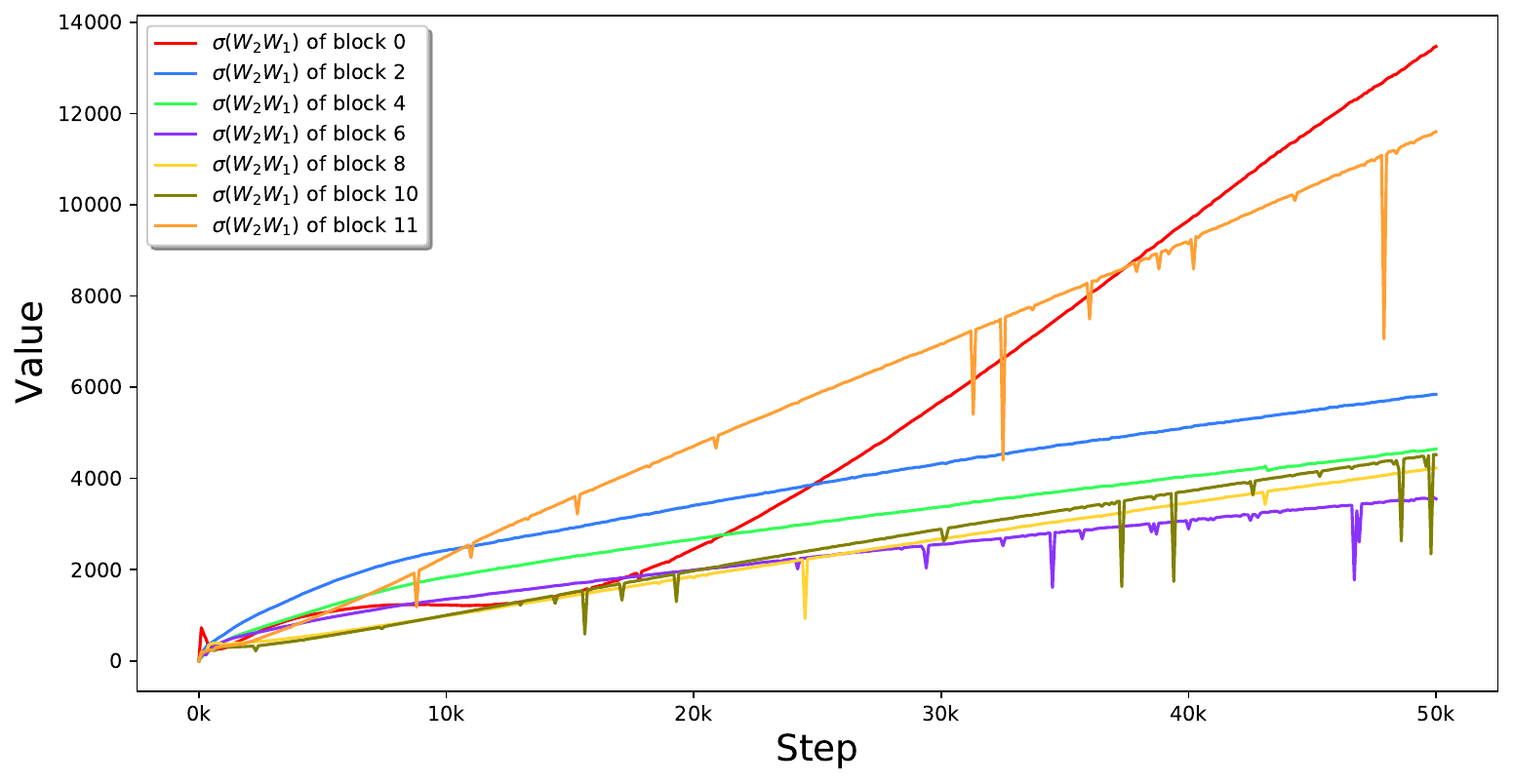}
  
		\caption*{(j) $\sigma_1\left({\bW_2\bW_1}\right)$}
		
	\end{subfigure}
        \begin{subfigure}{0.33\linewidth}
		\centering
            \includegraphics[width=4.5cm,height=3.8cm]{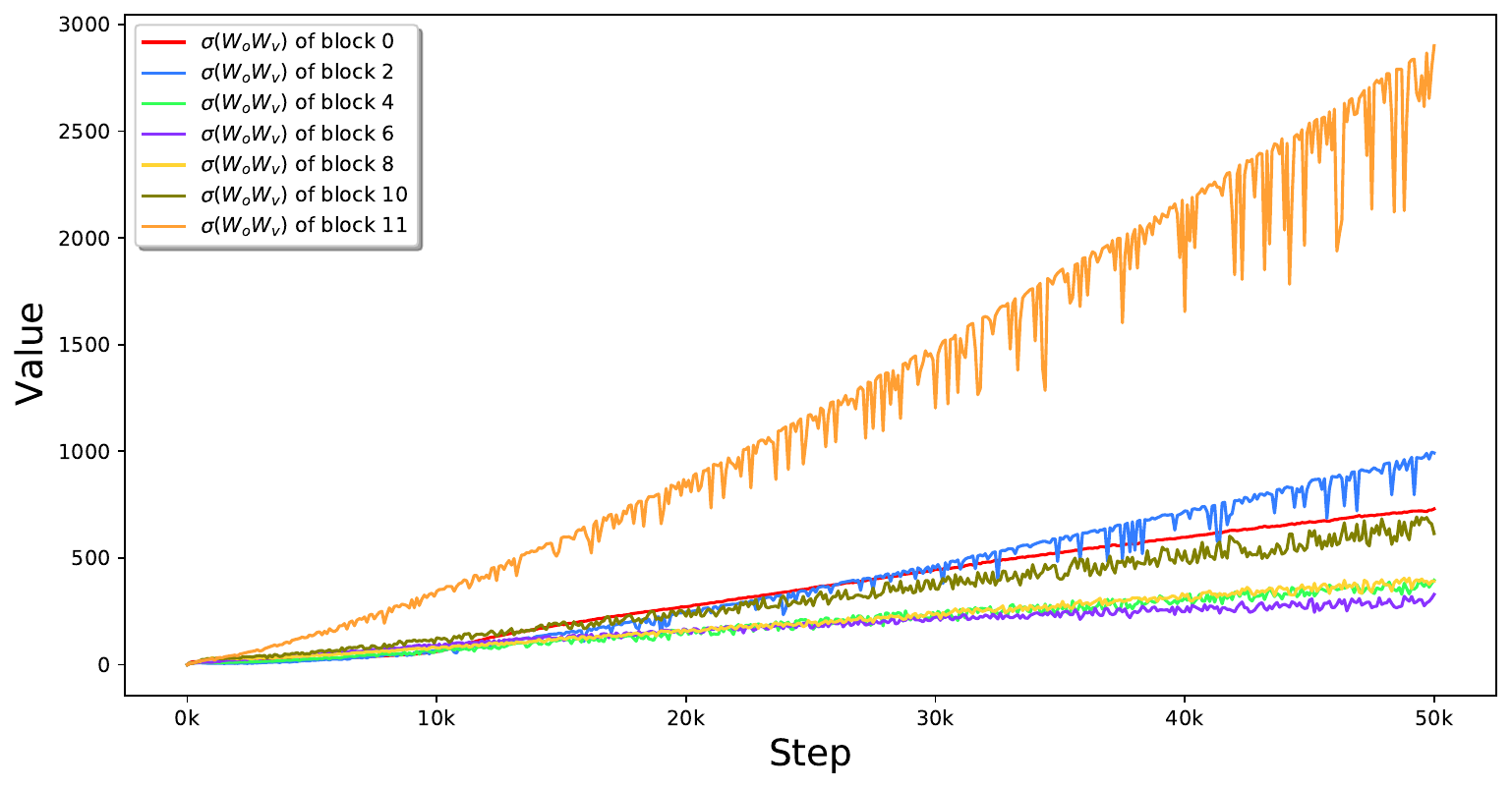}
		\caption*{(k) $\sigma_1\left({\bW_o\bW_v}\right)$}
		
	\end{subfigure}
        \begin{subfigure}{0.33\linewidth}
		\centering
            \includegraphics[width=4.5cm,height=3.8cm]{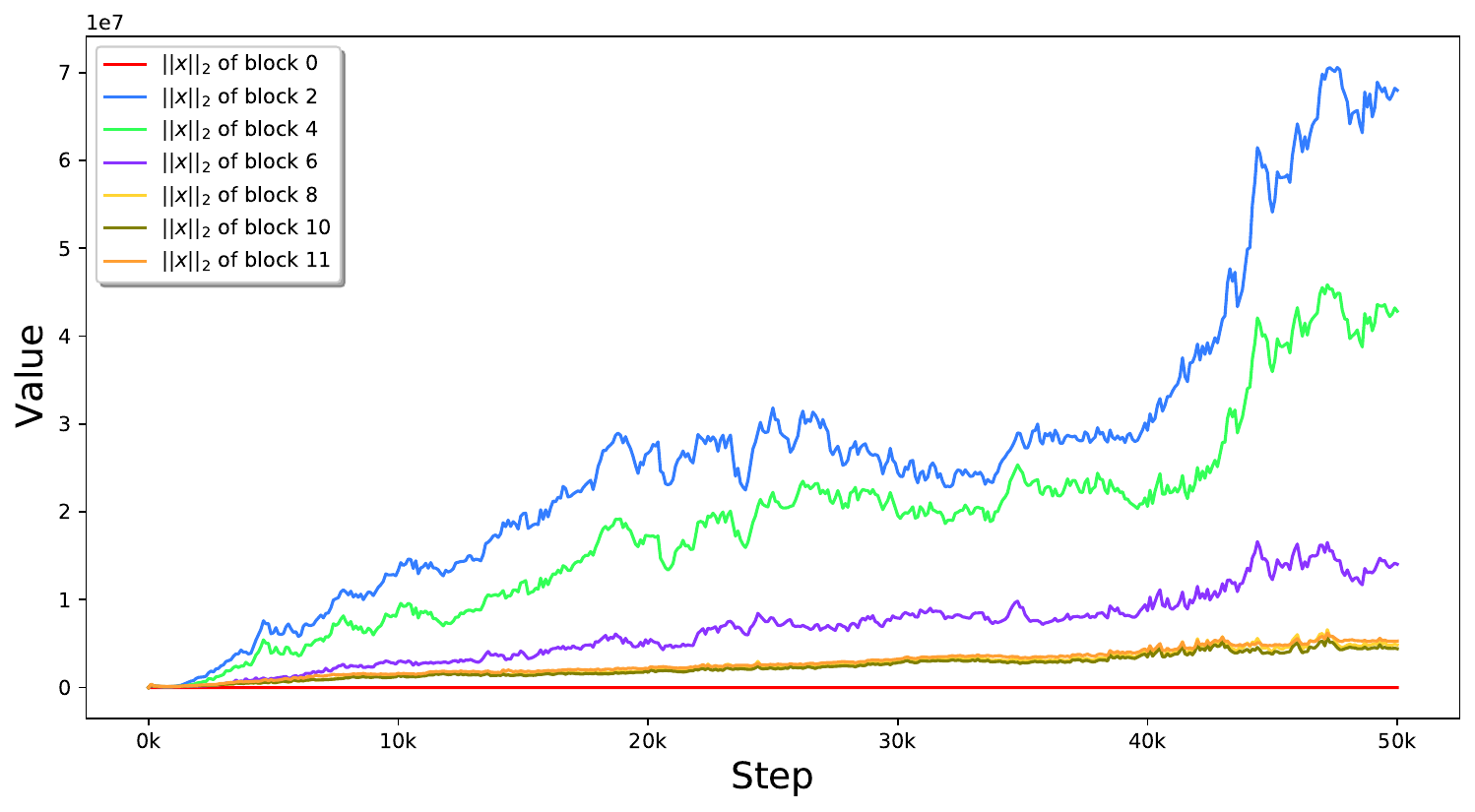}
		\caption*{(l) $\left\|{\bx}\right\|_2$}
		
	\end{subfigure}
        \begin{subfigure}{0.33\linewidth}
		\centering
            \includegraphics[width=4.5cm,height=3.8cm]{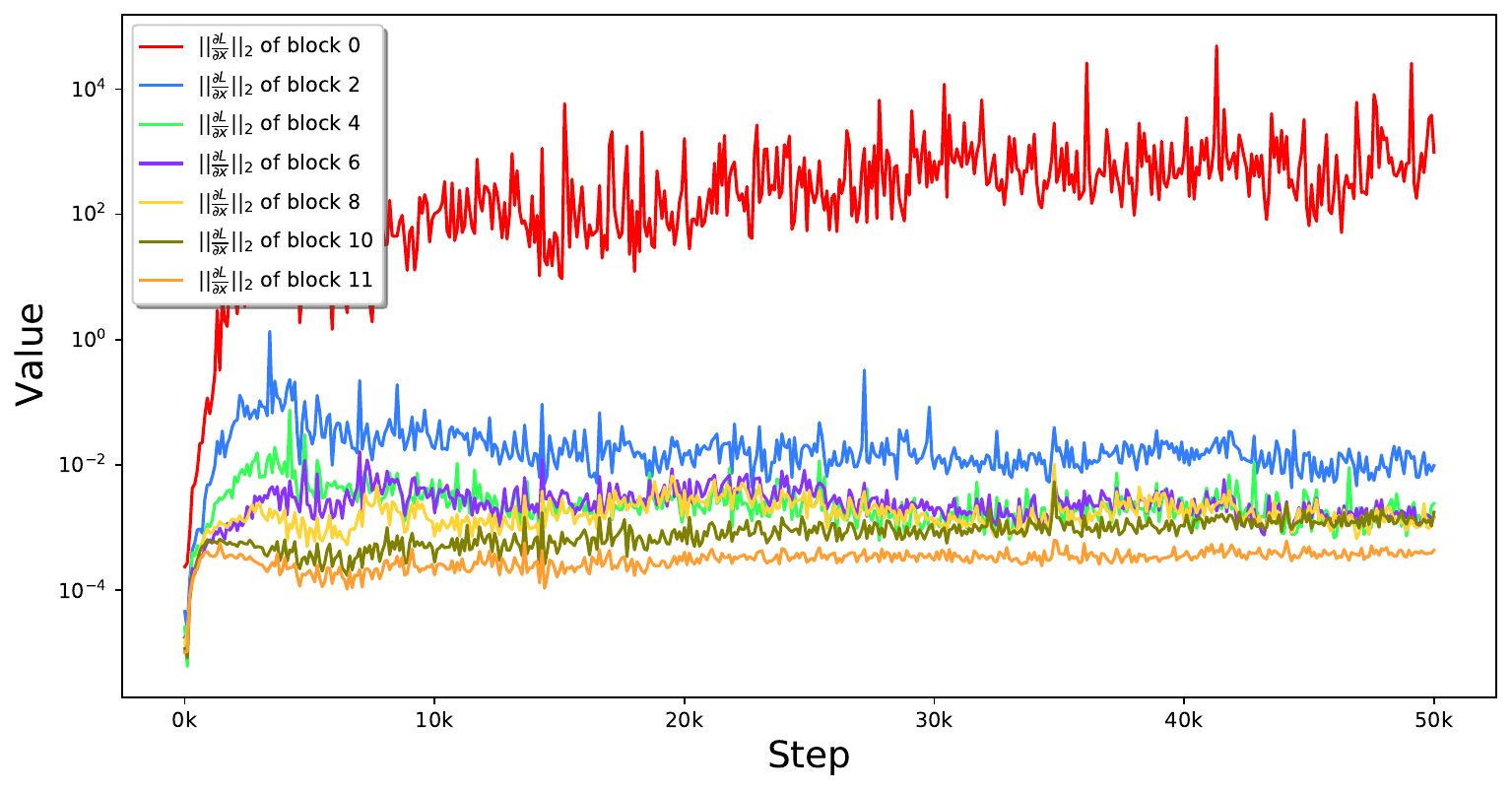}
            \caption*{(m) $\|{\frac{\partial L}{\partial \bx}}\|_2$}
		
	\end{subfigure}
 \caption{Training dynamics of an unsuccessful GPT training.} 
 \label{fig:why_gpt_fail}
\end{figure}

\newpage

\section{{Experiment of 1B ViT}}
{
To further evaluate the effectiveness of our method at a larger scale, we assessed ViT-g with 1B parameters. The ViT-g model architecture consists of 40 layers with a hidden dimension of 1408, 16 attention heads, and an MLP dimension of 6144. The total parameter count is 1011M, around one billion parameters.
We conducted a comparative study between ViT-g with $\text{AdamW}^2$ and ViT-g with AdamW, where ViT-g with AdamW was evaluated under two settings: with and without learning rate warmup. Our $\text{AdamW}^2$ does not use warmup. The comparison results are presented in Figure~\ref{fig:vit_g_loss} and Figure~\ref{fig:vit_g_acc}.}

{Figure~\ref{fig:vit_g_loss} shows that ViT-g with AdamW crashes after only a few training steps when running without warmup. While the use of warmup enables ViT-g to complete training, but the loss spikes one time. Our ViT-g with $\text{AdamW}^2$ not only achieves stable training without warmup but also demonstrates better performance. 
}

\begin{figure}[H]
\centering
\includegraphics[width=0.75\linewidth]{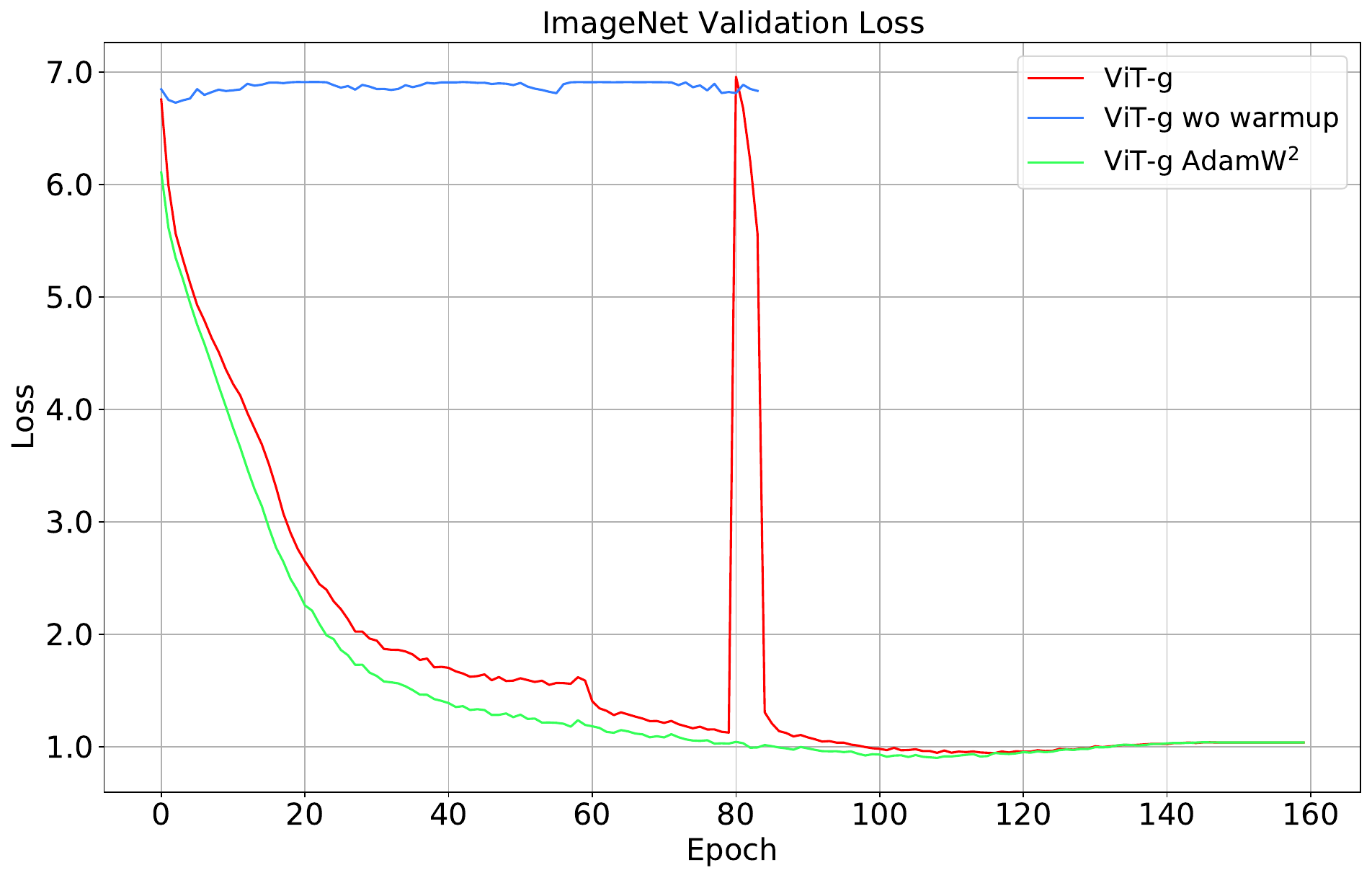}
\caption{Comparison of loss curve of $\text{AdamW}^2$ and AdamW on ViT-g model.}
\label{fig:vit_g_loss}
\end{figure}

\begin{figure}[H]
\centering
\includegraphics[width=0.75\linewidth]{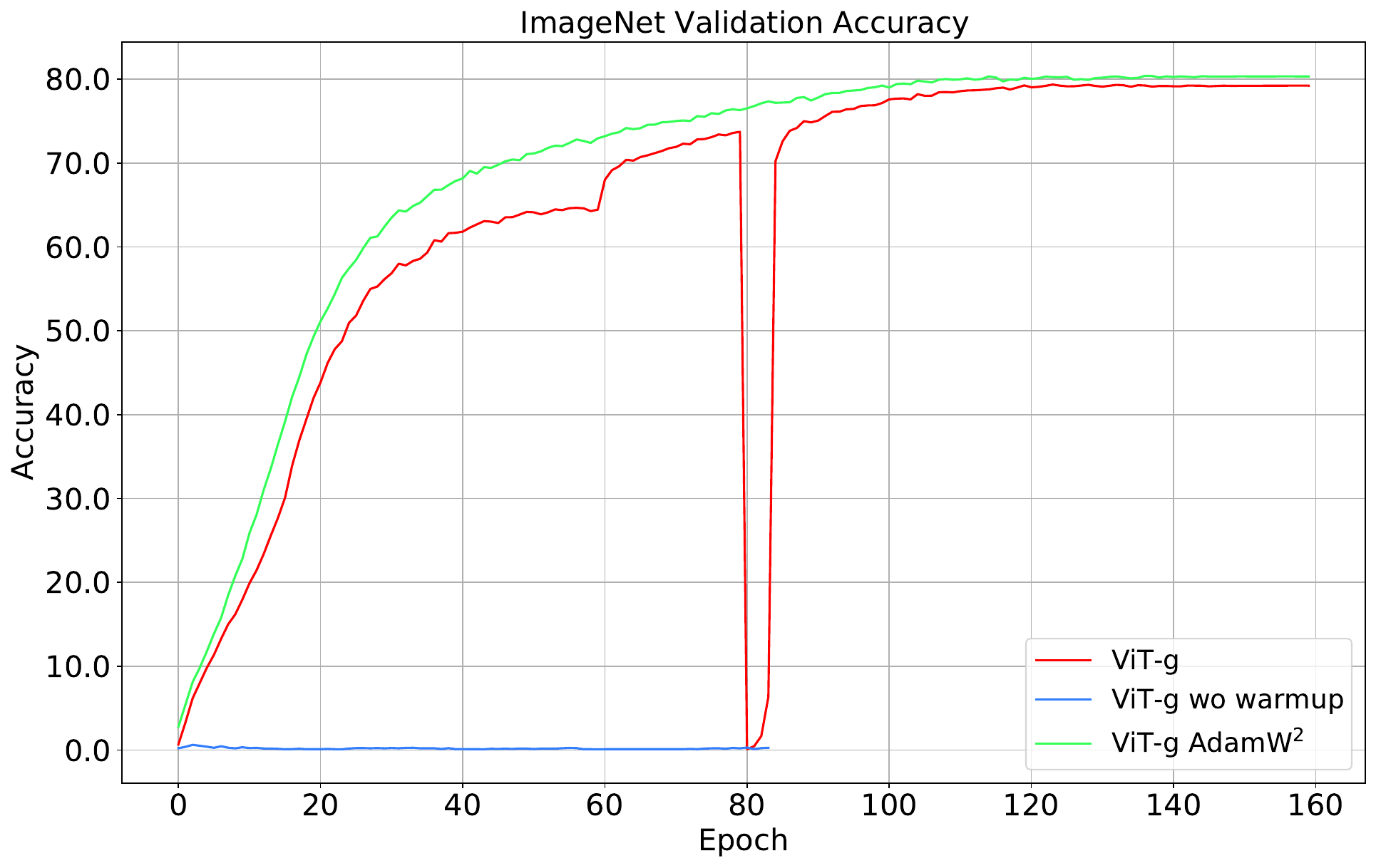}
\caption{Comparison of accuracy of $\text{AdamW}^2$ and AdamW on ViT-g model.}
\label{fig:vit_g_acc}
\end{figure}

\section{{Experiment of 774M nanoGPT}}
{We also evaluated the effectiveness of our method on  a larger-scale language model, termed as nanoGPT-large. The model architecture consists of 36 layers with a hidden dimension of 1280 and 20 attention heads. The total parameter count is 774M. Our experimental setup strictly follows the nanoGPT configuration, including all learning rate settings. It is important to note that training nanoGPT-large is computationally intensive, requiring two weeks to train 600K steps on 16 A800 GPUs. To reduce the training time, we limited our training to 100K steps instead of the full 600K steps. The comparison results are presented in Figure~\ref{fig:nanogpt_large_val_loss}. We can see from Figure~\ref{fig:nanogpt_large_val_loss}, nanoGPT-large achieves a stable training without warmup and obtains a similar validation loss with its counterpart, GPT2-large. This further verifies our understanding to the model crash of Transformer.}

\begin{figure}[H]
\centering
\includegraphics[width=0.75\linewidth]{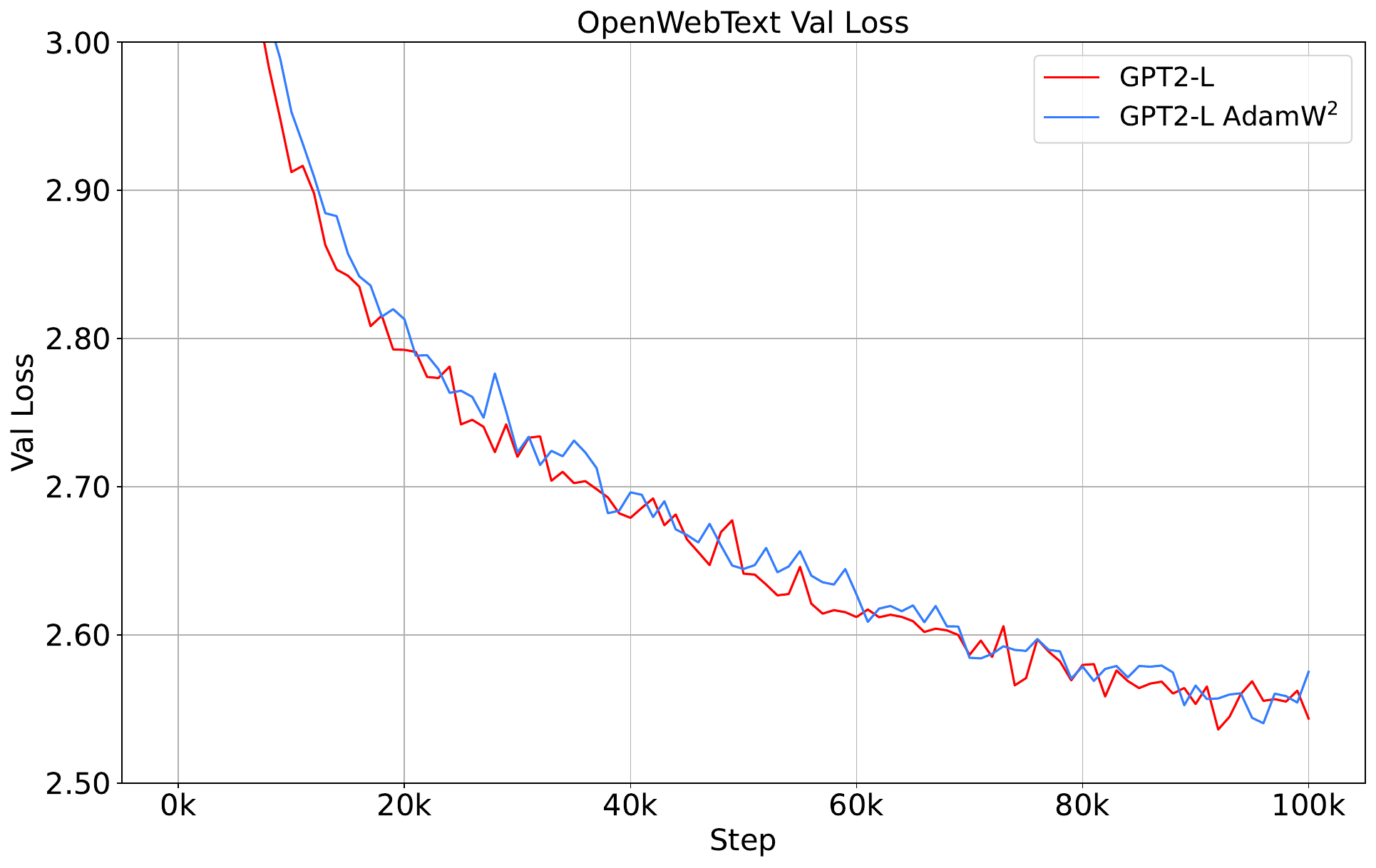}
\caption{Comparison of validation loss of $\text{AdamW}^2$ and AdamW on nanoGPT-large model.}
\label{fig:nanogpt_large_val_loss}
\end{figure}

\section{{Experiment of Flatten-Swin}}
{
Besides ViT, GPT, and Swin-Transformer, we further validated our approach on Flatten-Transformer~\cite{flatten_transformer_han2023flatten}. We used Flatten-Swin, and we compared the performance of our method and the baseline method training 150 and 300 epochs. Our method could stably train and demonstrate performance comparable to the baseline. This further verified the correctness of our understanding of neural network stability.}

\begin{figure}[H]
\centering
\includegraphics[width=0.75\linewidth]{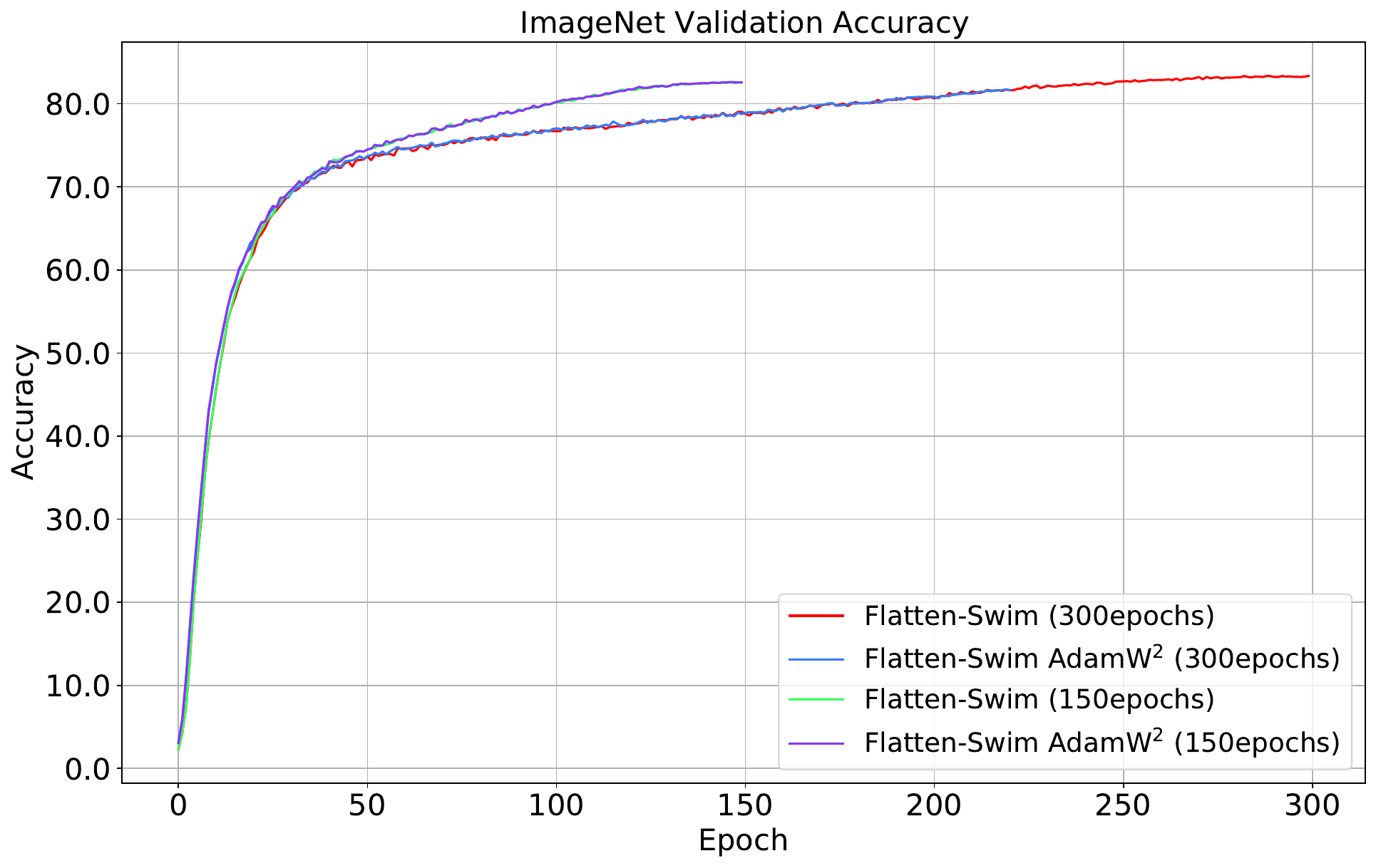}
\caption{Evaluation of Flatten-Swin.}
\label{fig:flatten_swin}
\end{figure}

\newpage

\section{{Actual Learning Rate Curve along with Training steps}}
{
We recorded the actual learning rate throughout the training steps, we sample one point every 50 steps. Our initial setting of the learning rate is  a cosine learning rate scheduler without warmup. If $\alpha_t \frac{\sigma_{1}(\nabla \bW_t)}{\sigma_{1}(\bW_{t-1})} > \tau$, then $\alpha_t$ will be truncated to $\tau \frac{\sigma_{1}(\bW_{t-1})}{\sigma_{1}(\nabla \bW_t)}$. From the figure, we observe that $\boldsymbol{\gamma_1}$ and $\boldsymbol{\gamma_2}$  in RMSNorm only exceed the preset $\tau$ during the initial training phase and rarely exceed it afterwards. For other curves, they somewhat look like a curve with learning rate  warmup, but we can see that different blocks have different learning rates.
}

{We also observe that shallower layers are more likely to violate the preset $\tau$ value. It means the shallower layers are more likely to lead to a greater update of weight matrix and typically require a smaller learning rate.
Additionally, we notice that for the weight matrix $\bW_2$, it is more prone to exceeding the preset $\tau$ value compared to the weight matrix $\bW_1$.}

\begin{figure}[htbp]
	\centering
	\begin{subfigure}{0.46\linewidth}
		\centering
		\includegraphics[width=6.2cm,height=4.8cm]{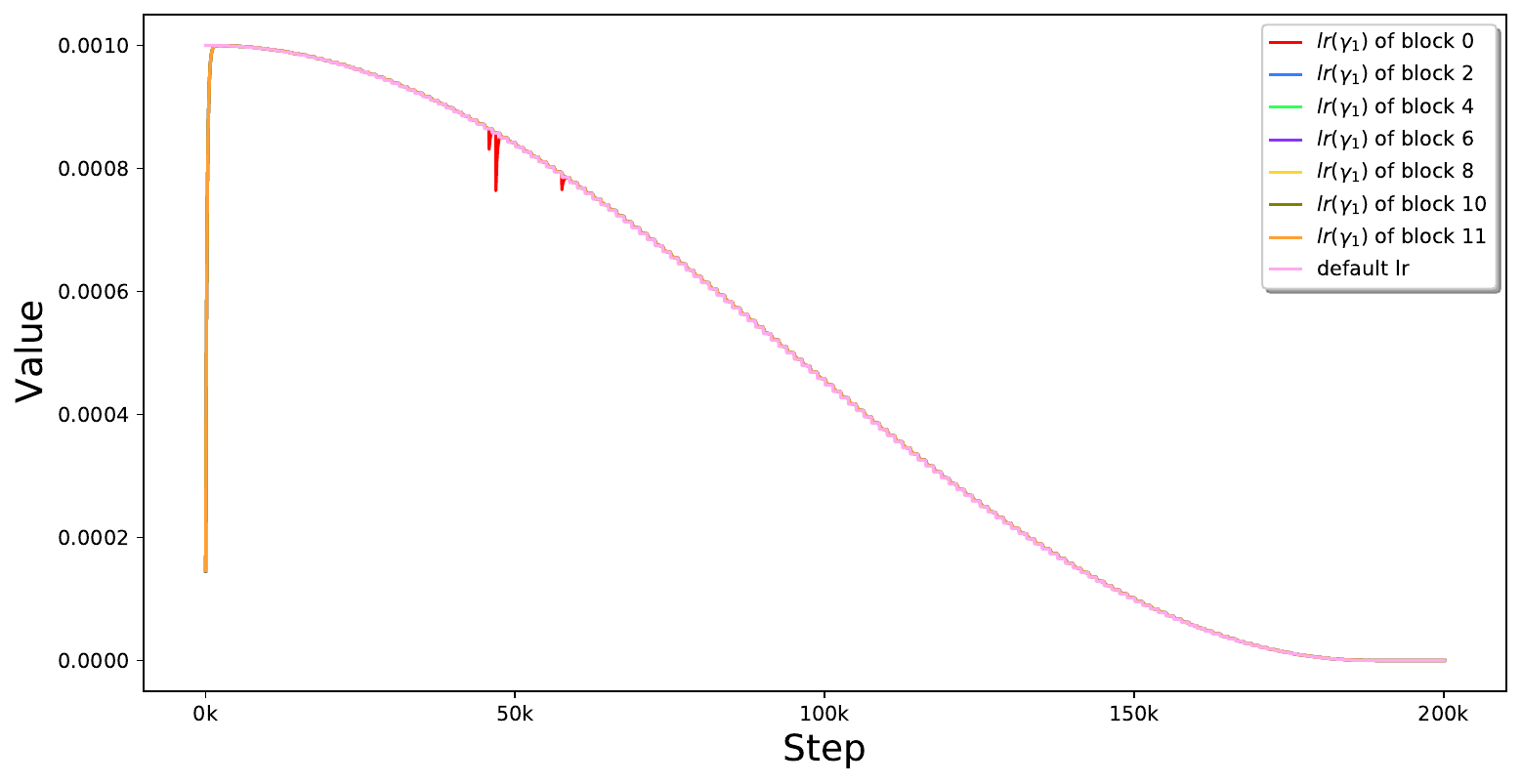}
		\caption*{(a) learning rate $\alpha_{t}$ of ${\boldsymbol{\gamma_1}}$}
	\end{subfigure}
	\begin{subfigure}{0.46\linewidth}
		\centering
		\includegraphics[width=6.2cm,height=4.8cm]{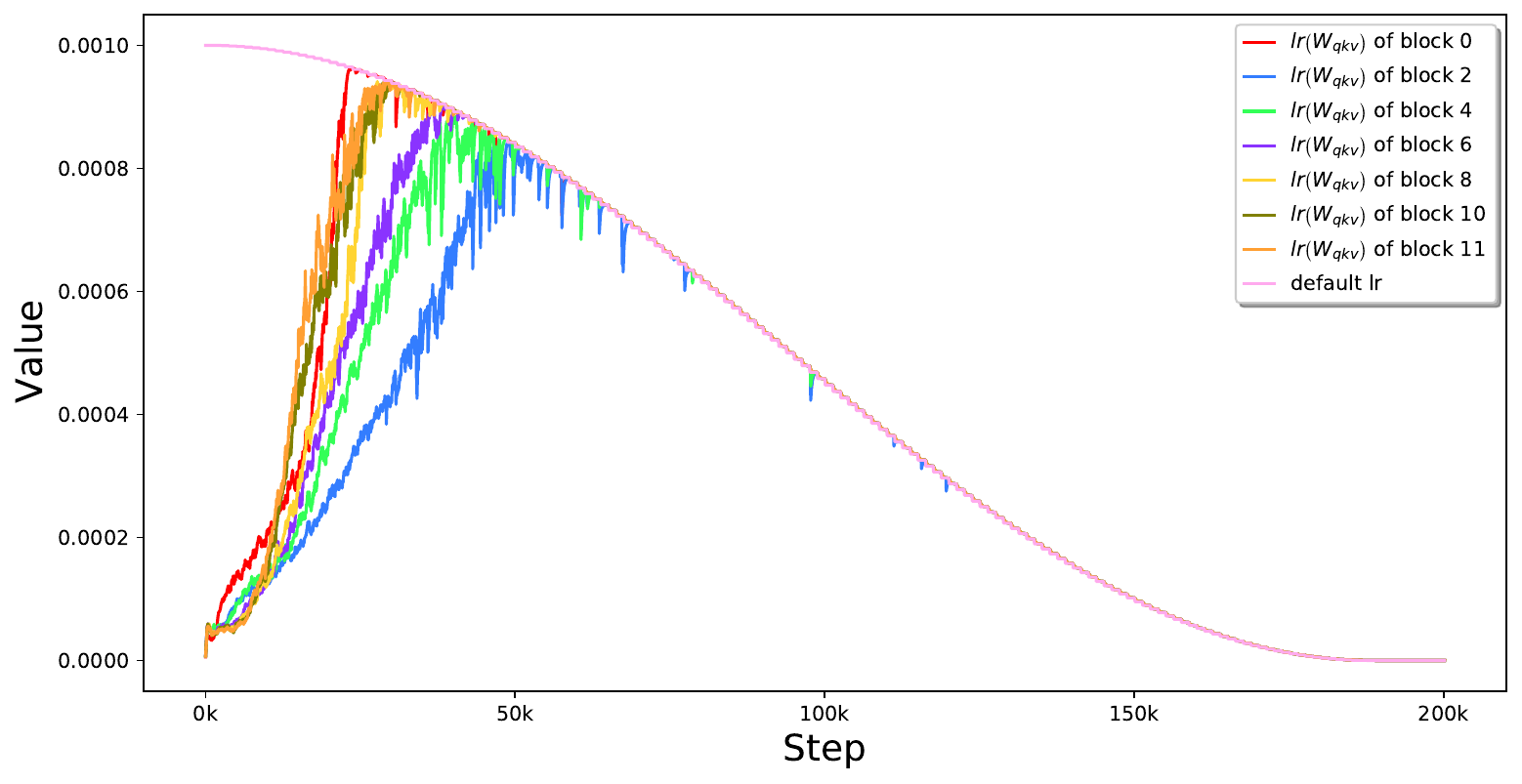}
		\caption*{(b) learning rate $\alpha_{t}$ of $\bW_q$, $\bW_k$ and $\bW_v$}
	\end{subfigure}
    
        \begin{subfigure}{0.46\linewidth}
		\centering
		\includegraphics[width=6.2cm,height=4.8cm]{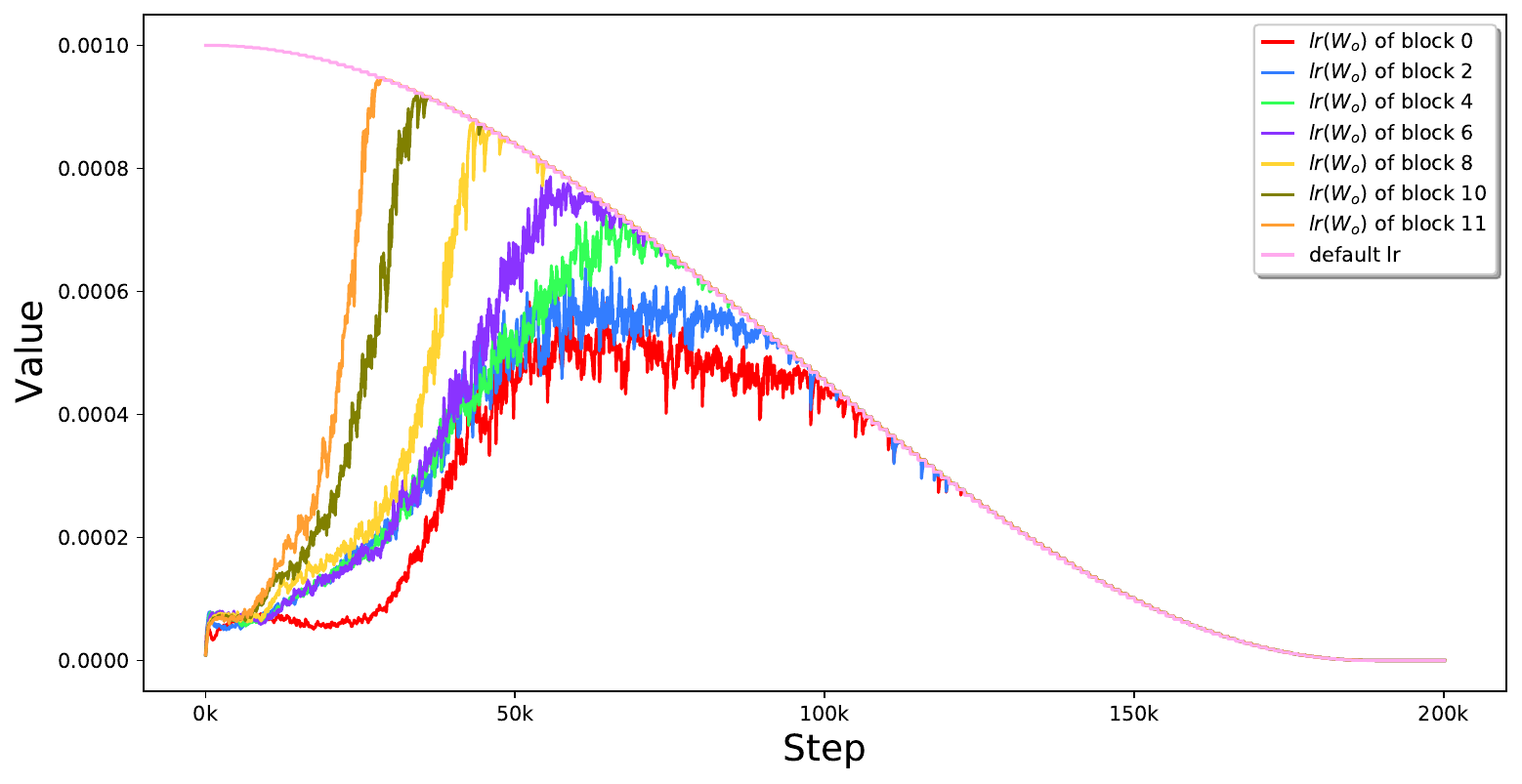}
		\caption*{(c) learning rate $\alpha_{t}$ of $\bW_o$}
	\end{subfigure}
	\begin{subfigure}{0.46\linewidth}
		\centering
		\includegraphics[width=6.2cm,height=4.8cm]{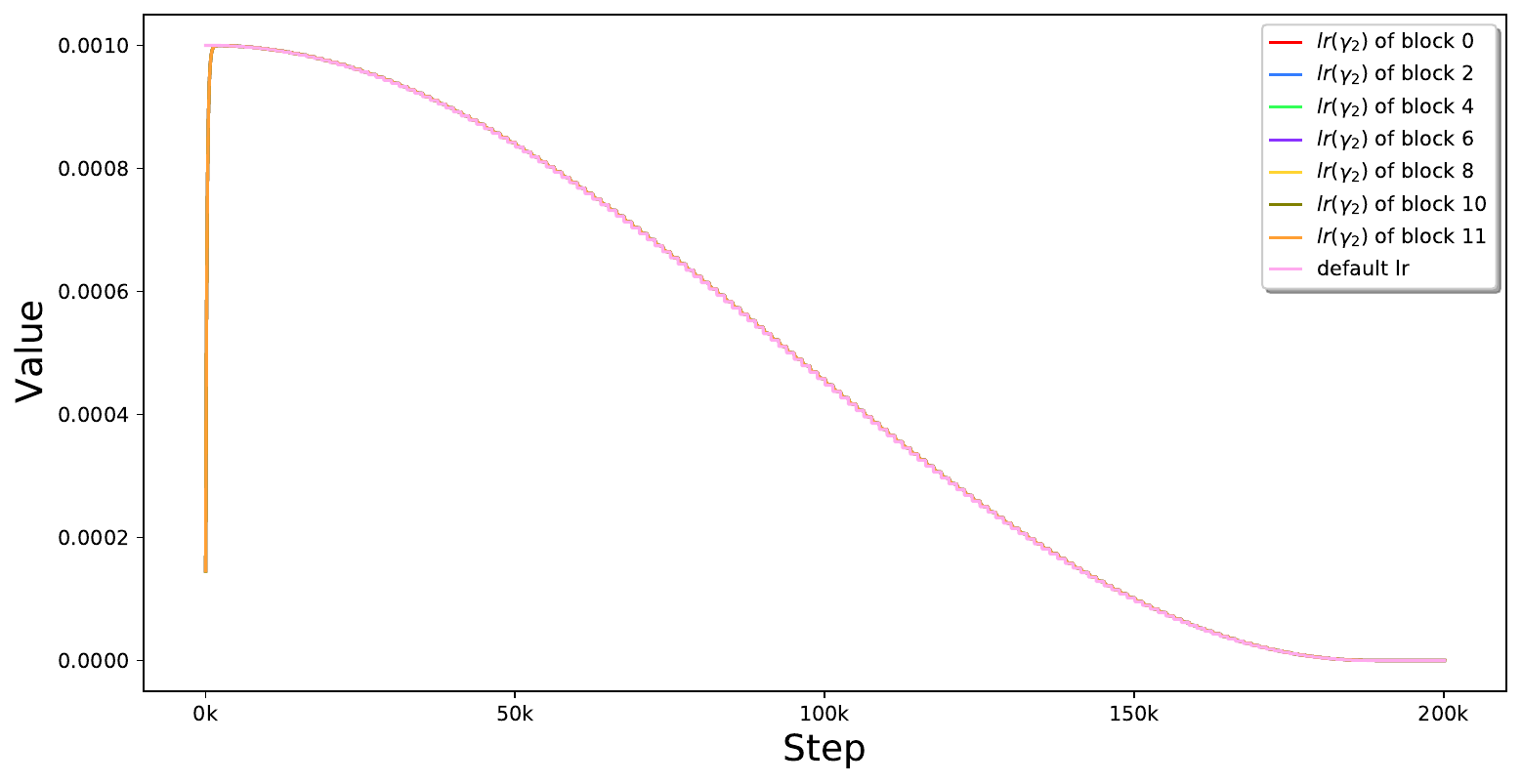}
		\caption*{(d) learning rate $\alpha_{t}$ of ${\boldsymbol{\gamma_2}}$}
		\label{fig:vit_sub_success_wqwk}
	\end{subfigure}
    
	\begin{subfigure}{0.46\linewidth}
		\centering
		\includegraphics[width=6.2cm,height=4.8cm]{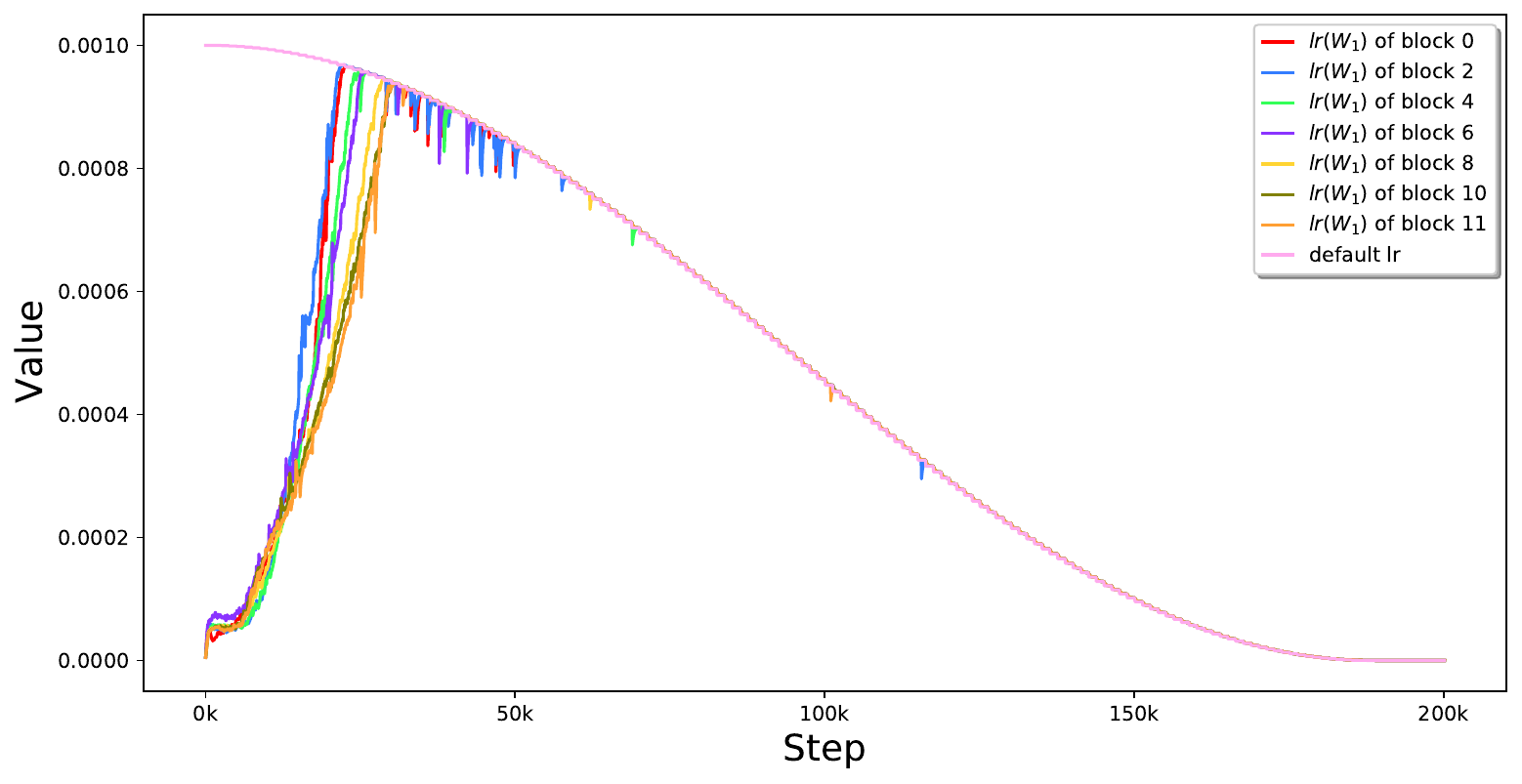}
		\caption*{(e) learning rate $\alpha_{t}$ of $\bW_1$}	
	\end{subfigure}
        \begin{subfigure}{0.46\linewidth}
		\centering
		\includegraphics[width=6.2cm,height=4.8cm]{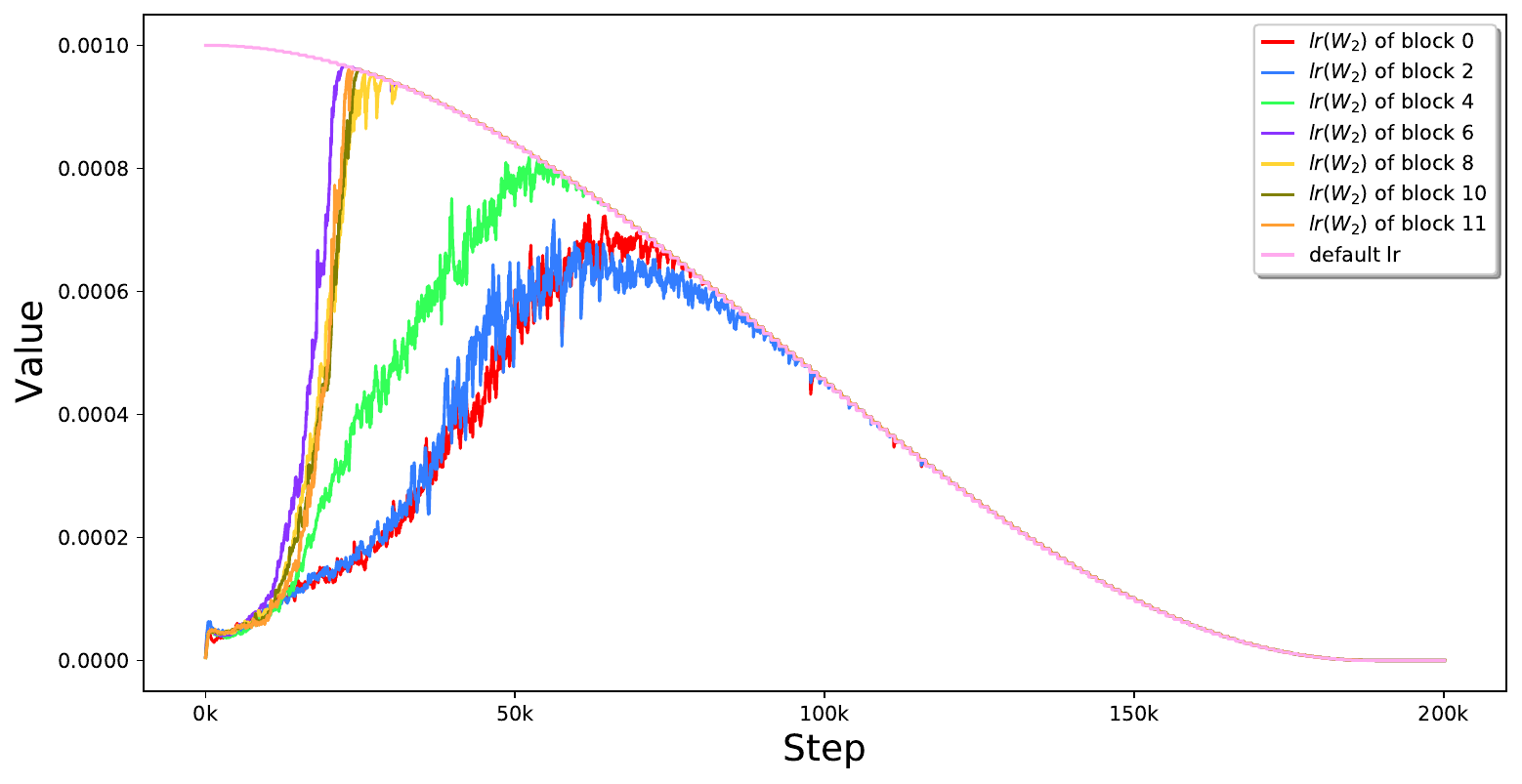}
		\caption*{(f) learning rate $\alpha_{t}$ of $\bW_2$}
		
	\end{subfigure}
 \caption{{Actual Learning Rate Curve along with Training steps.}}
 \label{fig:why_gpt_fail}
\end{figure}

\newpage

\section{Training Configurations}
\label{appendix:training_configurations}
{\textbf{Training Configurations.} We list the training configurations of  ViT, GPT, Swin-Transformer and Flatten-Swin in Table~\ref{tab:training_configuration}. For ViT, GPT, Swin-Transformer and Flatten-Swin, we do not use learning rate warmup.
For GPT, we follow the experimental configurations of nanoGPT~\citep{nanogpt_Karpathy2022}, all parameters are same as GPT2~\citep{gpt2_radford2019language}.  For ViT, we use Timm~\citep{timm_rw2019timm}. For Swin-Transformer, we use the original code provided by~\cite{liu2021swin}. For Flatten-Swin, we use the original code provided by~\citep{flatten_transformer_han2023flatten}.}

\begin{table}[H]
    \centering
\caption{Training configurations for ViT, GPT and Swin-Transformer.}
    \begin{subtable}[t]{0.495\linewidth}
    \caption*{(a) Training configurations for ViT.}
        \begin{tabular}{l|c} 
        training config & ViT-B/L/g ($224^{2}$)   \\ 
\hline
optimizer & $\text{AdamW}^2$ \\  
$\tau$ (In default)  & 0.004 or 0.003 \\
warmup epochs & 0 \\ 
weight init & Truncated Xavier \\ 
base learning rate & 1e-3 \\ 
weight decay & 0.05/0.1  \\ 
optimizer momentum & \(\beta_{1}, \beta_{2}=0.9,0.99\)  \\ 
batch size & 1024 \\ 
training epochs & 150  \\ 
learning rate schedule & cosine decay \\
randaugment  & \((9,0.5)\) \\ 
mixup  & 0.8 \\ 
cutmix & 1.0 \\
random erasing  & 0 \\ 
label smoothing  & 0.1 \\ 
stochastic depth  & \(0.1 / 0.5\) \\ 
gradient clip & None \\ 
exp. mov. avg. (EMA)  & no\\ 
\hline
\end{tabular}
    \end{subtable}
\begin{subtable}[t]{0.495\linewidth}
\centering
\caption*{(b) Training configurations for GPT.}
        \begin{tabular}{l|c} 
            training config & GPT-S/L   \\ 
\hline
optimizer & $\text{AdamW}^2$ \\  
$\tau$ & 0.01 \\
warmup epochs & 0 \\ 
weight init & Xavier  \\ 
baseline learning rate & 0.0006 or 0.00025\\ 
weight decay & 0.1  \\ 
optimizer momentum & \(\beta_{1}, \beta_{2}=0.9,0.95\) \\
tokens seen each update & 500,000 \\
max iters & 600K or 100K \\
batch size & 480 \\ 
sequence length & 1024 \\ 
dropout & 0.0 \\
bfloat16 & True  \\
gradient clipping & 1.0 \\
\hline
\end{tabular}     
\end{subtable} 

\begin{subtable}[t]{0.49\linewidth}
\vspace{10pt}
\centering
\caption*{(c) Training configurations for Swin-Transformer.}
\begin{tabular}{l|c} 
training config & Swin S/B ($224^{2}$)   \\ \hline
optimizer & $\text{AdamW}^2$ \\  
$\tau$ (In default) & 0.004 \\
warmup epochs & 0 \\ 
training epochs & 300  \\ 
others & same as~\cite{liu2021swin} \\
\hline
\end{tabular}
\end{subtable}
\begin{subtable}[t]{0.44\linewidth}
\vspace{10pt}
\centering
    \caption*{(d) Training configurations for Flatten-Swin.}
        \begin{tabular}{l|c} 
training config & Flatten-Swin S ($224^{2}$)   \\ 
\hline
optimizer & $\text{AdamW}^2$ \\  
$\tau$ (In default) & 0.004 \\
warmup epochs & 0 \\ 
training epochs & 150 or 300  \\ 
others & same as~\cite{flatten_transformer_han2023flatten} \\
\hline
\end{tabular}
\end{subtable}

\label{tab:training_configuration}
\end{table}

}

\

{
\section{non-symmetric positive quasi-definite square matrix}
\label{appendix:non_symmetric}
When we mention a non-symmetric positive quasi-definite square matrix, we mean it has the following three properties,
\begin{enumerate}[leftmargin=*]
    \item $\boldsymbol{W}_q^{\top} \boldsymbol{W}_k$ is not symmetric because generally,  $\boldsymbol{W}_q^{\top} \boldsymbol{W}_k \neq {\boldsymbol{W}_k^{\top} \boldsymbol{W}_q}$,
    \item $\boldsymbol{W}_q^{\top} \boldsymbol{W}_k$ is a square matrix and most of its eigenvalues are larger than 0, and only very few are less than 0.0. So we call it positive quasi-definite matrix. 
    \item if we assume $\boldsymbol{W}=\boldsymbol{W}_q^{\top} \boldsymbol{W}_k$ is positive definite matrix, if for each element in $\boldsymbol{x}$ is sampled from a standard Gaussian distribution, we can prove $\mathbb{E}\left[{\boldsymbol{x}_i}^{\top} \boldsymbol{W} \boldsymbol{x}_i\right] \gg \mathbb{E}\left[{\boldsymbol{x}_i}^{\top} \boldsymbol{W} \boldsymbol{x}_j\right]$ when $i\neq j$, see Appendix~\ref{appendix:benign_entropy_collapse} for the proof.
\end{enumerate}
}

{
\section{Discussion about Rank collapse, Entropy collapse and Sparse yet low-rank entropy matrix}

Before we start our discussion, let us see three matrices,

$\boldsymbol{A} = \begin{pmatrix} \frac{1}{5} & \frac{1}{5} & \frac{1}{5} & \frac{1}{5}& \frac{1}{5}\\ \frac{1}{5} & \frac{1}{5} & \frac{1}{5} & \frac{1}{5}& \frac{1}{5}\\ \frac{1}{5} & \frac{1}{5} & \frac{1}{5} & \frac{1}{5}& \frac{1}{5}\\ \frac{1}{5} & \frac{1}{5} & \frac{1}{5} & \frac{1}{5}& \frac{1}{5}\\ \frac{1}{5} & \frac{1}{5} & \frac{1}{5} & \frac{1}{5}& \frac{1}{5} \end{pmatrix}$, $\boldsymbol{B} = \begin{pmatrix} 1 & 0 & 0 & 0& 0\\ 0 & 1 & 0& 0& 0\\0 & 0 & 1& 0& 0\\ 0 & 0 & 0& 1& 0\\0 & 0 & 0& 0& 1  \end{pmatrix}$,  $\boldsymbol{C} = \begin{pmatrix} 1 & 0 & 0 & 0& 0\\ 1 & 0 & 0& 0& 0\\ 1 & 0 & 0& 0& 0\\ 1 & 0 & 0& 0& 0\\1 & 0 & 0& 0& 0 \end{pmatrix}$

We can see that $\boldsymbol{A}$ is low-rank, $\boldsymbol{B}$ is sparse but not low-rank, $\boldsymbol{C}$ is sparse and low-rank. 

In previous papers~\citep{rank_collapse_dong2021attention, stabilizing_transformer_zhai2023stabilizing}, researchers have analyzed the problem of model crash via \textit{rank collapse of activations and entropy collapse of attention map}. Dong et al.~\citep{rank_collapse_dong2021attention} attributes the model crash into \textit{rank collapse of the activations}, but \citet{stabilizing_transformer_zhai2023stabilizing} think it is the \textit{entropy collapse of the attention map} leading to the model crash.

However, based on our analysis, we can find a counterexamples  for entropy collapse, and meanwhile the rank collapse of the activation cannot fully describe the inner reason of the model crash (the weight matrix instead of activations).  When the state of $\boldsymbol{C}$ usually happens, the model crashed,
\begin{itemize}[leftmargin=*]
    \item Rank collapse of the activations cannot reveal the underlying cause that exists in the weight matrix. Weight matrix is the inner key ingredient of the model instead of activations.
    \item $\boldsymbol{B}$ is a counterexample of entropy collapse. we observe that in some successful cases, state $\boldsymbol{B}$ occurs. According to the definition of entropy collapse, state $\boldsymbol{B}$ should lead to model crash; however, our experiments show that the model remains stable in this state.
    \item Sparse yet low-rank attention matrix is the state of the attention map when a model crashs. We believe rank collapse of activations and entropy collapse of attention map are not enough to describe the state of the model crash precisely. According to our analysis, the Spectral Energy Concentration (SEC) of the $\boldsymbol{W}_q^{\top} \boldsymbol{W}_k$ is the inner reason the model crash, and the sparse yet low-rank attention matrix is  the phenomena observed on the attention matrix.
\end{itemize}

In summary, our paper, via a rigid theoritical analysis, our paper reveals the Spectral Energy Concentration (SEC) of the $\boldsymbol{W}_q^{\top} \boldsymbol{W}_k$ is the inner reason the model crash, and the sparse yet low-rank attention matrix is the phenomena that is observed on the attention matrix.
}

\

\ 

\ 

\section{Illustration of Figure 5}
\label{sec:illustration_of_figure5}
\myparagraph{Illustration of Arrow 1}

According to the property of Kronecker Product, we have

$$
\operatorname{rank}(\boldsymbol{X} \otimes \boldsymbol{X}) = \operatorname{rank}(\boldsymbol{X}) \cdot \operatorname{rank}(\boldsymbol{X}).
$$

Since $\boldsymbol{X}$ is low-rank, then $\boldsymbol{X} \otimes \boldsymbol{X}$ is also low-rank. In the following, we will also prove the singular values of $\boldsymbol{X} \otimes \boldsymbol{X}$ will also strengthen the concertration of spectral energy into some directions with large singular values.

\ 

\myparagraph{Illustration of Arrow 2}

According to the computation of Jacobian matrix, we have
$$
\frac{\partial \operatorname{vec}(\boldsymbol{P})}{\partial \operatorname{vec}({\boldsymbol{W}_q}^{\top}{\boldsymbol{W}_k} )} =  \boldsymbol{X}^{\top} \otimes \boldsymbol{X}^{\top}
$$
where $\boldsymbol{P} = \boldsymbol{X}^{\top} \boldsymbol{W}_q^{\top} \boldsymbol{W}_k \boldsymbol{X}^{\top}.$

\ 

\myparagraph{Illustration of Arrow 3}

Let $\boldsymbol{X} \in \mathbb{R}^{m \times n}$ be a matrix with rank $r \leq \min(m,n)$. Denote the singular values of $\boldsymbol{X}$ as $\sigma_1 \geq \sigma_2 \geq \cdots \geq \sigma_r > 0$ and $\sigma_{r+1} = \cdots = \sigma_{\min(m,n)} = 0$.

Definition. (Singular Values of Kronecker Product)**
For a matrix $\boldsymbol{X}$, define $\Lambda(\boldsymbol{X})$ as the set of all possible products of its singular values, i.e.,$ \Lambda(\boldsymbol{X}) = \{\sigma_i \sigma_j : 1\leq i,j \leq r\}$.

Theorem (Singular Values of Kronecker Product)
For a low-rank matrix $\boldsymbol{X} \in \mathbb{R}^{m \times n}$ with rank $r$, the singular values of $\boldsymbol{X} \otimes \boldsymbol{X}$ are precisely the elements in $\Lambda(\boldsymbol{X})$. More formally, let $\{\mu_k\}$ be the set of singular values of $\boldsymbol{X} \otimes \boldsymbol{X}$. Then $\{ \mu_k\} = \{\sigma_i \sigma_j \ \text{where} 1 \leq i,j \leq r \}.$

Proof. Consider the singular value decomposition (SVD) of $\boldsymbol{X} = \boldsymbol{U}\boldsymbol{\Sigma} \boldsymbol{V}^T$, where
 $\boldsymbol{U} \in \mathbb{R}^{m \times m}$ is an orthogonal matrix,
 $\boldsymbol{V} \in \mathbb{R}^{n \times n}$ is an orthogonal matrix, and 
 $\boldsymbol{\Sigma} = \text{diag}(\sigma_1, \ldots, \sigma_r, 0, \ldots, 0)$. Then, the Kronecker product $\boldsymbol{X} \otimes \boldsymbol{X}$ can be expanded as:
$\boldsymbol{X} \otimes \boldsymbol{X} = (\boldsymbol{U}\boldsymbol{\Sigma} \boldsymbol{V}^T) \otimes (\boldsymbol{U}\boldsymbol{\Sigma} \boldsymbol{V}^T) = (\boldsymbol{U} \otimes \boldsymbol{U})(\boldsymbol{\Sigma} \otimes \boldsymbol{\Sigma})(\boldsymbol{V}^T \otimes \boldsymbol{V}^T)$. 
The singular values of $\boldsymbol{U} \otimes \boldsymbol{U}$ and $\boldsymbol{V}^T \otimes \boldsymbol{V}^T$ are all 1, as they are composed of orthogonal matrices.

Therefore, the singular values of $\boldsymbol{X} \otimes \boldsymbol{X}$ are precisely the elements of $\boldsymbol{\Sigma} \otimes \boldsymbol{\Sigma}$, which are exactly $\{\sigma_i \sigma_j : 1 \leq i,j \leq r\}$.

In summary, the Singular Values of Kronecker Product $\boldsymbol{X} \otimes \boldsymbol{X}$ will also strengthen the concertration of spectral energy into some directions with large singular values. In this way, the singular values of  $\frac{\partial \operatorname{vec}(\boldsymbol{P})}{\partial \operatorname{vec}({\boldsymbol{W}_q}^{\top}{\boldsymbol{W}_k} )}$ will also concentrate to a few directions.

When we update $\boldsymbol{W}_q^{\top} \boldsymbol{W}_k$ according to the following equation:
$$
\operatorname{vec}(\boldsymbol{W}_q^{\top} \boldsymbol{W}_k)_{new} =  \operatorname{vec}(\boldsymbol{W}_q^\top \boldsymbol{W}_k)_{old} - \alpha \frac{\partial \mathcal{L}}{\partial \operatorname{vec}(\boldsymbol{P})} \frac{\partial \operatorname{vec}(\boldsymbol{P})}{\partial \operatorname{vec}({\boldsymbol{W}_q}^{\top}{\boldsymbol{W}_k} )}
$$

where $\mathcal{L}$ is the loss function, and $\alpha$ is the step size. Since the singular values of $\frac{\partial \operatorname{vec}(\boldsymbol{P})}{\partial \operatorname{vec}({\boldsymbol{W}_q}^{\top}{\boldsymbol{W}_k} )}$ will also concentrate to a few directions, the update will lead to the singular values of $\boldsymbol{W}_q^{\top} \boldsymbol{W}_k$ tends to concentrate.

\ 

\myparagraph{Illustration of Arrow 4}

Our Theorem 1 in the paper is to demonstrate that the spectral energy concentration of $\boldsymbol{W}_q^{\top} \boldsymbol{W}_k$ and the associated dominant large singular values will lead to $\boldsymbol{A}$ to be a sparse yet low-rank matrix.

\

\myparagraph{Illustration of Arrow 5}

Since $\boldsymbol{X}^{l+1} = \boldsymbol{V}\boldsymbol{A} $, where $\boldsymbol{V}$ is the value matrix after projection in attention module and $\boldsymbol{A}$ is the attention matrix, according to linear algebra, we have, 

$$
\text{rank}(\boldsymbol{X}^{l+1}) < \min\{ \text{rank}(\boldsymbol{V}), \text{rank}(\boldsymbol{A}) \}.
$$

Thus, we have that $\boldsymbol{X}^{l+1}$ is also a low-rank matrix.

\end{document}